\documentclass[10pt,twocolumn,letterpaper]{article}

\usepackage{iccv}
\usepackage{times}
\usepackage{epsfig}
\usepackage{graphicx}
\usepackage{amsmath}
\usepackage{amssymb}
\usepackage{amsthm}

\usepackage[]{graphicx}
\usepackage{amsfonts, amssymb}
\usepackage{mdwmath}
\usepackage{mdwtab}
\usepackage[tight,footnotesize]{subfigure}
\usepackage{algorithm}	
\usepackage{algpseudocode}
\usepackage{balance}
\usepackage{comment}
\usepackage{gensymb}
\usepackage{multirow}
\usepackage{comment}
\usepackage[]{appendix}
\usepackage{bm}
\usepackage{mathrsfs}
\usepackage[utf8]{inputenc}
\usepackage[english]{babel}

\newcommand{\PreserveBackslash}[1]{\let\temp=\\#1\let\\=\temp}
\newcolumntype{C}[1]{>{\PreserveBackslash\centering}p{#1}}
\newcolumntype{R}[1]{>{\PreserveBackslash\raggedleft}p{#1}}
\newcolumntype{L}[1]{>{\PreserveBackslash\raggedright}p{#1}}

\newtheorem{lemma}{Lemma}

\usepackage[pagebackref=true,breaklinks=true,letterpaper=true,colorlinks,bookmarks=false]{hyperref}

\iccvfinalcopy 

\def\iccvPaperID{5548} 

\ificcvfinal\fi

\begin{document}

\title{USIP: Unsupervised Stable Interest Point Detection from 3D Point Clouds}

\author{Jiaxin Li\thanks{now at nuTonomy: an APTIV company.} \qquad Gim Hee Lee \\
Department of Computer Science, National University of Singapore
}


\maketitle

\begin{abstract}

In this paper, we propose the USIP detector: an Unsupervised Stable Interest Point detector that can detect highly repeatable and accurately localized keypoints from 3D point clouds under arbitrary transformations without the need for any ground truth training data. Our USIP detector consists of a feature proposal network that learns stable keypoints from input 3D point clouds and their respective transformed pairs from randomly generated transformations. We provide degeneracy analysis of our USIP detector and suggest solutions to prevent it. We encourage high repeatability and accurate localization of the keypoints with a probabilistic chamfer loss that minimizes the distances between the detected keypoints from the training point cloud pairs. Extensive experimental results of repeatability tests on several simulated and real-world 3D point cloud datasets from Lidar, RGB-D and CAD models show that our USIP detector significantly outperforms existing hand-crafted and deep learning-based 3D keypoint detectors. Our code is available at the project website.\footnote{https://github.com/lijx10/USIP}
\end{abstract} 

\section{Introduction} \label{sec_intro}
\begin{figure}[h] 
        \centering
        \subfigure[]{\includegraphics[width=0.2\textwidth]{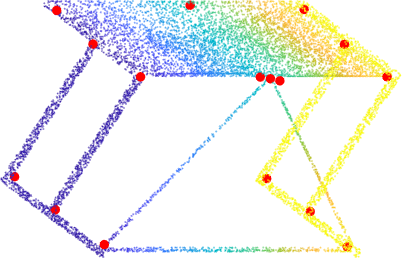}}
        \subfigure[]{\includegraphics[width=0.2\textwidth]{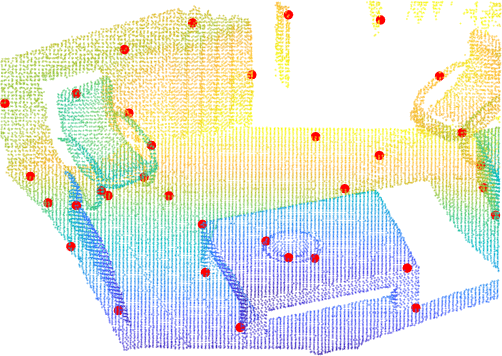}}
        \subfigure[]{\includegraphics[width=0.2\textwidth]{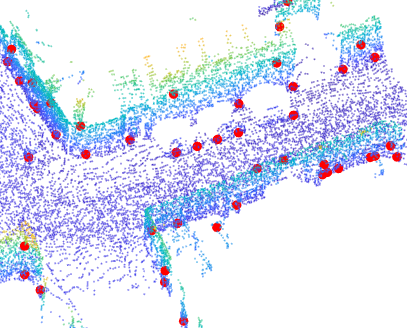}}
        \subfigure[]{\includegraphics[width=0.2\textwidth]{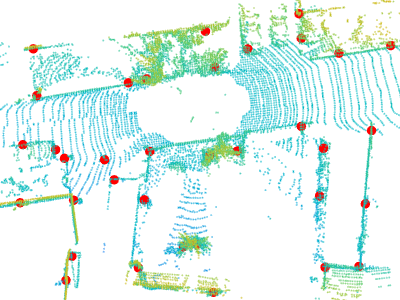}}
        \caption{Examples of keypoints detected by our USIP detector on four datasets: (a) ModelNet40 \cite{wu20153d}, object model. (b) Redwood \cite{choi2015robust} (Trained on ``RGB-D reconstruction dataset" \cite{zeng20173dmatch}), indoor RGB-D. (c) Oxford RobotCar \cite{RobotCarDatasetIJRR}, outdoor SICK LiDAR. (d) KITTI \cite{Geiger2012kitti} (Trained on Oxford), outdoor Velodyne LiDAR.} 
        \label{fig_intro}
        \vspace{-10pt}
\end{figure}

3D interest point or keypoint detection refers to the problem of finding stable points with well-defined positions that are highly repeatable on 3D point clouds under arbitrary SE(3) transformations. These detected keypoints play important roles in many computer vision and robotics tasks, where 3D point clouds are widely adopted as the data structure to represent objects and scenes in the 3D space. Examples include geometric registration for 3D object modeling \cite{besl1992method} or point cloud-based Simultaneous Localization and Mapping (SLAM) \cite{montemerlo2002fastslam},
and 3D object \cite{Tangelder2008,Lian2012} or place recognition \cite{uy2018pointnetvlad}. In these tasks, the detected keypoints are respectively used as correspondences to compute rigid transformations, and locations to extract representative signatures for efficient retrievals. Hence, a keypoint detector that cannot produce highly repeatable and well-localized keypoints from 3D point clouds under arbitrary transformations would render these tasks to fail catastrophically. 

Despite the high number of successful hand-crafted detectors proposed for 2D images \cite{rublee2011orb, lowe2004distinctive, harris1988combined}, significantly lesser hand-crafted detectors \cite{tombari2013performance} with limited success are proposed for hand-crafted detectors on 3D point clouds. This difference can be largely attributed to the difficulty in hand-crafting powerful algorithms to extract meaningful information solely from the Euclidean coordinates of the point cloud in comparison to images that contain richer information from the additional RGB channels. 
The problem is further aggravated by the fact that it is difficult to hand-craft 3D detectors to handle 3D point clouds in arbitrary transformations, \ie, different reference coordinate frames. In particular, different transformations applied to the same 3D point cloud cause the Euclidean coordinates of each point to change significantly, thus severely affecting the repeatability of the keypoints from the 3D detectors. 

It seems evidential that all the above mentioned problems with hand-crafted detectors for 3D point clouds can be resolved by the highly successful data-driven deep networks. 
However, very few deep learning-based 3D keypoint detectors exist (only one deep learning-based approach \cite{jian20183dfeat} exists to date) in contrast to its increasing success on learning 3D keypoint descriptors \cite{deng2018ppfnet, deng2018ppf, zeng20173dmatch, khoury2017learning}. This is due to the lack of ground truth training datasets to supervise deep learning-based detectors on 3D point clouds. Unlike 3D descriptors that are supervised by easily available ground truth registered overlapping 3D point clouds \cite{deng2018ppfnet, deng2018ppf, khoury2017learning, zeng20173dmatch, jian20183dfeat, gojcic2018perfect}, it is impossible for anyone to identify and label the ``ground truth'' keypoints on 3D point clouds. Consequently, most of the works on 3D descriptors \cite{deng2018ppfnet, deng2018ppf, zeng20173dmatch, khoury2017learning} ignored the detector problem and are built on top of existing hand-crafted 3D detectors or uniform sampling.

In view of the challenges on both hand-crafted and deep learning-based 3D detectors, we propose the USIP detector: an \textbf{\underline{U}}nsupervised \textbf{\underline{S}}table \textbf{\underline{I}}nterest \textbf{\underline{P}}oint deep learning-based detector that can detect highly repeatable, and accurately localized keypoints from 3D point clouds under arbitrary transformations \emph{without} the need for any ground truth training data. To this end, we design a Feature Proposal Network (FPN) that outputs a set of keypoints and their respective saliency uncertainties from an input 3D point cloud. 
Our FPN improves keypoint localization by estimating their positions on contrary to existing 3D detectors \cite{pcl, jian20183dfeat, zhong2009intrinsic} that select existing points in the point cloud as keypoints, which causes quantization errors.
During training, we apply randomly generated SE(3) transformations on each point cloud to get a set of corresponding pairs of transformed point clouds as inputs to the FPN. Furthermore, we identify and prevent the degeneracy of our USIP detector. 
We encourage high repeatability and accurate localization of the keypoints with a probabilistic chamfer loss that minimizes the distances between the detected keypoints from the training point cloud pairs. 
Additionally, we introduce a point-to-point loss to enforce the constraint of getting keypoints that lie close to the point cloud. 
We verify our USIP detector by performing extensive repeatability tests on several simulated and real-world benchmark 3D point cloud datasets from Lidar, RGB-D and CAD models. Some qualitative results are shown in Fig~\ref{fig_intro}.

\vspace{0.2cm}
\noindent Our key contributions are summarized as follows:
\begin{itemize}
    \item Our USIP detector is fully unsupervised, thus avoids the need for ground truth that are impossible to obtain. 
    \item We provide degeneracy analysis of our USIP detector and suggest solutions to prevent it. 
    \item Our FPN improves keypoint localization by estimating the keypoint position instead of choosing it from an existing point in the point cloud. 
    \item We introduce the probabilistic chamfer loss and point-to-point loss to encourage high repeatability and accurate keypoint localization. 
    \item The use of randomly generated transformations on point clouds during training inherently allows our network to achieve good performance under rotations.
\end{itemize}

\section{Related Work}  \label{sec_related_work}
Unlike the recent success of deep learning-based 3D keypoint descriptors \cite{deng2018ppfnet, deng2018ppf, khoury2017learning, zeng20173dmatch, jian20183dfeat, gojcic2018perfect}, most existing 3D keypoint detectors remain hand-crafted. A comprehensive review and evaluation of existing hand-crafted 3D keypoint detectors can be found in \cite{tombari2013performance}. 
Local Surface Patches (LSP) \cite{chen20073d} and Shape Index (SI) \cite{dorai1997cosmos} are based on the maximum and minimum principal curvatures of a point, and consider the point as a keypoint if it is a global extremum in a predefined neighborhood. Intrinsic Shape Signatures (ISS) \cite{zhong2009intrinsic} and KeyPoint Quality (KPG) \cite{mian2010repeatability} select salient points that has a local neighborhood with large variations along each principal axis. MeshDoG \cite{zaharescu2009surface} and Salient Points (SP) \cite{castellani2008sparse} construct a scale-space of the curvature with the Difference-of-Gaussian (DoG) operator similar to SIFT \cite{lowe2004distinctive}. Points with local extrema values over an one-ring neighborhood are selected as keypoints. These methods can be regarded as the 3D extension of SIFT. Laplace-Beltrami Scale-sapce (LBSS) \cite{unnikrishnan2008multi} computes the saliency by applying a Laplace-Beltrami operator on increasing supports for each point. 

More recently, LORAX \cite{elbaz20173d} proposes the method of projecting the point set into a depth map and use Principal Component Analysis (PCA) to select keypoints with commonly found geometric characteristics.
All hand-crafted approaches share the common trait of relying on the local geometric properties of the points to select keypoints. Hence, the performances of these detectors deteriorate under disturbances such as noise, density variations and/or arbitrary transformations. In contrast, our deep learning-based USIP detector is more resilient to these disturbances by learning from data. 
To the best of our knowledge, the only existing deep learning-based 3D keypoint detector is the weakly supervised 3DFeatNet \cite{jian20183dfeat}, which is trained with GPS/INS tagged point clouds. However, the training of 3Dfeat-Net is largely focused on learning discriminative descriptors using the Siamese architecture with an attention score map that estimates the saliency of each point as its by-product. It does not ensure good performance of the keypoint detection. In comparison, our USIP is designed to encourage high repeatability and accurate localization of the keypoints. Furthermore, our method is fully unsupervised and does not rely on any form of ground truth datasets.

\begin{figure*}[t!] \centering
\subfigure[]{\includegraphics[width=0.8\textwidth]{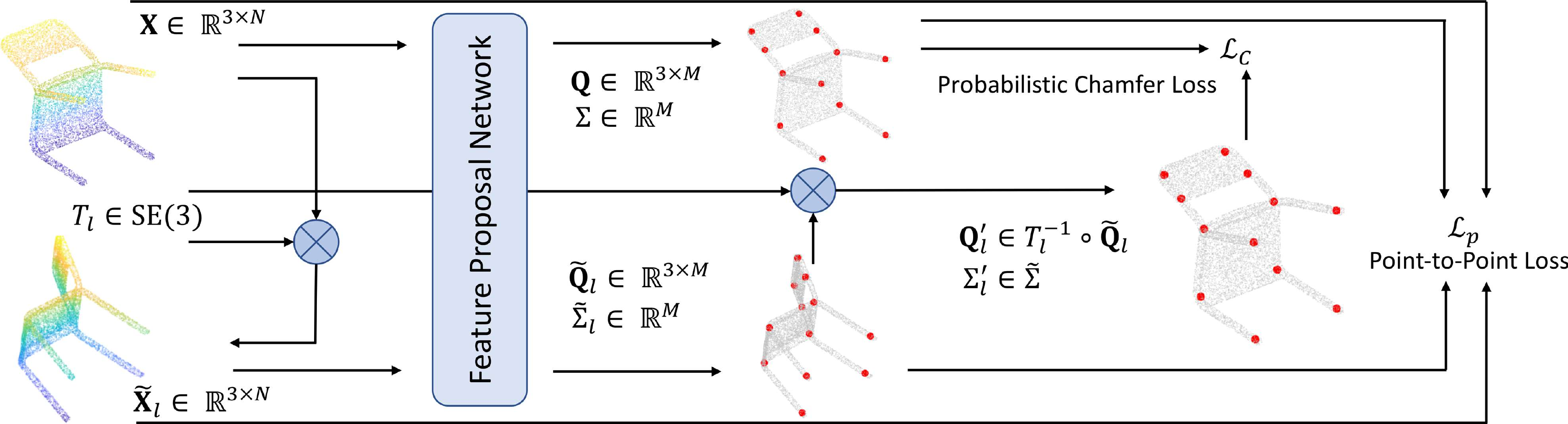}\label{fig_dec_pipeline}}
\subfigure[]{\includegraphics[width=0.90\textwidth]{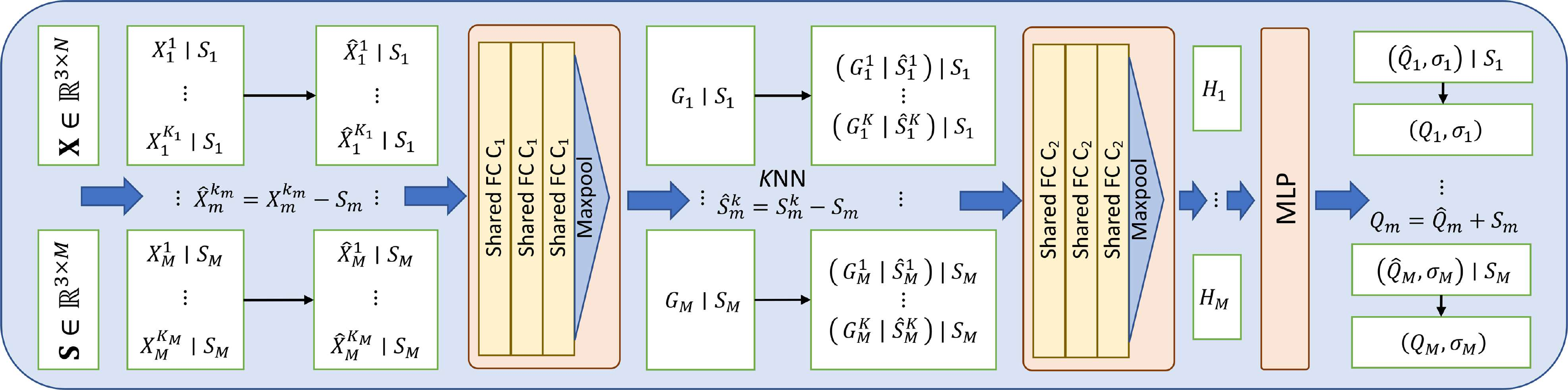}\label{fig_dec_fpn}}
\caption{(a) The training pipeline of USIP detector. (b) The architecture of our Feature Proposal Network (FPN). See text for more detail.}
\vspace{-8pt}
\end{figure*}

\section{Our USIP Detector} \label{sec_usip}
Fig.~\ref{fig_dec_pipeline} shows the illustration of the pipeline to train our USIP detector. We denote a point cloud from the training dataset as $\mathbf{X}=[X_0, \cdots, X_N] \in \mathbb{R}^{3 \times N}$. A set of transformation matrices $\{T_1, \cdots, T_L\}$, where $T_l \in \text{SE}(3)$ is randomly generated and applied to the point cloud $\mathbf{X}$ to form $L$ pairs of training inputs denoted as $\{ \{\mathbf{X}, \mathbf{\tilde{X}}_1\}, \cdots, \{\mathbf{X}, \mathbf{\tilde{X}}_L\} \}$, where $\mathbf{\tilde{X}}_l=T_l \circ \mathbf{X} \in \mathbb{R}^{3 \times N}$. Here, we use the operator $\circ$ to denote matrix multiplication under homogeneous coordinate with a slight abuse of notation. 
We drop the indices $l$ for brevity and refer to a triplet of training pair of point clouds and their corresponding transformation matrix as $\{\mathbf{X}, \mathbf{\tilde{X}}, T\}$. During training, $\mathbf{X}$ and $\mathbf{\tilde{X}}$ are respectively fed into the FPN, which outputs $M$ proposal keypoints and its saliency uncertainties denoted as $\{\mathbf{Q} = [Q_1, \cdots, Q_M], \Sigma = [\sigma_1, \cdots, \sigma_M]^T\}$ and $\{\mathbf{\tilde{Q}} = [\tilde{Q}_1, \cdots, \tilde{Q}_M], \tilde{\Sigma} = [\tilde{\sigma}_1, \cdots, \tilde{\sigma}_M]^T\}$ for the respective point cloud. $Q_m \in \mathbb{R}^3$, $\tilde{Q}_m \in \mathbb{R}^3$, 
$\sigma_m \in \mathbb{R}^+$ and $\tilde{\sigma}_m \in \mathbb{R}^+$. We enforce $\sigma_m \in \mathbb{R}^+$ and $\tilde{\sigma}_m \in \mathbb{R}^+$ so that it is a valid rate parameter in our probabilistic chamfer loss (see later paragraph).
To improve keypoint localization, it is not necessary for all $Q_m \in \mathbf{Q}$ to be any of the points in $\mathbf{X}$. Similar condition applies to all $\tilde{Q}_m \in \mathbf{\tilde{Q}}$. 

We undo the transformation on $\mathbf{\tilde{Q}}$ with a slight abuse of notation to get $\mathbf{Q}' = T^{-1} \circ \mathbf{\tilde{Q}} \in \mathbb{R}^{3 \times M}$, so that $\mathbf{Q}'$ can be compared directly to $\mathbf{Q}$. Here, we made an assumption that the saliency uncertainties remain unaffected after the transformation, \ie, $\Sigma' = \tilde{\Sigma}$. The objectives of detecting keypoints that are highly repeatable and accurately localized from 3D point clouds under arbitrary transformations can now be achieved by formulating a loss function that minimizes the difference between $\mathbf{Q}$ and $\mathbf{Q}'$. To this end, we propose the loss function: $\mathcal{L} = \mathcal{L}_{c} + \lambda\mathcal{L}_p$,
where $\mathcal{L}_{c}$ is the probabilistic chamfer loss that minimizes the probabilistic distances between all correspondence pairs of keypoints in $\mathbf{Q}$ and $\mathbf{Q}'$. $\mathcal{L}_p$ is the point-to-point loss that minimizes the distance of the estimated keypoints to their respective nearest neighbor in the point cloud. This constraints the estimated keypoints to be close to the point cloud. $\lambda$ is a hyperparameter that adjust the relative contribution of $\mathcal{L}_{c}$ and $\mathcal{L}_{p}$ to the total loss. More specifically:
\vspace{-0.4cm}
\paragraph{Probabilistic Chamfer Loss} \label{sec_loss_prob_chamfer}
A straightforward way to minimize the difference between $\mathbf{Q}$ and $\mathbf{Q}'$ is to use the chamfer loss:
\begin{equation} \label{equ_chamfer_vanilla}
    \sum_{i=1}^{M}\underset{Q'_j \in \mathbf{Q}'}{\text{min}}\|Q_i - Q'_j\|^2_2 + \sum_{j=1}^{M}\underset{Q_i \in \mathbf{Q}}{\text{min}}\|Q_i - Q'_j\|^2_2,
\end{equation}

\noindent that minimizes the distance of each point in one point cloud with its nearest neighbor in the other point cloud. However, the $M$ proposals are not equally salient. The receptive field of a point $Q_i$ can be a featureless surface since the receptive field is limited to a small volume. In this case, it is detrimental to force the FPN to minimize the distance between  $Q_i$ and $Q'_j$, where $Q'_j$ is the nearest neighbor of $Q_i$ in $\mathbf{Q}'$. 

To mitigate the above problem, we design our FPN to learn the saliency uncertainties $\Sigma$ and $\Sigma'$ of the proposal keypoints $\mathbf{Q}$ and $\mathbf{Q}'$ with a probabilistic chamfer loss $\mathcal{L}_c$.
In particular, we propose to formulate $\mathcal{L}_c$ with an exponential distribution that measures the probabilistic distances between $\mathbf{Q}$ and $\mathbf{Q}'$ with the saliency uncertainties $\Sigma$ and $\Sigma'$.
More formally, the probability distribution between $Q_i$ and $Q'_{j}$ for $i=1, \cdots, M$ is given by:
\begin{equation} \label{equ_exponential_dist}
\begin{split}
    p(d_{ij} \mid \sigma_{ij}) = \frac{1}{\sigma_{ij}} \exp{\bigg(-\frac{d_{ij}}{\sigma_{ij}}\bigg)},~~~~~~~~\text{where} \\
    \sigma_{ij}=\frac{\sigma_i + \sigma'_j}{2} > 0,~~~d_{ij}=\underset{Q'_j \in \mathbf{Q}'}{\text{min}}\|Q_i-Q'_j\|_2 \geq 0.
\end{split}
\end{equation}

\noindent $p(d_{ij} \mid \sigma_{ij})$ is a valid probability distribution since it integrates to 1. A shorter distance $d_{ij}$ between the proposal keypoints $Q_i$ and $Q'_j$ gives a higher probability that $Q_i$ and $Q'_j$ are highly repeatable and accurately localized keypoints in the point clouds $\mathbf{X}$ and $\tilde{\mathbf{X}}$. 
Assuming i.i.d for all $d_{ij} \in D_{ij}$, the joint distribution between $\mathbf{Q}$ and $\mathbf{Q}'$ is given by:
\begin{equation} \label{eq:jointDist_ij}
p(D_{ij} \mid \Sigma_{ij}) = \prod_{i=1}^{M} p(d_{ij} \mid \sigma_{ij}).
\end{equation}

\noindent It is important to note that the probability distribution is not symmetrical when the order of the point cloud is swapped, \ie, $\mathbf{Q}'$ and $\mathbf{Q}$, due to a different set of nearest neighbors, \ie, $d_{ij} \neq d_{ji}$ and $\sigma_{ij} \neq \sigma_{ji}$. Hence, the joint distribution between $\mathbf{Q}'$ and $\mathbf{Q}$ is given by:
\begin{equation} \label{eq:jointDist_ji}
\begin{split}
p(D_{ji} \mid \Sigma_{ji}) = \prod_{j=1}^{M} p(d_{ji} \mid \sigma_{ji}),~~~~~~~~\text{where} \\
    \sigma_{ji}=\frac{\sigma'_j + \sigma_i}{2} > 0,~~~d_{ji}=\underset{Q_i \in \mathbf{Q}}{\text{min}}\|Q_i-Q'_j\|_2 \geq 0.
\end{split}
\end{equation}


Finally, the probabilistic chamfer loss $\mathcal{L}_c$ between $\mathbf{Q}'$ and $\mathbf{Q}$ is given by the negative log-likelihood of the joint distributions defined in Eq.~\ref{eq:jointDist_ij} and \ref{eq:jointDist_ji}:  
\begin{align} \label{equ_prob_chamfer}
    \mathcal{L}_c & = \sum_{i=1}^{M}-\ln p(d_{ij} \mid \sigma_{ij}) + \sum_{j=1}^{M}-\ln p(d_{ji} \mid \sigma_{ji}) \nonumber\\
        & = \sum_{i=1}^{M} \bigg(\ln \sigma_{ij} + \frac{d_{ij}}{\sigma_{ij}}\bigg) + \sum_{j=1}^{M} \bigg(\ln \sigma_{ji} + \frac{d_{ji}}{\sigma_{ji}}\bigg).
\end{align}

\noindent We further analyze the physical meaning of $\sigma_{ij}$ or $\sigma_{ji}$ by computing the extrema of Eq.~\ref{equ_exponential_dist} from its first derivative over $\sigma_{ij}$: 
\begin{equation}
    \frac{\partial p(d_{ij} \mid \sigma_{ij})}{\partial \sigma_{ij}} = \frac{d_{ij}\exp(-d_{ij}/\sigma_{ij})}{\sigma_{ij}^3} - \frac{\exp(-d_{ij}/\sigma_{ij})}{\sigma_{ij}^2},
\end{equation}
\noindent and solve for the stationary points:
\begin{equation}
    \frac{\partial p(d_{ij} \mid \sigma_{ij})}{\partial \sigma_{ij}} = 0 \Rightarrow \sigma_{ij}=d_{ij}.
\end{equation}

\noindent Furthermore, the second derivative $p''(d_{ij}\mid\sigma_{ij}) |_{\sigma_{ij}=d_{ij}}<0$ means that given a fixed $d_{ij} \neq 0$, the highest probability $p(d_{ij} \mid \sigma_{ij})$ is achieved at $\sigma_{ij}=d_{ij}$. Consider any triplet of proposal keypoints $\{Q_i, Q'_j, Q'_k\}$, where $d_{ij}$ and $d_{ki}$ are the distances between the nearest neighbors $\{Q_i, Q'_j\}$ and $\{Q'_k, Q_i\}$ ($Q_i$ can be the nearest neighbor in both orders of $\mathbf{Q}$ and $\mathbf{Q}'$ since chamfer distance is not bijective). $\sigma'_k$ has to take a large value when $d_{ij} \to 0$ and $d_{kj}$ is large because we have shown that $\sigma_{ij}=d_{ij}$ and $\sigma_{ki}=d_{ki}$ at optimum. Furthermore, $d_{ij} \to 0$ and $d_{kj}$ is large implies that $\{Q_i, Q'_j\}$ are repeatable and accurately localized keypoints while $Q'_k$ is not. Hence, a large saliency uncertainty $\sigma'_k$ for a bad proposal keypoint $Q'_k$ at optimum shows that our probabilistic chamfer loss is guiding the FPN to learn correctly.

\paragraph{Point-to-Point Loss}\label{sec_loss_faeture_on_pc}
To avoid quantization error in the positions of the keypoints, we design the FPN such that it is not necessary that the proposal keypoints $\mathbf{Q}$ to be any of the points in $\mathbf{X}$. However, this can cause the FPN to give erroneous proposal keypoints $\mathbf{Q}$ that are far away from the point cloud $\mathbf{X}$. We circumvent this problem by adding a loss function $\mathcal{L}_p$ that penalizes $Q_m \in \mathbf{Q}$ for being too far from $\mathbf{X}$. We also apply similar penalty on $\tilde{\mathbf{Q}}$ and $\tilde{\mathbf{X}}$. 
This loss can be formulated as either the point-to-point loss \cite{besl1992method}:
\begin{equation} \label{equ_onpc_point}
\mathcal{L}_{\text{point}} =
    \sum_{i=1}^{M}\underset{X_j \in \mathbf{X}}{\text{min}}\|Q_i - X_j\|_2^2 + \sum_{i=1}^{M}\underset{\tilde{X}_j \in \tilde{\mathbf{X}}}{\text{min}}\|\tilde{Q}_i - \tilde{X}_j\|_2^2,
\end{equation}
\noindent where $X_j \in \mathbf{X}$ is the nearest neighbor of $Q_i$ or the point-to-plane loss \cite{rusinkiewicz2001efficient,chen1992object}: 
\begin{equation}\label{equ_onpc_plane}
    \mathcal{L}_{\text{plane}} =
    \sum_{i=1}^{M} \mathcal{N}_j^T(Q_i - X_j) + \sum_{i=1}^{M} \tilde{\mathcal{N}}_j^T(\tilde{Q}_i - \tilde{X}_j),
\end{equation} 
where $\mathcal{N}_j$ and $\tilde{\mathcal{N}}_j$ are the nearest surface normal in $\mathbf{X}$ to $Q_i$ and $\tilde{\mathbf{X}}$ to $\tilde{Q}_i$, respectively.
We set $\mathcal{L}_p = \mathcal{L}_{\text{point}}$ by default since we found experimentally that both loss functions give similar performances.

\section{Feature Proposal Network} \label{sec_det_fpn}
The network architecture of our FPN is shown in Fig.~\ref{fig_dec_fpn}.
We first sample $M$ nodes denoted as $\mathbf{S} = [S_1, \cdots, S_M] \in \mathbb{R}^{3 \times M}$ with Farthest Point Sampling (FPS) from a given input point cloud $\mathbf{X}\in \mathbb{R}^{3\times N}$. 
A neighborhood of points is built for each node $S_m \in \mathbf{S}$ using point-to-node grouping \cite{li2018so,kohonen1998self}, which is denoted as $\{\{X_1^1|S_1, ..., X_1^{K_1}|S_1\},~\cdots,~\{X_M^1|S_M,~...,~X_M^{K_M}|S_M\}\}$. $K_1, \cdots, K_M$ represents the number of points associated with the each of the nodes in $\mathbf{S}$.
The advantage of point-to-node association over node-to-point $k$NN search or radius-based ball-search is two-fold: (1) Every point in $\mathbf{X}$ is associated with one node, while some points may be left out
in node-to-point $k$NN search and ball-search. (2) Point-to-node grouping automatically adapts to various scale and point density, while $k$NN search and ball-search are vulnerable to density variation and varying scales, respectively. To make FPN \textit{translation equivariant}, we normalize each neighborhood point $\{X_m^{1}|S_m, \cdots, X_m^{K_m}|S_m\}$ into  $\{\hat{X}_m^{1}|S_m, \cdots, \hat{X}_m^{K_m}|S_m\}$ by subtracting from its respective node $S_m$, \ie, $\hat{X}_m^{k_m}=X_m^{k_m} - S_m$. Each cluster of normalized local neighborhood points is then fed into a PointNet-like network \cite{qi2016pointnet} shown in Fig.~\ref{fig_dec_fpn} to get a local feature vector $G_m$ associated with $S_m$. A $k$NN grouping layer is applied on the set of local feature vectors $\{G_1|S_1, \cdots, G_M|S_M\}$ to achieve hierarchical information aggregation. Specifically,
the $k$ nearest neighbors of each pair of $(G_m | S_m)$ are retrieved as $\{(G_m^1 | S_m^1) | S_m, \cdots, (G_m^K | S_m^K) | S_m\}$. 
These $k$NN local feature vectors are then normalized by subtracting with its respective $S_m$ to get a position-independent neighborhood denoted as $\{G_m^1 | \hat{S}_m^K) | S_m, \cdots, (G_m^K | \hat{S}_m^K) | S_m\}$, where $\hat{S}_m^K = S_m^K - S_m$,
before feeding into another network to get a set of feature vectors $\{H_1, \cdots, H_M\}$. 
A simple Multi-Layer Perceptron (MLP) is then used to estimate
$M$ proposal keypoints $\{\hat{Q}_1 | S_1, \cdots, \hat{Q}_M | S_M\}$, where $\hat{Q}_m \in \mathbb{R}^3$, and saliency uncertainties $\{\sigma_1, \cdots, \sigma_M\}$, where $\sigma_m \in \mathbb{R}^+$ from $\{H_1, \cdots, H_M\}$. Finally, we un-normalize each $\hat{Q}_m$ with $S_m$, \ie, $Q_m = \hat{Q}_m + S_m$ to get the final proposal keypoints $\{Q_1, \cdots, Q_M\}$.   
It is important to note that the size of the receptive field is controlled by the number of proposals $M$ and $K$ in $k$NN layers and it determines the level-of-detail for each feature. Large receptive field leads to features that are salient on a large-scale and vice versa.

\section{Degeneracy Analysis} \label{sec_degeneration}
Let us denote the FPN as $f(\mathbf{Y}):\mathbf{Y}\rightarrow \mathbb{R}^{3 \times M}$, where $\mathbf{Y} = [Y_1, \cdots, Y_N] \in \mathbb{R}^{3 \times N}$ is the input of the network. We further denote a transformation matrix $T \in \text{SE}(3)$, where $R \in \text{SO}(3)$ and $t \in \mathbb{R}^3$ are the rotation matrix and translation vector in $T$. 
We get 
$\mathbf{Y}' = R\mathbf{Y} \oplus t$, where $\oplus$ is the operator to denote the addition of $t$ to every $3 \times 1$ entries of the other term. We say that the network is degenerate when it outputs \textbf{\emph{trivial solutions}} where $f(\mathbf{Y}') \equiv Rf(\mathbf{Y}) \oplus t$ is satisfied for all $R$ and $t$.

\begin{lemma}
$f(\mathbf{Y}') \equiv Rf(\mathbf{Y}) \oplus t$ when $f(.)$ outputs the \textbf{centroid} of the input point cloud, \ie, $f(\mathbf{Y}) = \frac{1}{N}\sum_n Y_n$ and $f(\mathbf{Y}') = \frac{1}{N}\sum_n Y'_n$.
\end{lemma}

\begin{proof}
Putting $Y'_n = RY_n + t$ into $f(\mathbf{Y}') = \frac{1}{N}\sum_n Y'_n$, we get $f(\mathbf{Y}') = \frac{1}{N}\sum_n (RY_n + t) = R(\frac{1}{N}\sum_n Y_n) + t = Rf(\mathbf{Y}) \oplus t $. Hence, $f(Y') \equiv Rf(Y) \oplus t$ which completes our proof that the network degenerates when it outputs the centroid of the input point cloud.
\end{proof}

\begin{lemma}
$f(\mathbf{Y}') \equiv Rf(\mathbf{Y}) \oplus t$ when $f(.)$ is translational equivariant, \ie, $f(\cdot) \oplus t = f(\cdot \oplus t)$, and outputs points that are in the linear subspace of any \textbf{principal axis} from the input point cloud denoted as $\mathbf{U} = [U_1, U_2, U_3] \in \mathbb{R}^{3 \times 3}$, \ie, $f(\mathbf{Y}) = [c_1U_i^T, \cdots, c_MU_i^T]^T$ and
\begin{align}
\begin{split}
f(\mathbf{Y}') &= f(R\mathbf{Y} \oplus t) \\ &= f(R\mathbf{Y}) \oplus t~~~~~~~~~~~~~~~\text{\footnotesize{(translation equivariance)}} \\
&= [c_1{U'_i}^T, \cdots, c_M{U'_i}^T]^T \oplus t,        
\end{split}
\end{align}
where $U_i$ can be any principal axis in $\mathbf{U}$ and $c_1, \cdots, c_M$ are scalar coefficients in $\mathbb{R}$.
\end{lemma}

\begin{proof}
Let $V = \frac{1}{N}\sum_n (Y_n - \Bar{Y}) (Y_n - \Bar{Y})^T$ and $V' = \frac{1}{N}\sum_n (Y'_n - \Bar{Y}') (Y'_n - \Bar{Y}')^T$ denote the covariance matrices of $\mathbf{Y}$ and $\mathbf{Y}'$, respectively. $\Bar{Y} = \frac{1}{N}\sum_n Y_n$ and $\Bar{Y}' = \frac{1}{N}\sum_n Y'_n$ are the centroids of $\mathbf{Y}$ and $\mathbf{Y}'$, respectively. Putting $Y'_n = RY_n + t$ into $\Bar{Y}'$ and $V'$, we get:
\begin{equation}
    V' 
       = R \frac{1}{N}\sum_n (Y_n - \Bar{Y}) (Y_n - \Bar{Y})^T R^T = RVR^T.
\label{eq:covariance}
\end{equation}
\noindent Taking the Singular Value Decomposition (SVD) of $V$ and $V'$, we get $V=\mathbf{U}\mathbf{D}\mathbf{U}^T$ and $V'=\mathbf{U}'\mathbf{D}'\mathbf{U}'^T$, where $\mathbf{D}$ and $\mathbf{D}'$ are the $3 \times 3$ diagonal matrices of singular values, and $\mathbf{U}$ and $\mathbf{U}'$ are the $3 \times 3$ Eigenvectors that are also the principal axes of $\mathbf{Y}$ and $\mathbf{Y}'$, respectively. Putting the SVD of $V$ and $V'$ into Eq.~\ref{eq:covariance}, we get:
\begin{align}
\begin{split}
    V' &= RVR^T = R\mathbf{U}\mathbf{D}\mathbf{U}^TR^T = (R\mathbf{U})\mathbf{D}(R\mathbf{U})^T \\
    &\equiv \mathbf{U}'\mathbf{D}'\mathbf{U}'^T~~~\Rightarrow \mathbf{U}' = R\mathbf{U}.
\end{split} \label{eq:uRelation}
\end{align}
Putting the relationship from Eq.~\ref{eq:uRelation} into $f(\mathbf{Y}') = [c_1 {U'_i}^T, \cdots, c_M {U'_i}^T]^T \oplus t$, we get:
\begin{equation}
f(\mathbf{Y}') = R[c_1U_i^T, \cdots, c_MU_i^T]^T \oplus t \equiv Rf(\mathbf{Y}) \oplus t,
\label{eq:principalDegeneracy}
\end{equation}
which completes our proof that the network degenerates when it outputs a set of points on any principal axis. 
\end{proof}

\paragraph{Discussions} We note that the network requires sufficient global semantic information of the input point cloud, \eg, the input is the whole point cloud or clusters of local neighbor points that contain large receptive fields, to learn the trivial solutions of centroid or set of points on the principal axes.
Hence, the degeneracies can be easily prevented by limiting the receptive fields of the FPN. We achieve this by setting the number of clusters $M$ and $K$ nearest neighbors of the clusters in the FPN (refer to Sec.~\ref{sec_det_fpn} for the definitions of $M$ and $K$) to reasonable values. Small values for $M$ or high values for $K$ increases the receptive field and causes the FPN to degenerate. Fig.~\ref{fig_wine_bottle_3d_2d} show some examples of the degeneracies with different $K$ values at $M=64$. It is interesting to note that the principal axis degeneracy occurs when $K$ is set to a mid-range value, and centroid degeneracy occurs when $K$ is set to a high value. This implies that larger receptive fields, \ie, a higher global semantic information is needed for the network to learn the centroid. We also notice experimentally that the degeneracies (both centroid and principal axis) occur in point clouds with more regular shapes, \eg objects from ModelNet40 where the centroid and principal axes are more well-defined.
\begin{figure}[h!] \centering
\subfigure{\includegraphics[width=0.15\textwidth]{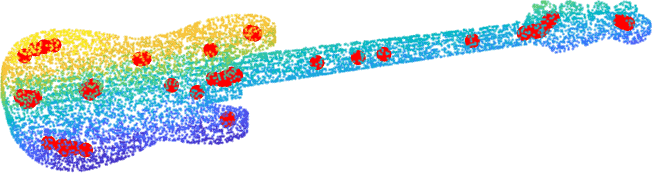}}
\subfigure{\includegraphics[width=0.15\textwidth]{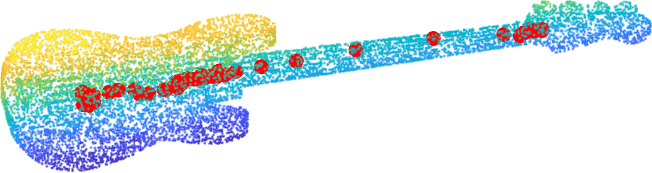}}
\subfigure{\includegraphics[width=0.15\textwidth]{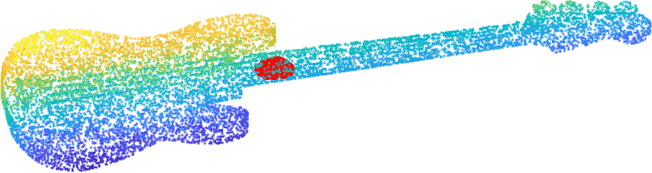}}
\setcounter{subfigure}{0}
\subfigure[]{\includegraphics[width=0.15\textwidth]{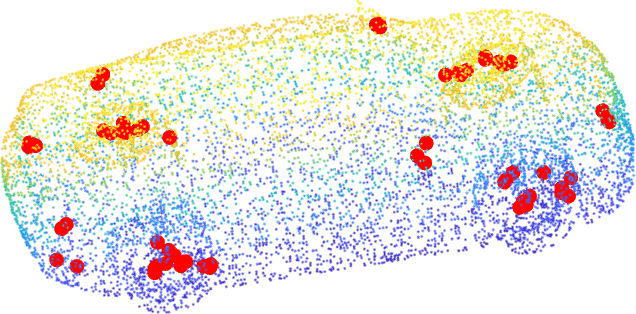}}
\subfigure[]{\includegraphics[width=0.15\textwidth]{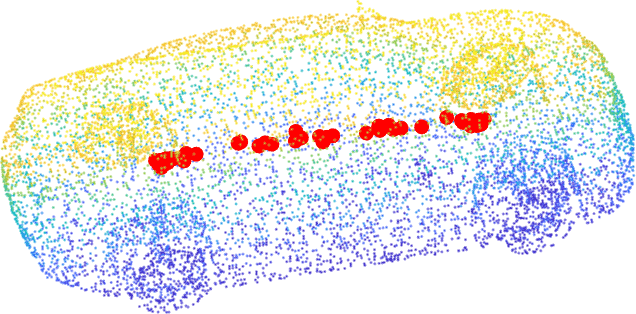}}
\subfigure[]{\includegraphics[width=0.15\textwidth]{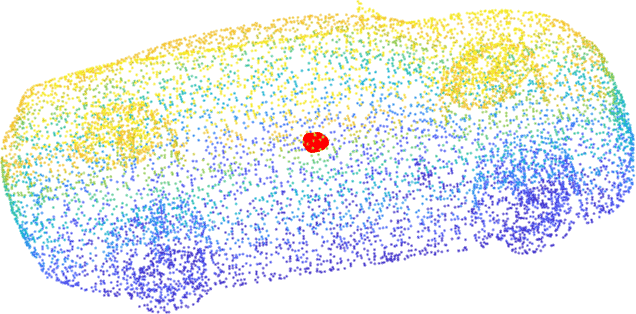}}
\caption{Increasing $K$ values in FPN causes degeneracies ($M=64$). (a) No degeneracy with $K=9$ (low value). (b) Principal axis degeneracy with $K=24$ (mid-range value). (c) Centroid degeneracy with $K=64$ (high value).} \label{fig_wine_bottle_3d_2d}
\vspace{-8pt}
\end{figure}

\section{Experiments} \label{sec_experiments}
\begin{figure*}[t!] \centering
    \includegraphics[width=0.24\textwidth]{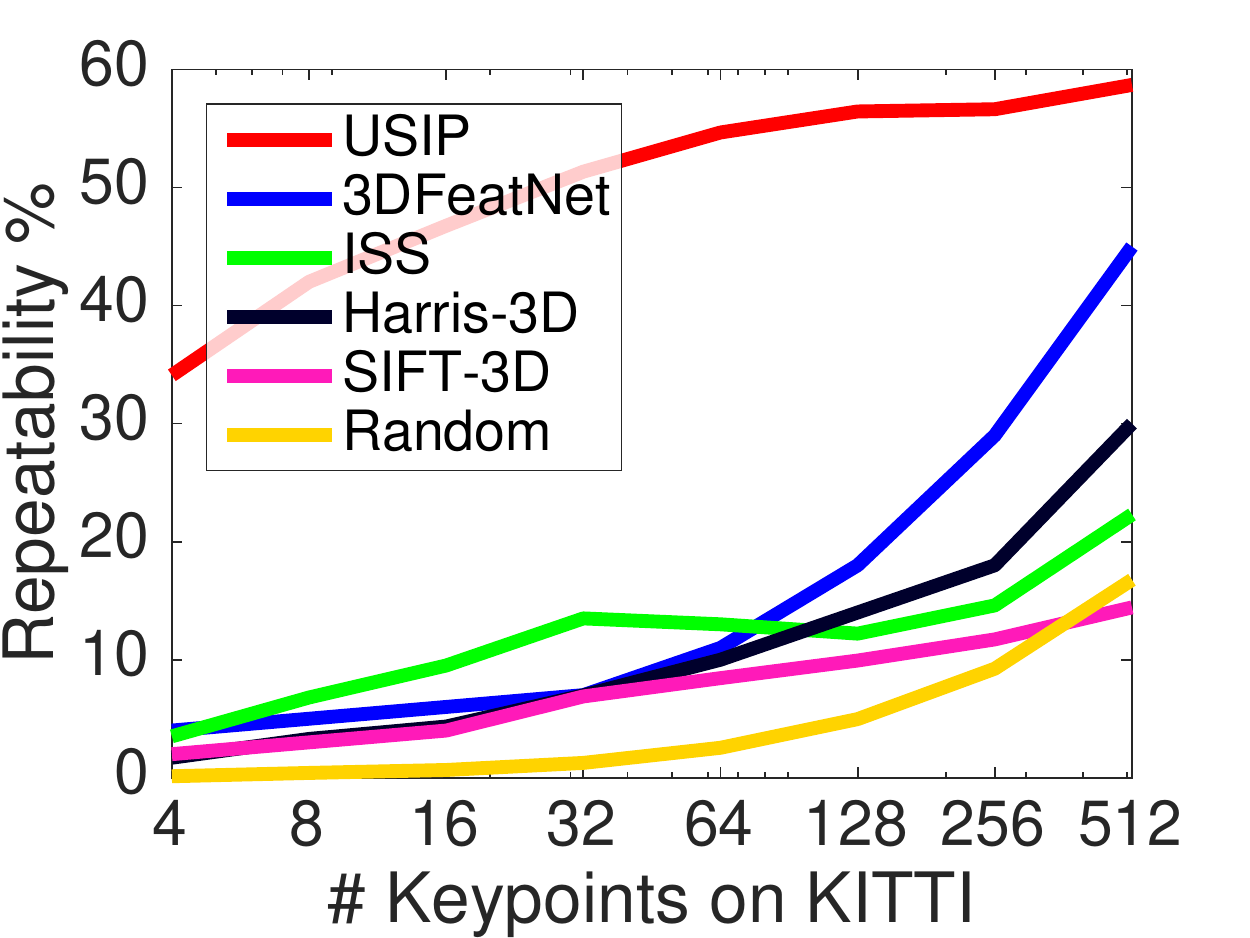}
    \includegraphics[width=0.24\textwidth]{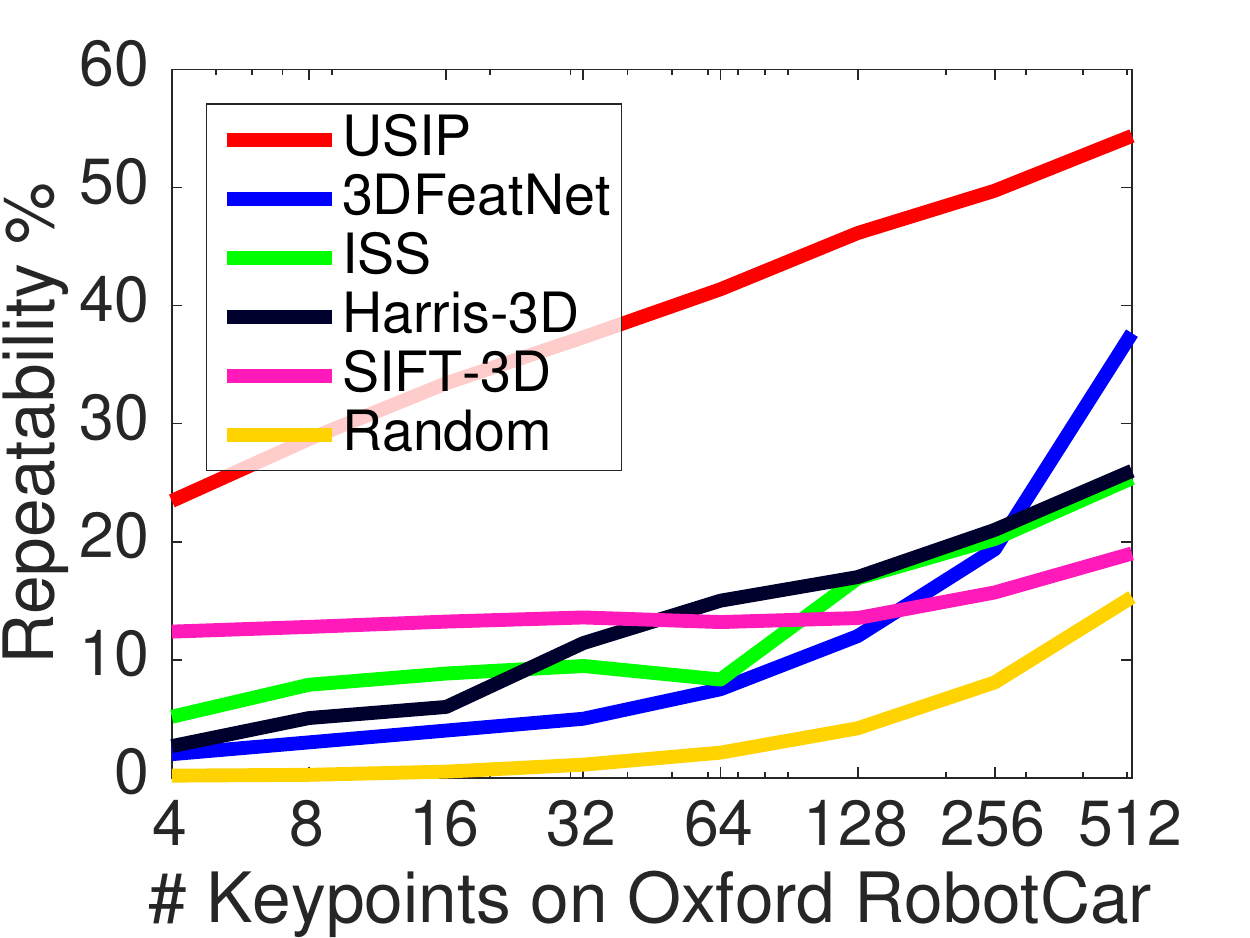}
    \includegraphics[width=0.24\textwidth]{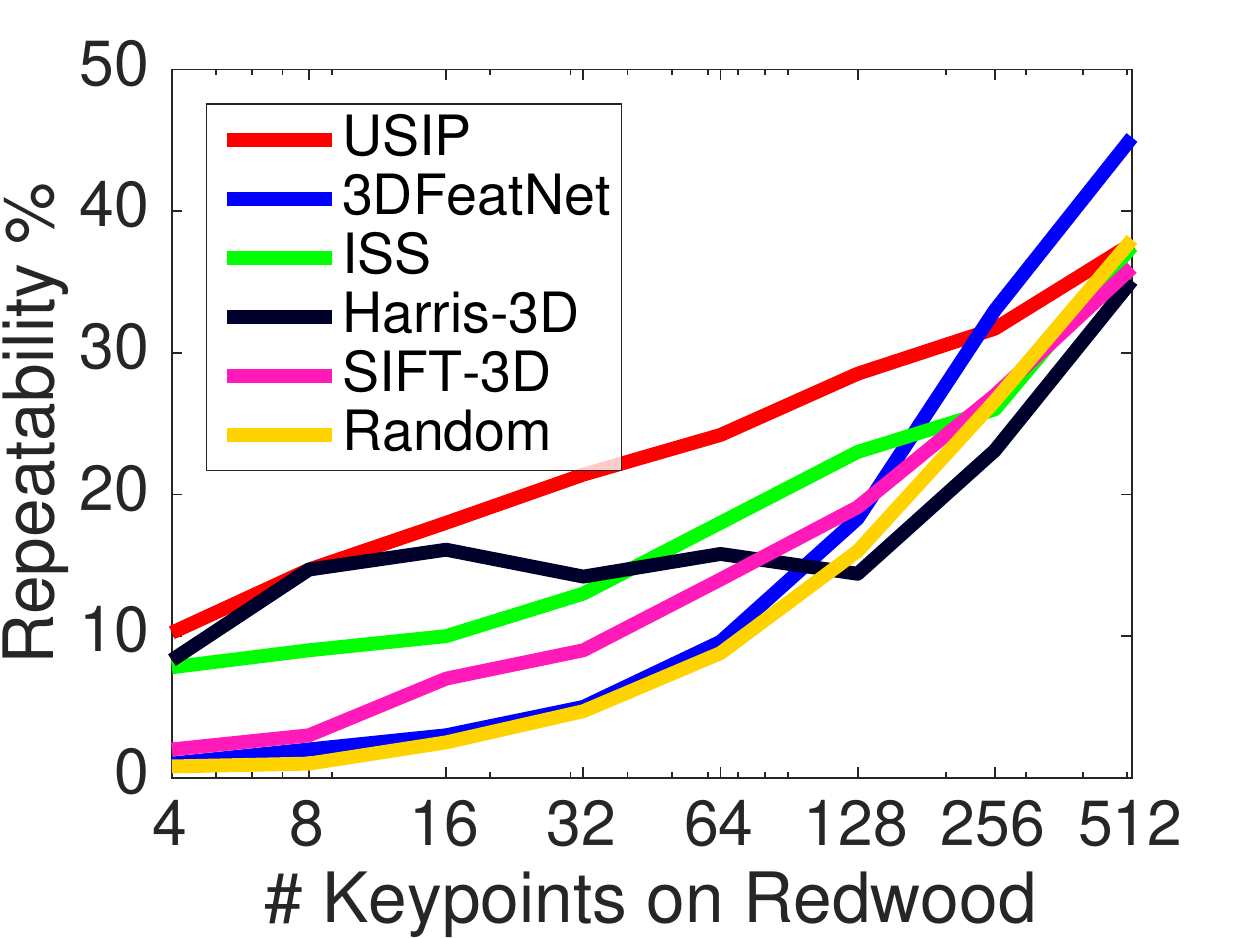}
    \includegraphics[width=0.24\textwidth]{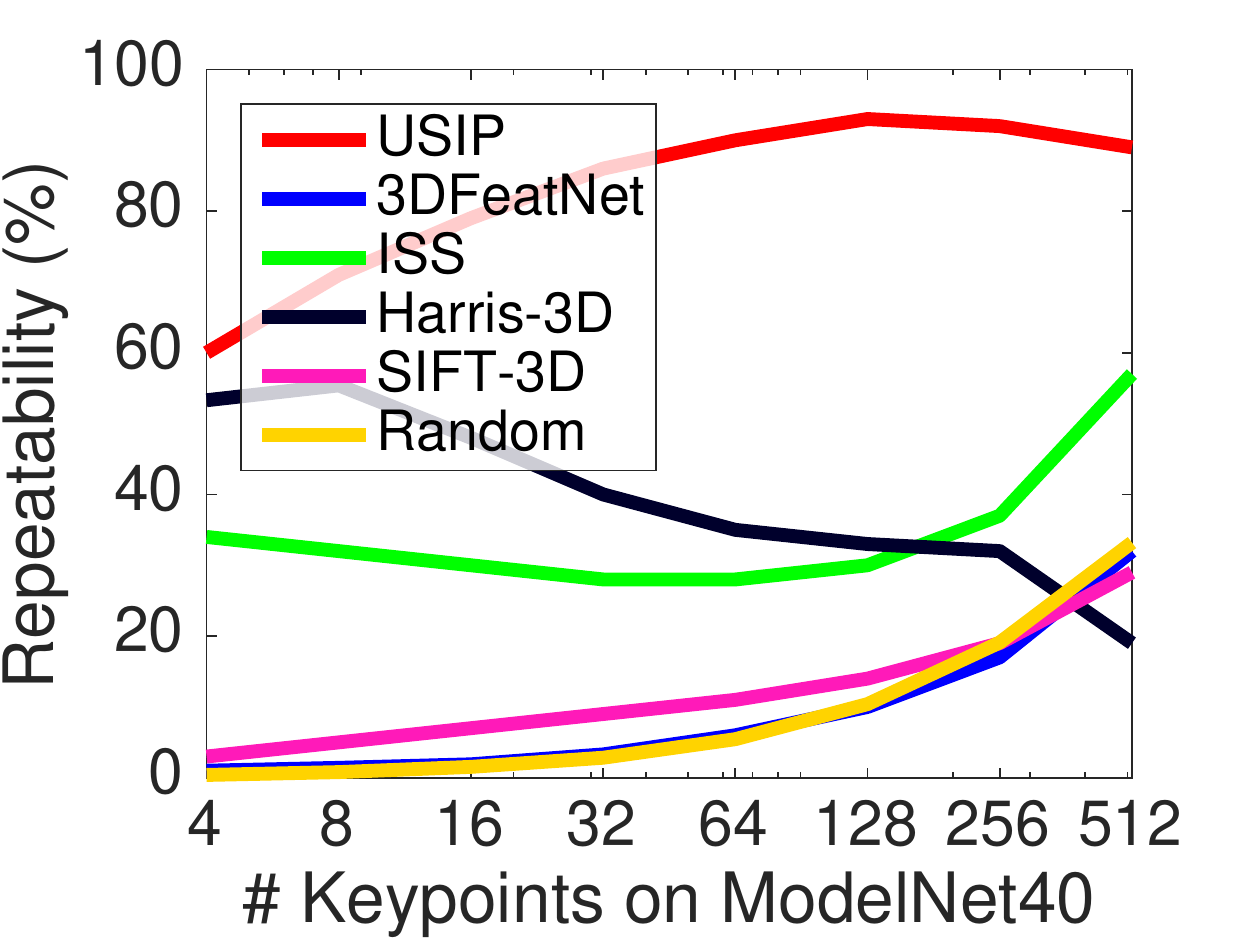}
    \caption{Relative repeatability when different number of keypoints are detected. Left to right: KITTI, Oxford, Redwood, ModelNet40.} \label{fig_repeatability}
    \vspace{-8pt}
\end{figure*} 
\begin{figure*}[t!] \centering
    \includegraphics[width=0.24\textwidth]{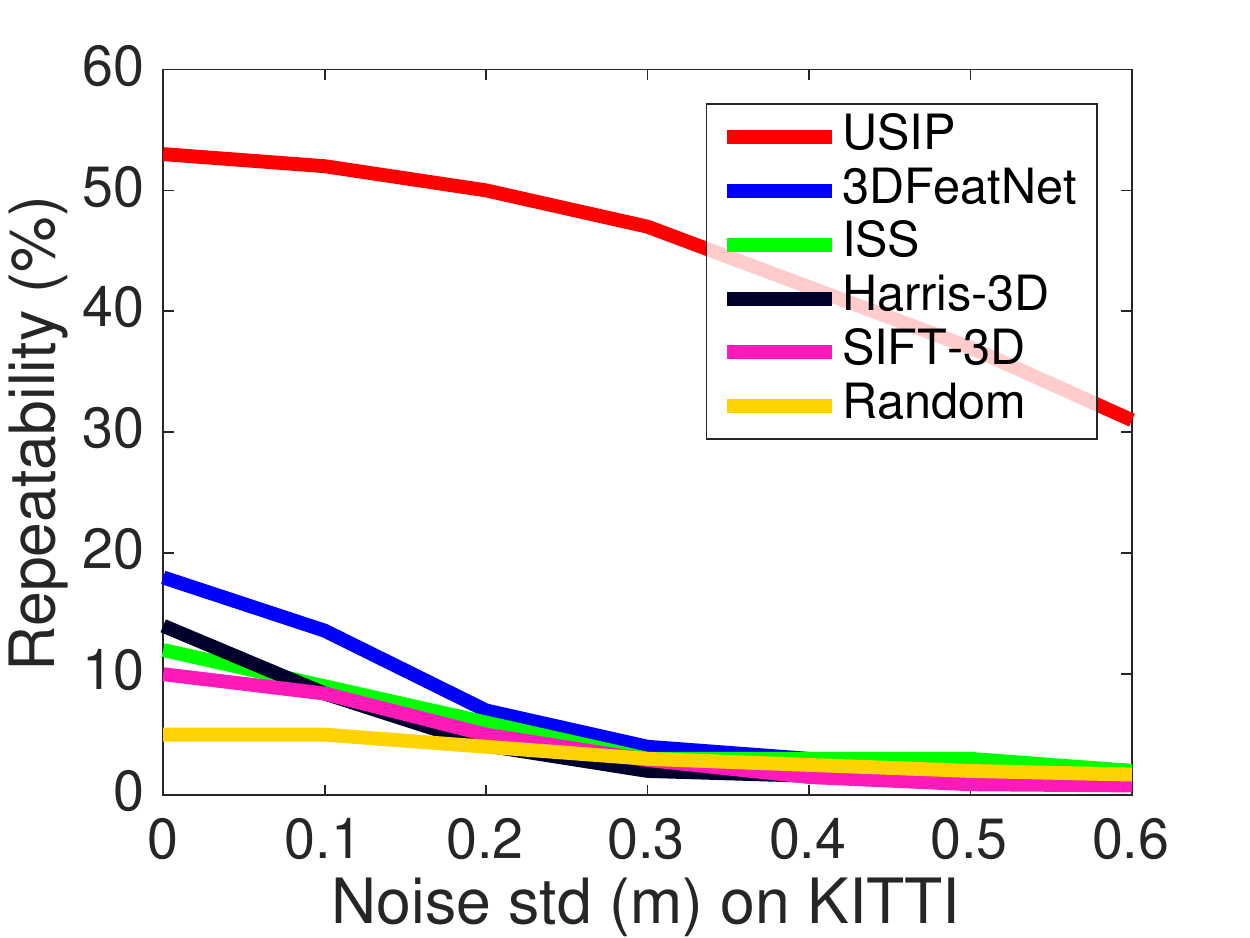}
    \includegraphics[width=0.24\textwidth]{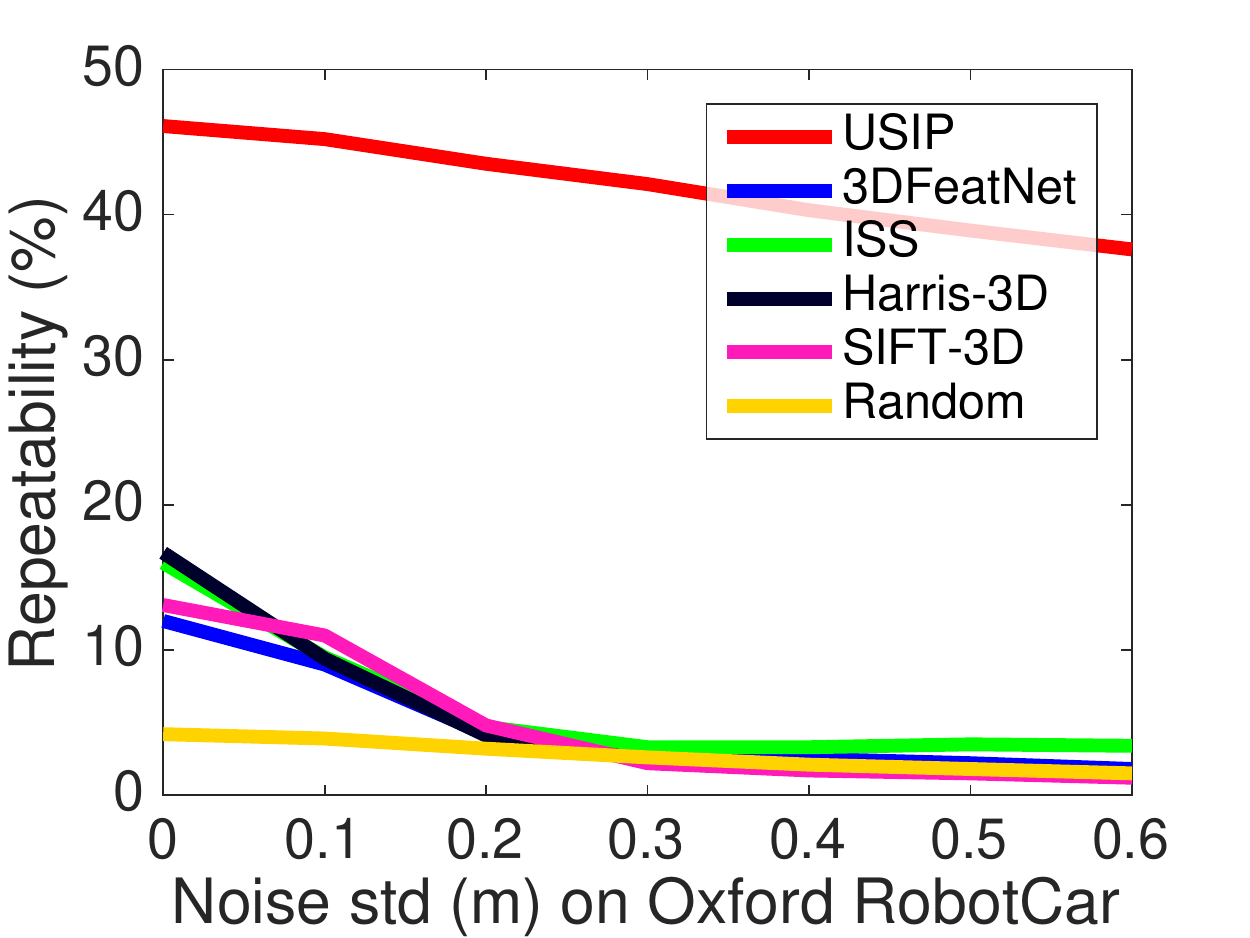}
    \includegraphics[width=0.24\textwidth]{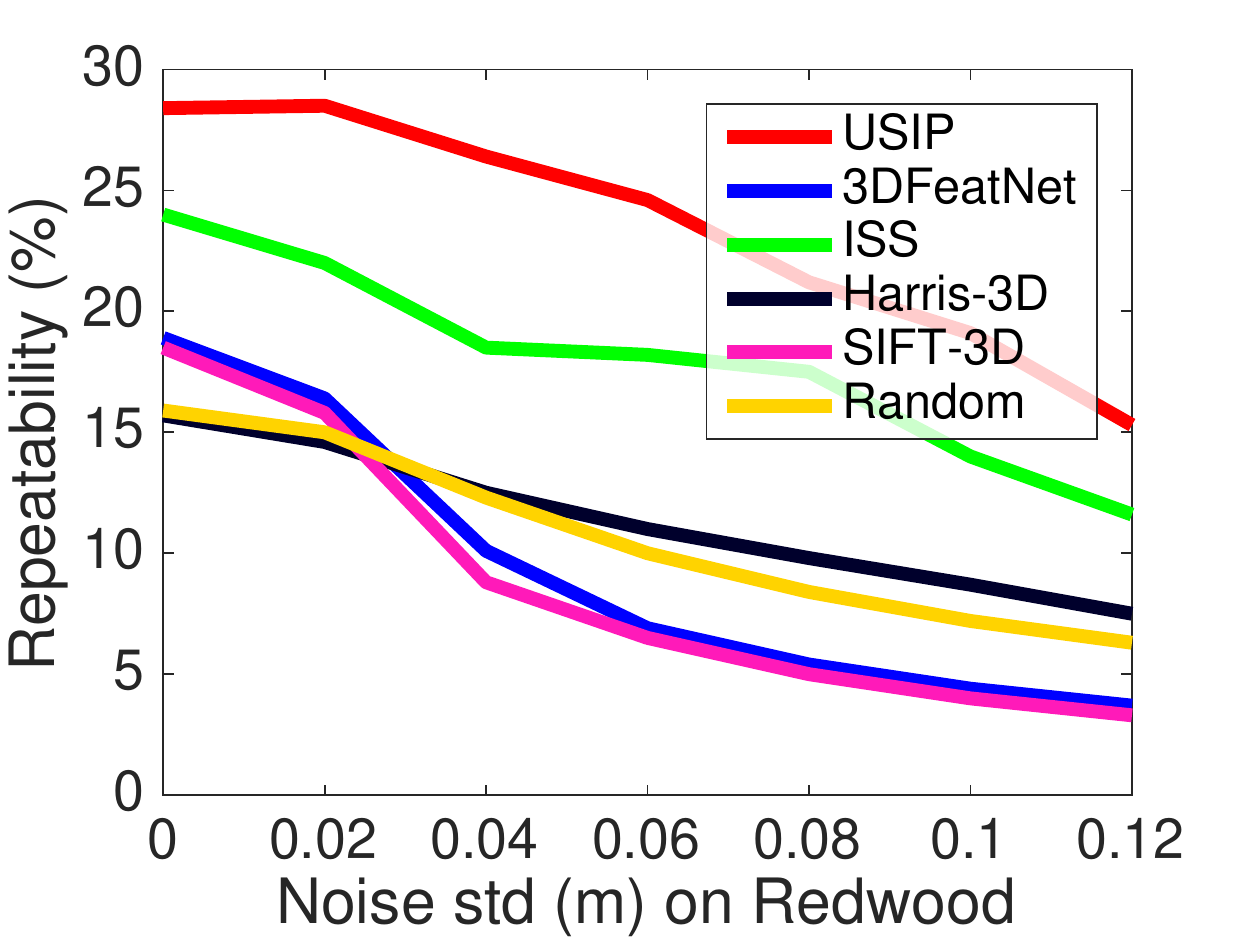}
    \includegraphics[width=0.24\textwidth]{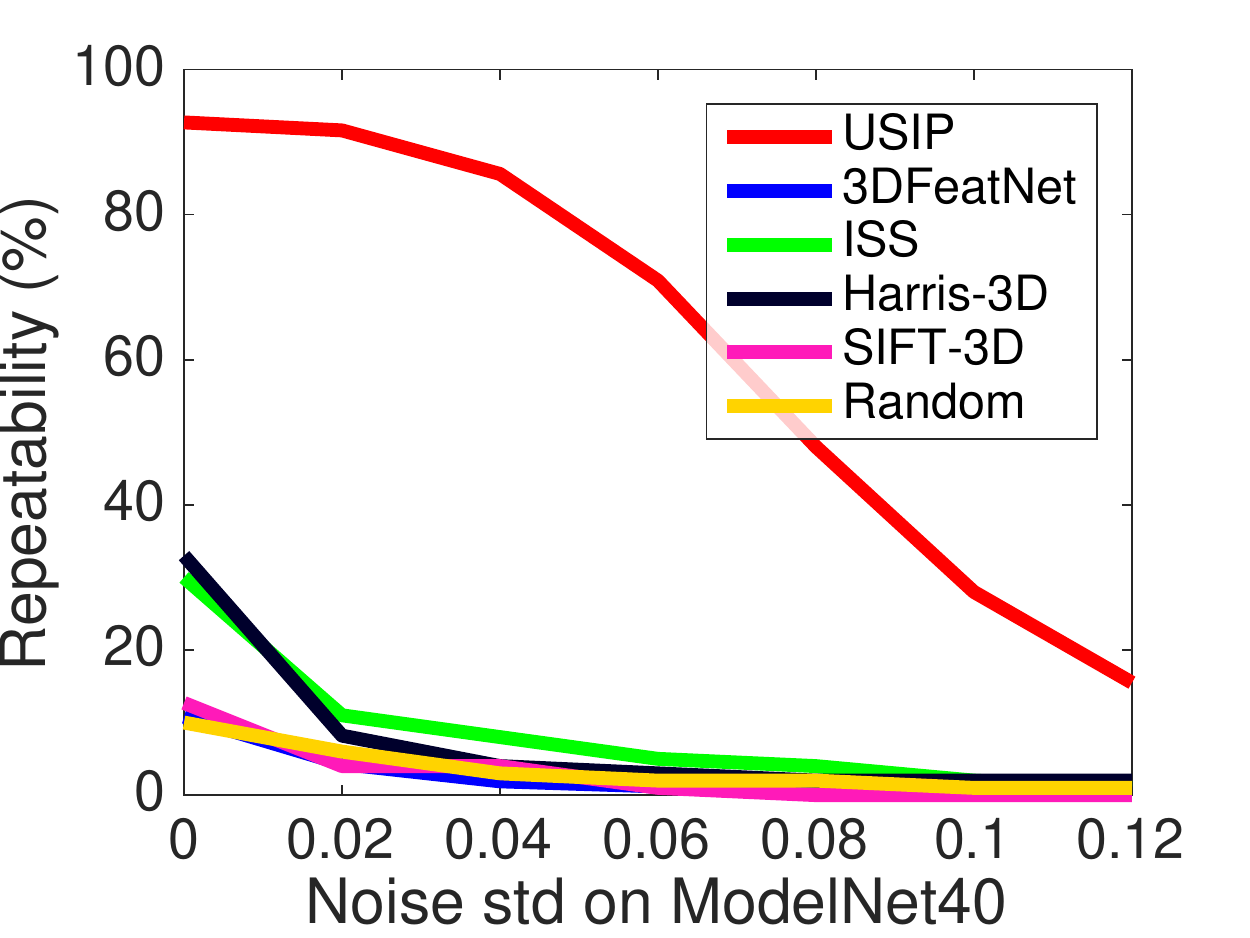}
    \caption{Relative repeatability when Gaussian noise $\mathcal{N}(0, \sigma_{noise})$ is added to the input point clouds. Keypoint number is fixed to 128. } \label{fig_repeatability_noisy}
    \vspace{-8pt}
\end{figure*} 
\begin{figure*}[t!] \centering
    \includegraphics[width=0.24\textwidth]{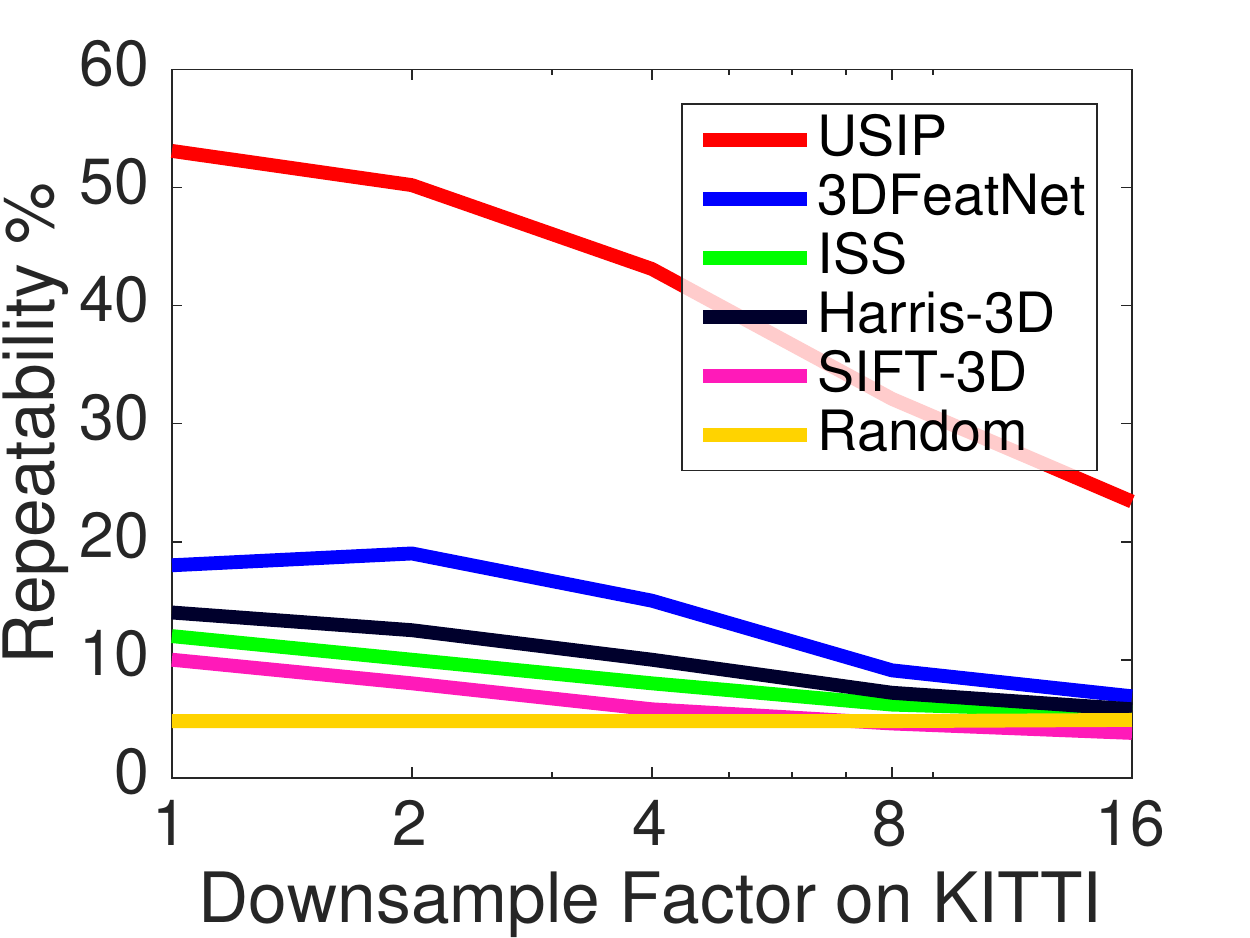}
    \includegraphics[width=0.24\textwidth]{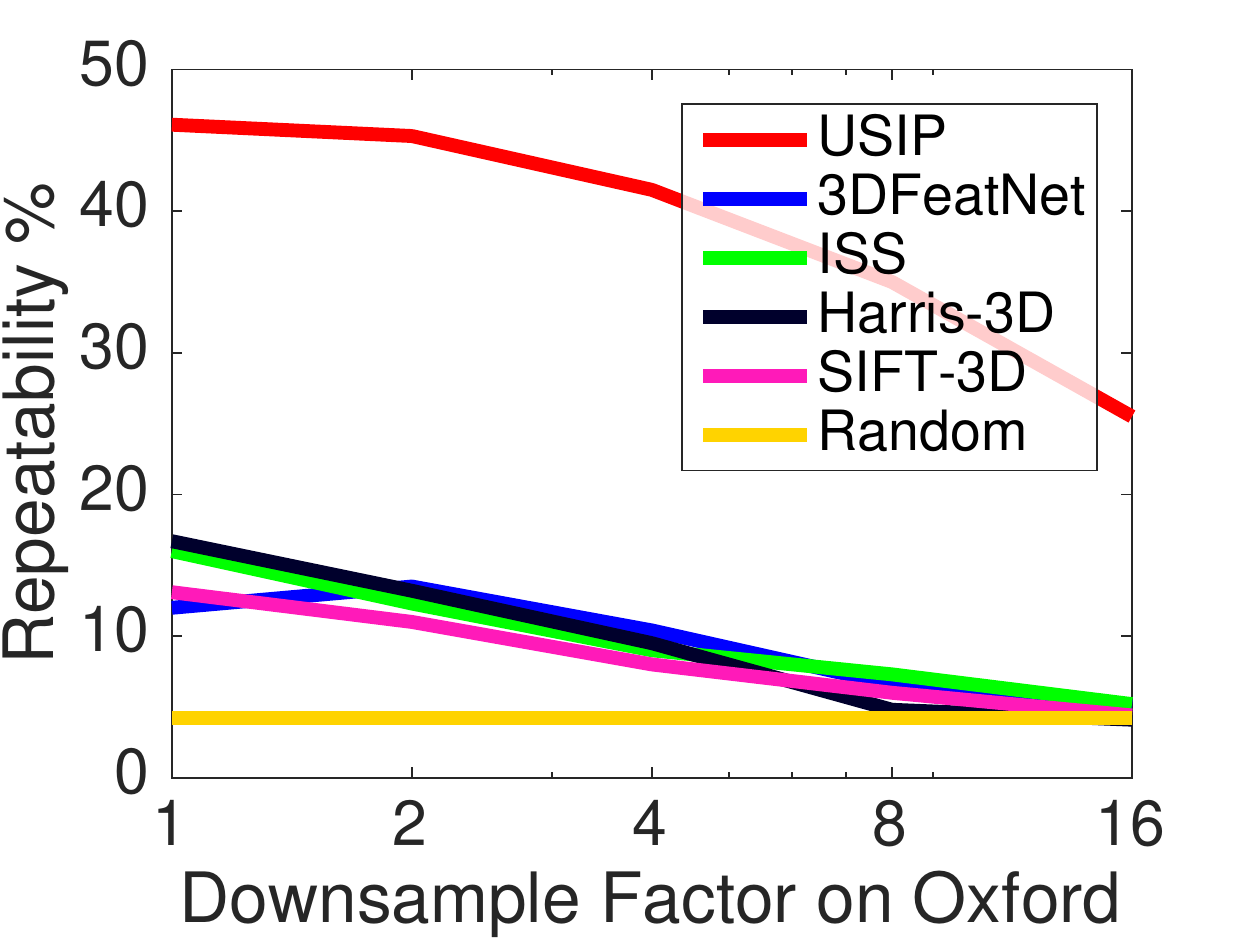}
    \includegraphics[width=0.24\textwidth]{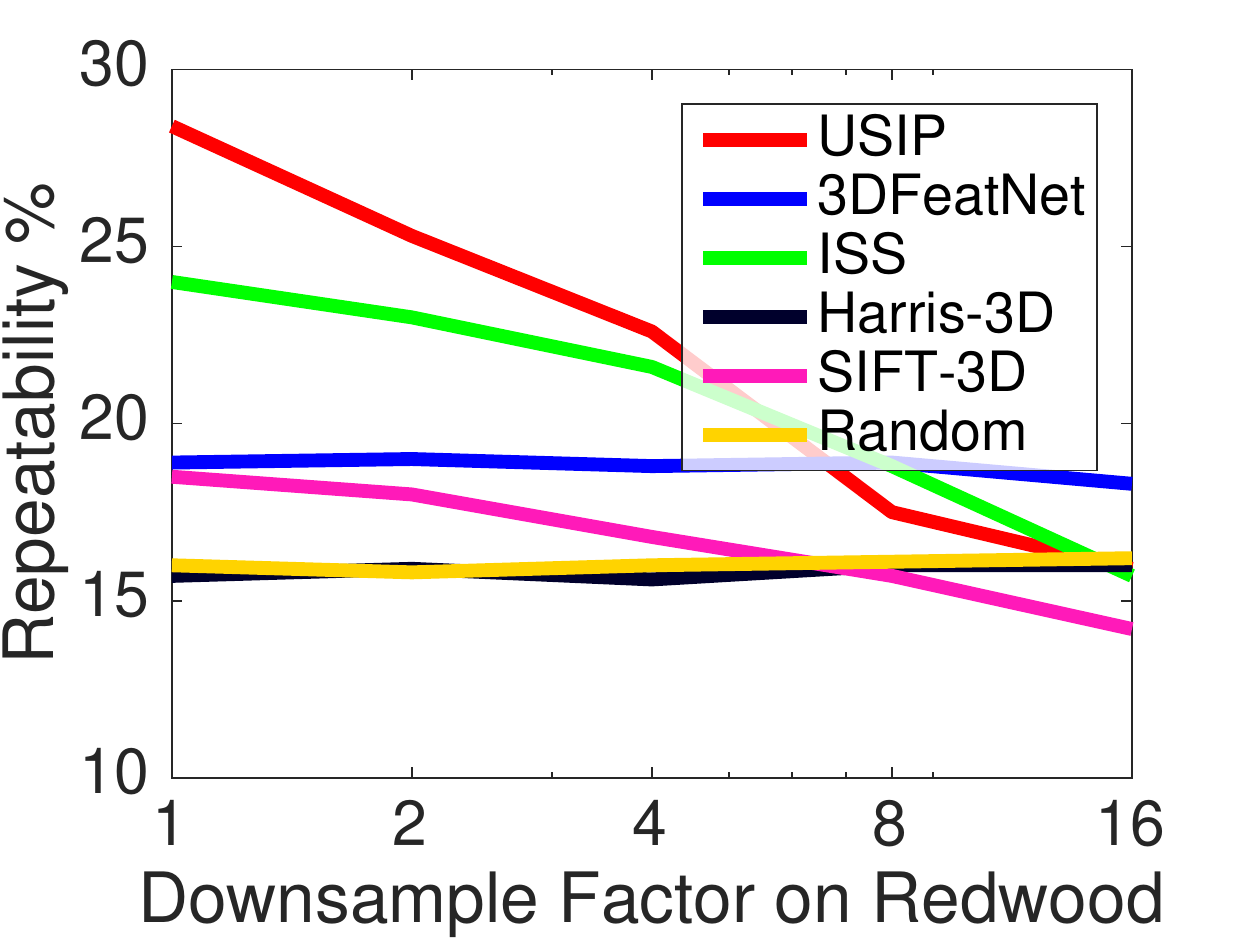}
    \includegraphics[width=0.24\textwidth]{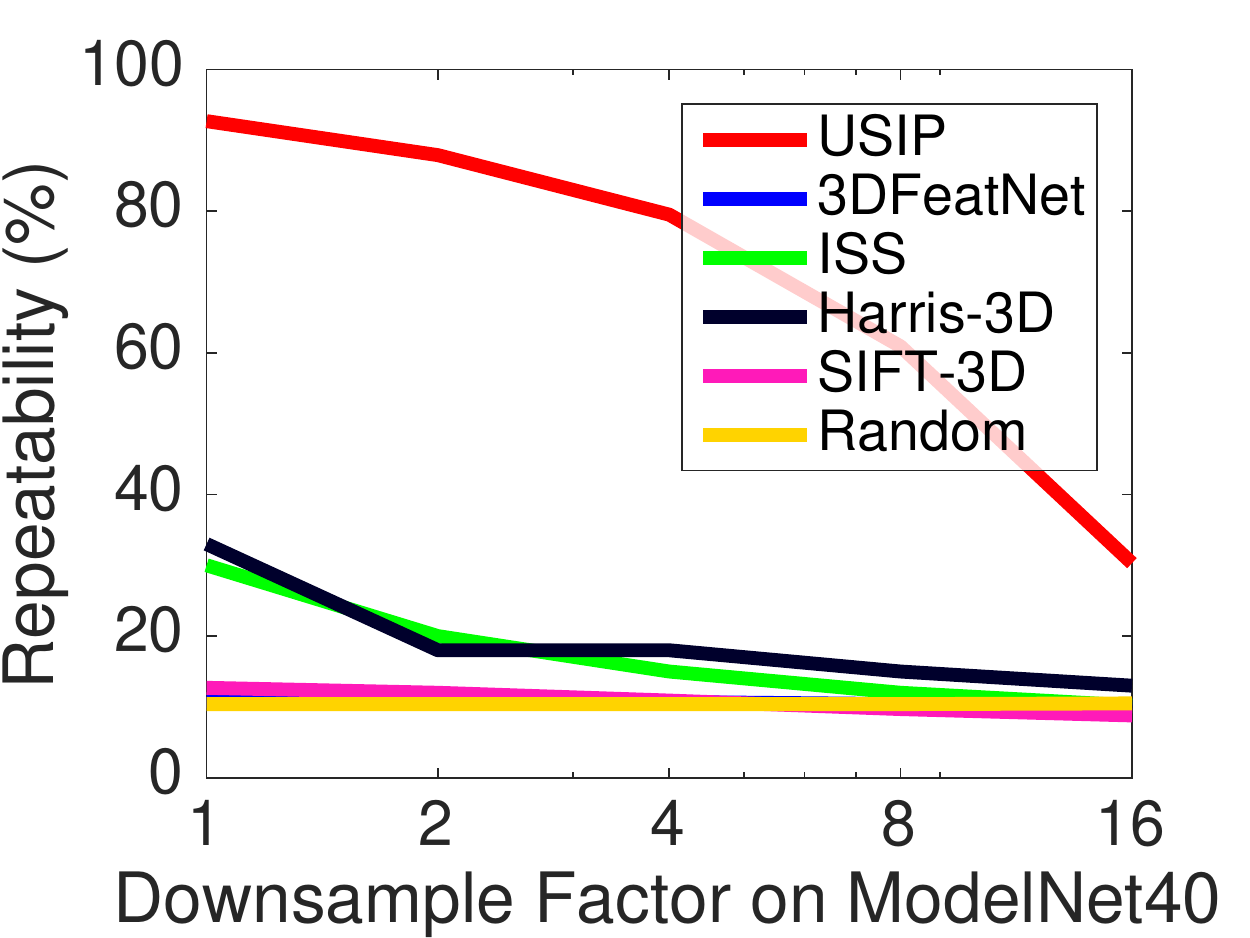}
    \caption{Relative repeatability when the input point cloud is randomly downsampled by some factors. Keypoint number is fixed to 128.} \label{fig_repeatability_downsample}
    \vspace{-8pt}
\end{figure*} 
Following \cite{tombari2013performance}, we evaluate the \textit{repeatability} (Sec.~\ref{sec_exp_repeatability}), \textit{distinctiveness} (Sec.~\ref{sec_exp_registration}) and \textit{computational efficiency} (Sec.~\ref{sec_exp_timing}) of our USIP detector on 4 datasets from object models, outdoor Lidar and indoor RGB-D scans. Additionally, we compare our evaluations to existing detectors - ISS \cite{zhong2009intrinsic}, Harris-3D \cite{harris1988combined}, SIFT-3D \cite{lowe2004distinctive} and 3DFeat-Net \cite{jian20183dfeat} . 

\vspace{-0.2cm}
\paragraph{Implementation Details}
Three USIP detectors are respectively trained for outdoor Lidars, RGB-D scans and object models. Specifically, we use the Oxford \cite{RobotCarDatasetIJRR} for outdoor Lidar, ``RGB-D reconstruction dataset" \cite{zeng20173dmatch} for RGB-D, and ModelNet40 \cite{wu20153d} for object models.
The PCL \cite{pcl} implementations of the classical detectors, \ie, ISS, Harris-3D and SIFT-3D are used for the comparisons. We take the pretrained models of 3DFeat-Net \cite{jian20183dfeat} for KITTI \cite{Geiger2012kitti} and Oxford, and train separate models for Redwood and ModelNet40 using the codes provided by 3DFeat-Net.

\vspace{-0.3cm}
\paragraph{Qualitative Visualization} 
Fig.~\ref{fig_exp_modelnet40} shows some results from our USIP detector on ModelNet40. We can clearly see that our USIP learns keypoints on corners, edges, center of small surfaces, etc. Keypoints in the first row of Fig.~\ref{fig_exp_modelnet40} are selected with Non-Maximum Suppression (NMS) and thresholding on the saliency uncertainty $\sigma$. In the second row, keypoints are selected with only NMS. Keypoints with small $\sigma$ are shown in bright red and get darker with larger $\sigma$.

\begin{figure}[h!] \centering
    \includegraphics[width=0.12\textwidth, height=0.12\textwidth, keepaspectratio]{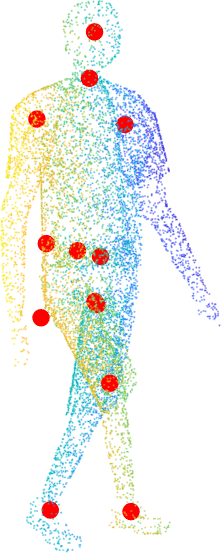}
    \includegraphics[width=0.12\textwidth, height=0.12\textwidth, keepaspectratio]{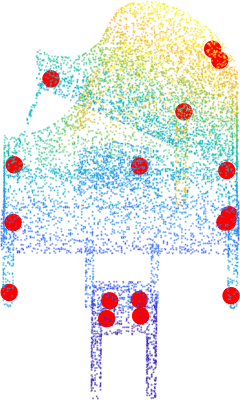}
    \includegraphics[width=0.12\textwidth, height=0.12\textwidth, keepaspectratio]{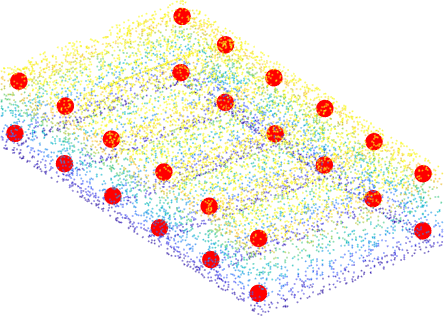}
    \includegraphics[width=0.12\textwidth, height=0.12\textwidth, keepaspectratio]{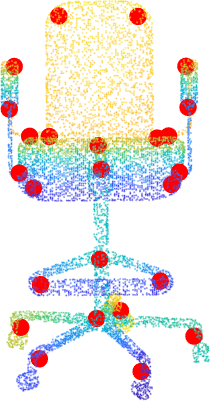}
    \includegraphics[width=0.12\textwidth, height=0.12\textwidth, keepaspectratio]{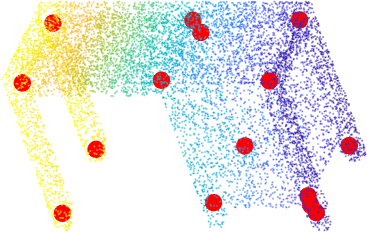}

    \includegraphics[width=0.12\textwidth, height=0.12\textwidth, keepaspectratio]{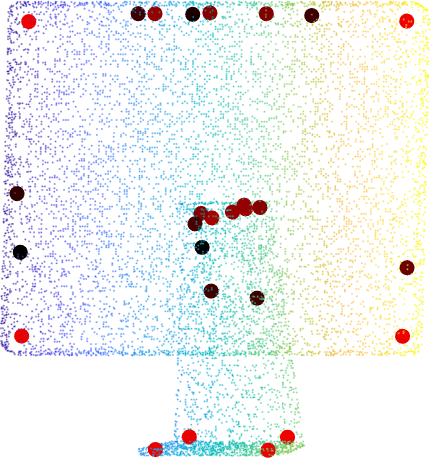}
    \includegraphics[width=0.12\textwidth, height=0.12\textwidth, keepaspectratio]{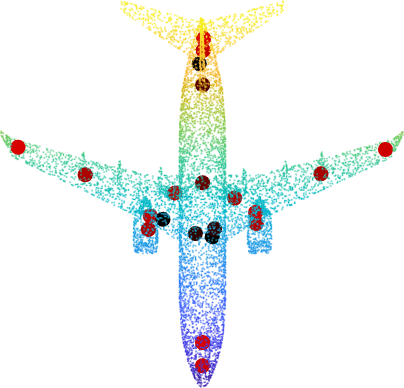}
    \includegraphics[width=0.12\textwidth, height=0.12\textwidth, keepaspectratio]{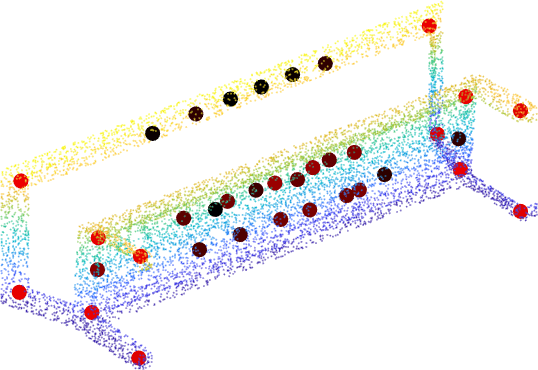}
    \includegraphics[width=0.12\textwidth, height=0.12\textwidth, keepaspectratio]{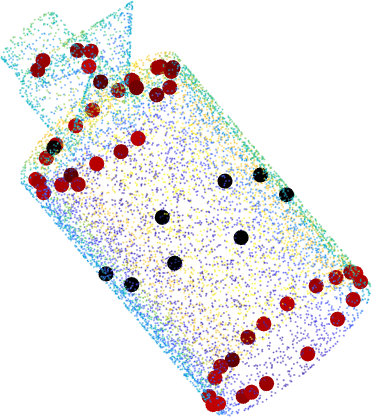}
    
    \caption{Examples of keypoints from our USIP on ModelNet40.} \label{fig_exp_modelnet40}
    \vspace{-8pt}
\end{figure}

\begin{table}[t]
\centering
\resizebox{0.48\textwidth}{!}
{%
\setlength\tabcolsep{2pt} 
\begin{tabular}{l|llll}
\hline
              & KITTI  & Oxford & Redwood  & ModelNet40 \\ \hline
Type          & Velodyne lidar & SICK lidar & RGB-D  & CAD Model \\
Scale         & 200m   & 60m    & 10m      & 2         \\
\# point       & 16,384 & 16,384 & 10,240   & 5,000      \\
$\epsilon$ in Eq.~\ref{equ_keypoint_repeated_def}    & 0.5m   & 0.5m   & 0.1m     & 0.03      \\
Rotation      & 2D     & 2D     & 3D       & 3D         \\
Noise         & Sensor & Sensor & Gaussian & Gaussian   \\
Occlusion     & Yes    & Yes    & Yes      & No         \\
Density Variation & Yes  & No   & No       & No         \\
Missing Parts & Yes    & Yes    & Yes      & No         \\ \hline
\end{tabular}%
}
\caption{Datasets used in evaluating keypoint repeatability.}
\label{tbl_datasets}
\vspace{-8pt}
\end{table}
\begin{table*}[t]
\centering
\begin{tabular}{l|llll|llll}
\hline
\multirow{2}{*}{} & \multicolumn{4}{c|}{Registration Failure Rate (\%)} & \multicolumn{4}{c}{Inlier Ratio (\%)} \\ \cline{2-9} 
 & Our Desc. & 3DFeatNet\cite{jian20183dfeat} & FPFH\cite{rusu2009fast} & SHOT\cite{tombari2010shot} & Our Desc. & 3DFeatNet & FPFH & SHOT \\ \hline
Random & 18.83 & 42.14 & 49.95 & 68.39 & 7.47 & 4.48 & 5.45 & 4.46 \\
SIFT-3D\cite{pcl, lowe2004distinctive} & 15.44 & 42.63 & 79.72 & 84.49 & 7.36 & 5.47 & 4.24 & 4.11 \\
ISS\cite{pcl, zhong2009intrinsic} & 5.97 & 25.96 & 37.09 & 69.83 & 8.52 & 4.71 & 4.44 & 3.45 \\
Harris-3D\cite{pcl, harris1988combined} & 3.81 & 13.56 & 49.49 & 51.29 & 10.57 & 6.58 & 4.78 & 5.00 \\
3DFeatNet\cite{jian20183dfeat} & 2.61 & 2.26 & 12.15 & 11.76 & 15.66 & 10.76 & 9.55 & 8.46 \\ \hline
USIP & \textbf{1.41} & \textbf{1.55} & \textbf{8.37} & \textbf{5.40} & \textbf{32.20} & \textbf{22.48} & \textbf{18.77} & \textbf{18.21} \\ \hline
\end{tabular}
\caption{Point cloud registration results on KITTI. The number of keypoints is fixed to 256.}
\label{tbl_registration}
\vspace{-8pt}
\end{table*}

\subsection{Repeatability} \label{sec_exp_repeatability}
Repeatibility refers to the ability of a detector to detect keypoints in the same locations under various disturbances such as view-point variations, noise, missing parts, etc. It is often taken as the most important measure of keypoint detectors because it is a standalone measure that depends only on the detector (without a descriptor).
Given two point clouds $\{\mathbf{X}, \tilde{\mathbf{X}}\}$ of a scene captured from different view-points such that $\{\mathbf{X}, \tilde{\mathbf{X}}\}$ are related by a rotation matrix $R \in \text{SO}(3)$ and a translational vector $t\in \mathbb{R}^{3}$. A keypoint detector detects a set of keypoints $\mathbf{Q} = [Q_1, \cdots, Q_M]$ and $\tilde{\mathbf{Q}} = [\tilde{Q}_1, \cdots, \tilde{Q}_M]$ from $\{\mathbf{X}, \tilde{\mathbf{X}}\}$, respectively.
A keypoint $Q_i \in \mathbf{Q}$ is repeatable if the distance between $RQ_i+t$ and its nearest neighbor $\tilde{Q}_j \in \tilde{\mathbf{Q}}$ is less than a threshold $\epsilon$, \ie, 
\begin{equation} \label{equ_keypoint_repeated_def}
    \|RQ_i+t - \tilde{Q}_j\|_2 < \epsilon.
\end{equation}

\vspace{-0.4cm}
\paragraph{Test Datasets} We evaluate repeatability on four test datasets - KITTI, Oxford, Redwood and ModelNet40. Note that our USIP is not trained on KITTI nor Redwood. We use the KITTI and Oxford test datasets prepared by 3DFeat-Net \cite{jian20183dfeat}. Each pair of point clouds $\{\mathbf{X}, \tilde{\mathbf{X}}\}$ are captured from nearby locations of within 10m and manually augmented with random 2D rotations. $\{\mathbf{X}, \tilde{\mathbf{X}}\}$ in Redwood are from simulated RGB-D cameras with 3D rotations / translations and Gaussian noise. The overlapped areas between $\{\mathbf{X}, \tilde{\mathbf{X}}\}$ are as low as 30\%. In ModelNet40, $\tilde{\mathbf{X}}$ is obtained by augmenting $\mathbf{X}$ with random 3D rotations. Points in KITTI, Oxford and Redwood are in its original scale while points in ModelNet40 are normalized to $[-1, 1]$. Details of the datasets are shown in Tab.~\ref{tbl_datasets}. The scale refers to the diameter of the point clouds. 

\vspace{-0.3cm}
\paragraph{Relative Repeatability}
We use relative repeatability that normalizes over the total number of detected keypoints $|\mathbf{Q}|$ for fair comparisons, \ie, $\text{repeatability} = |\mathbf{Q}_{\text{rep}}| / |\mathbf{Q}|$,
where $\mathbf{Q}_{\text{rep}}$ is the number of keypoints that passed the repeatability test in Eq.~\ref{equ_keypoint_repeated_def}.
We set the parameters of each keypoint detector in each dataset to generate 4, 8, 16, 32, 64, 128, 256 and 512 keypoints or close to these numbers when it is not possible to set the detectors (SIFT-3D, Harris-3D and ISS) to generate exact number of keypoints. Note that in general the repeatability should be proportional to the number of keypoints. In the extreme case that $\mathbf{Q}=\mathbf{X}$, \ie, each point is regarded as a keypoint, the repeatability is the same as the percentage of overlap between $\{\mathbf{X}, \tilde{\mathbf{X}}\}$. 
As shown in Fig.~\ref{fig_repeatability}, our USIP generally outperforms other detectors by a significant margin on the 4 datasets over 8 different number of keypoints. In the extremely hard case that only 4 keypoints are detected, our method achieved relative repeatability of 34\%, 23\%, 10\% and 60\% for KITTI, Oxford, Redwood and ModelNet40, respectively. In the case of 64 keypoints, our performance is roughly 4.2x, 2.8x, 1.3x and 2.6x higher than the second best detector.

\vspace{-0.35cm}
\paragraph{Robustness to Noise}
The original points in KITTI and Oxford are already corrupted with sensor noise. We further augment the point clouds in the 4 datasets with Gaussian noise $\mathcal{N}(0, \sigma_{noise})$, where $\sigma_{noise}$ is up to $0.6$m for KITTI and Oxford, $0.12$m for Redwood and $0.12$ (no unit) for ModelNet40. The number of keypoints is fixed to 128.
Our USIP is a lot more robust than other detectors as shown in Fig.~\ref{fig_repeatability_noisy}. In KITTI and Oxford, the performances of other detectors fall to the level of random sampling when $\sigma_{noise} \geq 0.2$m, while our USIP does not show significant drop in performance even with $\sigma_{noise}\geq0.6m$. In Redwood, other methods except USIP and ISS deteriorate to random sampling with $\sigma_{noise}\geq0.02$m. In ModelNet40, our method maintain high repeatability of 91\% with $\sigma_{noise}=0.02$, while all other methods drop below 8\%.

\vspace{-0.35cm}
\paragraph{Robustness to Downsampling}
We evaluate the repeatability of the detectors on input point clouds downsampled by some factors using random selection.
The results are shown in Fig.~\ref{fig_repeatability_downsample}, where the down-sample factor denoted as $\alpha$ means the number of points is reduced to $\frac{1}{\alpha}$ of the original number shown in Tab.~\ref{tbl_datasets}. We can see that the repeatability of our USIP remains satisfactory even with a $16\times$ downsampling on KITTI, Oxford and ModelNet40. The only exception is the Redwood dataset, where almost all detectors perform poorly on high downsample factors. Indoor RGB-D scans in Redwood consist of many large and flat surface like wall, ceiling, etc. Furthermore, there are very few distinguishable and non-occluded structures, which are further aggrevated by severe downsampling. Hence, it is difficult to detect repeatable keypoints with these RGB-D scans.

\subsection{Distinctiveness: Point Cloud Registration} \label{sec_exp_registration}
Distinctiveness is a measure of the performance of keypoint detectors and descriptors for finding correspondences in point cloud registration. Hence, distinctiveness is not as good as repeatability as an evaluation criterion on keypoint detectors because it is confounded with the performance of the descriptor. We mitigate this limitation by evaluating point cloud registration over several existing keypoint descriptors. We also use the results to show that our USIP detector works with different existing keypoint descriptors.

\vspace{-0.3cm}
\paragraph{Experiment Setup}
We follow the point cloud registration pipeline from 3DFeat-Net \cite{jian20183dfeat} on their KITTI test dataset. Four descriptors are used to perform keypoint description, \ie, three off-the-shelf descriptors: 3DFeatNet, FPFH \cite{rusu2010fast}, SHOT\cite{tombari2010shot}, and our own descriptor inspired by 3DFeat-Net with minor modifications, which is denoted as ``Our Desc.'' (details are in our supplementary material). Registration of a pair of point clouds involves 4 steps: (a) Extract keypoints and their corresponding descriptor vectors from each point cloud. (b) Establish keypoint-to-keypoint correspondences by nearest neighbor search of the descriptor vectors. (c) Perform RANSAC on the two matched keypoint sets to find the rotation and translation that have the most inliers. (d) Compare the resulted rotation and translation with the ground truth. A pair of point cloud is regarded as successfully registered if $\text{Relative Translational Error (RTE)} < 2$m, and $\text{Relative Rotation Error (RRE)} < 5\degree$.

\vspace{-0.3cm}
\paragraph{Registration Results} We perform registration evaluations over the combination of 6 keypoint detectors and 4 descriptors. The registration failure rate and keypoint inlier ratio are shown in Tab.~\ref{tbl_registration}. Compared to other detectors, our USIP achieves the lowest registration failure rate and the highest inlier ratio with a considerable margin on all the 4 descriptors. The significance of the results in Tab.~\ref{tbl_registration} is two fold. First, our USIP works well with
various hand-crafted (FPFH and SHOT) and deep learning-based (our desc. and 3DFeat-Net) descriptors. 
Second, our USIP produces more distinctive keypoints since it consistently outperforms other keypoint detectors over the different descriptors on registration failure rate and keypoint inlier ratio as shown in Tab.~\ref{tbl_registration}.
The experimental configurations in Tab.~\ref{tbl_registration} is not the optimal setting for our USIP detector and descriptor nor the 3DFeatNet because we have to fix the number of keypoints for fair comparison. In Tab.~\ref{tbl_registration_best_config}, we illustrate the best registration results for our USIP and 3DFeatNet on KITTI without limitation on the number of keypoints. We again achieve lower failure rate and higher inlier ratio. In addition, we show the visualization of keypoint matching results of two examples from KITTI and Oxford in Fig.~\ref{fig_exp_outdoor_vis}.

\begin{table}[h!]
\centering
\resizebox{0.48\textwidth}{!}{%
\setlength\tabcolsep{1pt} 
\begin{tabular}{l|l|llll}
\hline
Detector & Descriptor & Fail(\%) & Inlier(\%) & RTE(m) & RRE (\degree) \\ \hline
3DFeat-Net & 3DFeat-Net & 0.57 & 12.9 & $0.26 \pm 0.26$ & $0.56\pm0.46$ \\ 
USIP & Our Desc. & \textbf{0.24} & \textbf{28.0} & $\bm{0.21\pm0.24}$ & $\bm{0.42\pm0.32}$ \\ \hline
\end{tabular}%
}
\caption{Point cloud registration on KITTI from the optimal configurations of 3DFeat-Net and our USIP.}
\label{tbl_registration_best_config}
\vspace{-12pt}
\end{table}

\begin{figure}[h!]
    \centering
    \includegraphics[width=0.23\textwidth]{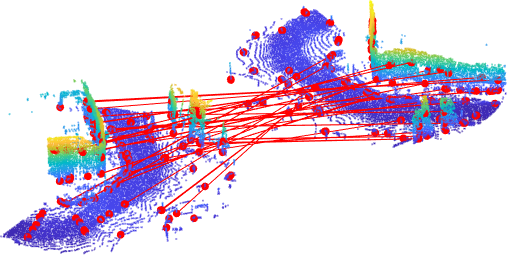}
    \includegraphics[width=0.23\textwidth]{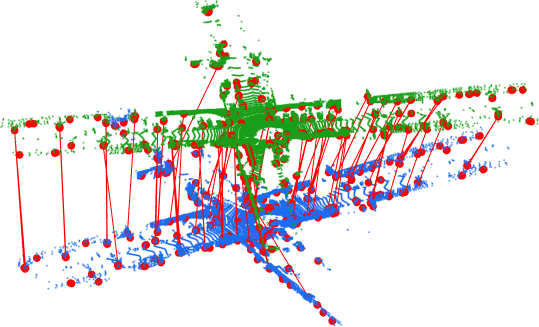}
    \caption{Keypoints and matches from our USIP detector and ``Our Desc.". Best view with color and zoom-in.} \label{fig_exp_outdoor_vis}
    \vspace{-8pt}
\end{figure}

\subsection{Computational Efficiency} \label{sec_exp_timing}
Hand-crafted detectors are deployed with single thread C++ codes on an Intel i7 6950X CPU. Our USIP and 3DFeatNet are deployed on a Nvidia 1080Ti, with PyTorch and TensorFlow, respectively. Computational efficiency is evaluated on 2,391 KITTI point clouds, where each point cloud is downsampled to 16,384 points. We record the average time taken to extract 128 keypoints from each point cloud. As shown in Tab.~\ref{tbl_timing}, our USIP is an order of magnitude faster than other detectors except random sampling.

\begin{table}[h!]
\centering
\resizebox{0.48\textwidth}{!}{%
\setlength\tabcolsep{4pt} 
\begin{tabular}{llllll}
\hline
Random & SIFT-3D & ISS & Harris-3D & 3DFeatNet & USIP \\ \hline
0.0005 & 0.163 & 0.388 & 0.150 & 0.438 & \textbf{0.011} \\ \hline
\end{tabular}
}
\caption{Average time (in seconds) to extract 128 keypoints from KITTI point clouds respectively downsampled to 16,384 points.}
\label{tbl_timing}
\vspace{-4pt}
\end{table}

\section{Conclusion} \label{sec_conclusion}
In this paper, we present the USIP detector, an unsupervised deep learning-based keypoint detector for 3D point clouds. A probabilistic chamfer loss is proposed to guide the network to learn highly repeatable keypoints. We provide mathematical analysis and solutions for network degeneracy, which are supported by experimental results. Extensive evaluations are performed with Lidar scans, RGB-D images and CAD models. Our USIP detector out-performs existing detectors by a significant margin in terms of repeatability, distinctiveness and computational efficiency.

%


{\small
\bibliographystyle{ieee}
\bibliography{ICCV2019_REF}
}




\clearpage
\appendix

\section{Overview}
We provide more details on the algorithms and experiments described in the main paper. Sec.~\ref{sec_suppl_degeneracy} presents more examples of the network degeneracy. Sec.~\ref{sec_suppl_L_p} evaluates the effect of point-to-point loss $\mathcal{L}_p$ on the keypoint repeatability. Sec.~\ref{sec_suppl_descriptor} illustrates the details of our feature descriptor design. Sec.~\ref{sec_suppl_pc_registration} gives more experiments on point cloud registration tasks. Sec. \ref{sec_suppl_visualization} presents visualizations of our USIP keypoints in various datasets.

\section{More Examples on Degeneracy} \label{sec_suppl_degeneracy}
As analyzed in Sec.~5, our FPN degenerates when the receptive field becomes sufficiently large, 
\ie, it has gained sufficient global semantic information.
The receptive field of the FPN is controlled by two parameters: number of keypoint proposals $M$ and number of neighbors $K$ in the $K$NN feature aggregation. More specifically, the receptive field size is proportional to $K$ and inversely proportional to $M$. In this section, we visualize the network degeneracy by gradually enlarging the receptive field. Fig.~\ref{fig_suppl_vis_degeneracy_m64} shows the degeneracies when $M=64$ and $K=\{9, 24, 32, 40, 48, 64\}$. Fig.~\ref{fig_suppl_vis_degeneracy_k9} shows the degeneracies when $K=9$ and $M=\{64, 24, 20, 16, 12, 9\}$.

\begin{figure}[H] \centering
    \includegraphics[width=0.15\textwidth, height=0.16\textwidth, keepaspectratio, angle=0]{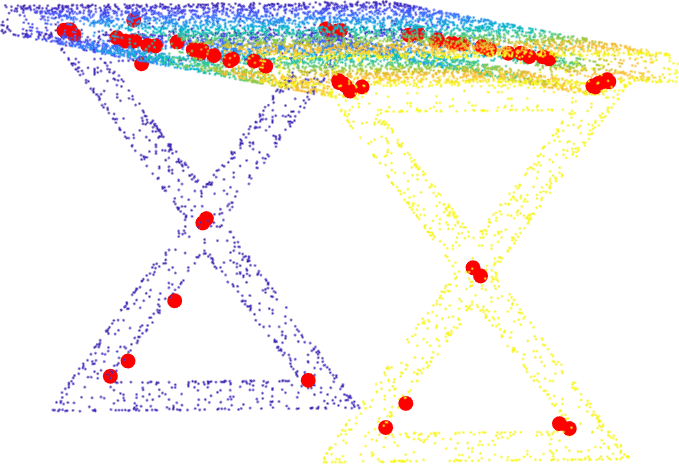}
    \includegraphics[width=0.15\textwidth, height=0.16\textwidth, keepaspectratio, angle=0]{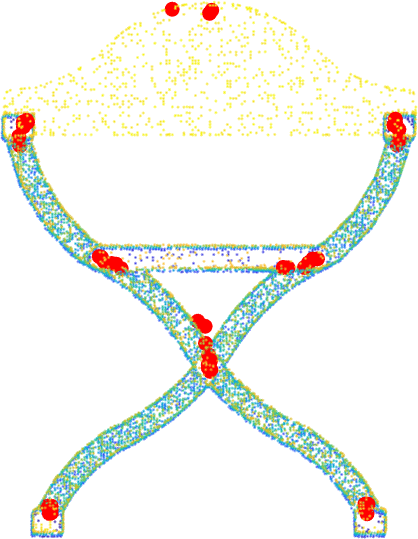}
    \includegraphics[width=0.15\textwidth, height=0.16\textwidth, keepaspectratio, angle=90]{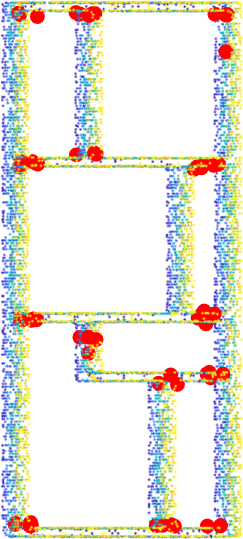}
    
    \includegraphics[width=0.15\textwidth, height=0.16\textwidth, keepaspectratio, angle=0]{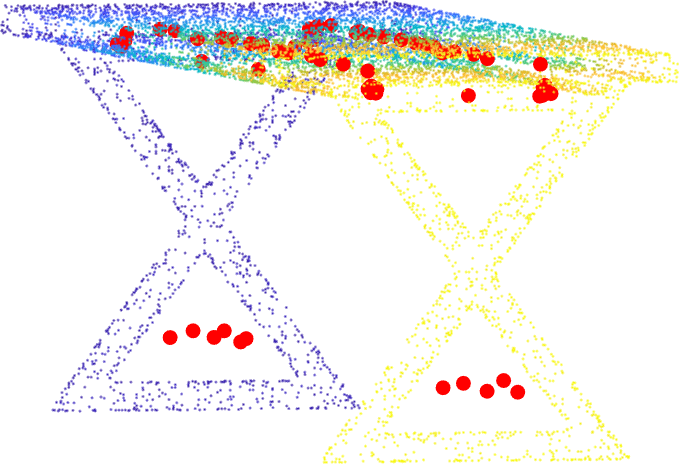}
    \includegraphics[width=0.15\textwidth, height=0.16\textwidth, keepaspectratio, angle=0]{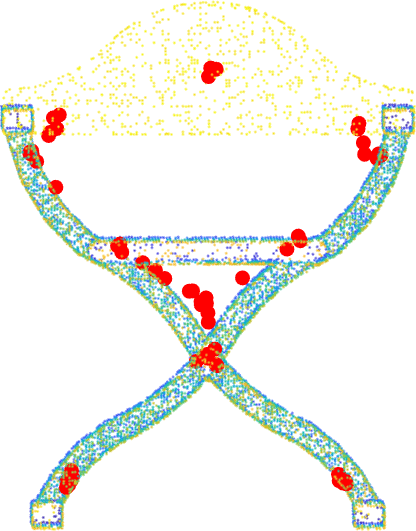}
    \includegraphics[width=0.15\textwidth, height=0.16\textwidth, keepaspectratio, angle=90]{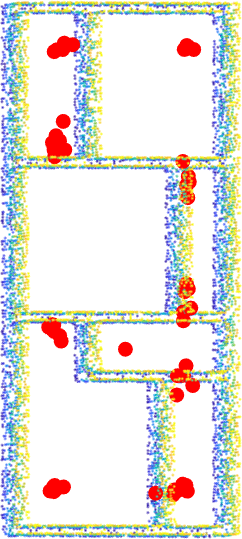}
    \caption{Visualization of USIP keypoints with different $\lambda$ in Point-to-Point loss. First row $\lambda=6$, second row $\lambda=0$.} \label{fig_suppl_lambda_effect_vis}
\end{figure}

\section{Effect of $\lambda$ in Point-to-Point Loss $\mathcal{L}_p$} \label{sec_suppl_L_p}
Sec.~3 of the main paper describes the point-to-point loss $\mathcal{L}_p$ to penalize $Q_m \in \mathbf{Q}$ for being too far from $\mathbf{X}$. The point-to-point loss $\mathcal{L}_p$ is added to the loss function with the weight $\lambda$. Here, we show that our USIP is very robust to the value of $\lambda$. Specifically, the repeatability of our USIP keypoints remains almost the same over a wide range of values for $\lambda$. Keypoint repeatability is illustrated in Fig.~\ref{fig_suppl_repeatability_lambda} with various $\lambda$. Fig.~\ref{fig_suppl_repeatability_lambda} shows that the USIP keypoints are highly repeatable even when $\lambda$ is small. This is probably because our design to limit the receptive field already guides the network to learn repeatable keypoints even without the point-to-point loss. On the other hand, the network fails to converge when $\lambda$ is too large because the point-to-point loss dominates the training process. Nonetheless, training the network without the point-to-point loss does not ensure the keypoints to be close to the input point cloud. The top row of Fig. \ref{fig_suppl_lambda_effect_vis} shows keypoints from our USIP detector trained with $\lambda=6$, \ie, with point-to-point loss. They are close to the input point cloud. In comparison, the bottom row of Fig. \ref{fig_suppl_lambda_effect_vis} shows  
from our USIP detector trained without point-to-point loss, \ie, $\lambda=0$. These are less desirable keypoints that are farther from the input point cloud.

\begin{figure*}[t!] \centering
    \includegraphics[width=0.24\textwidth]{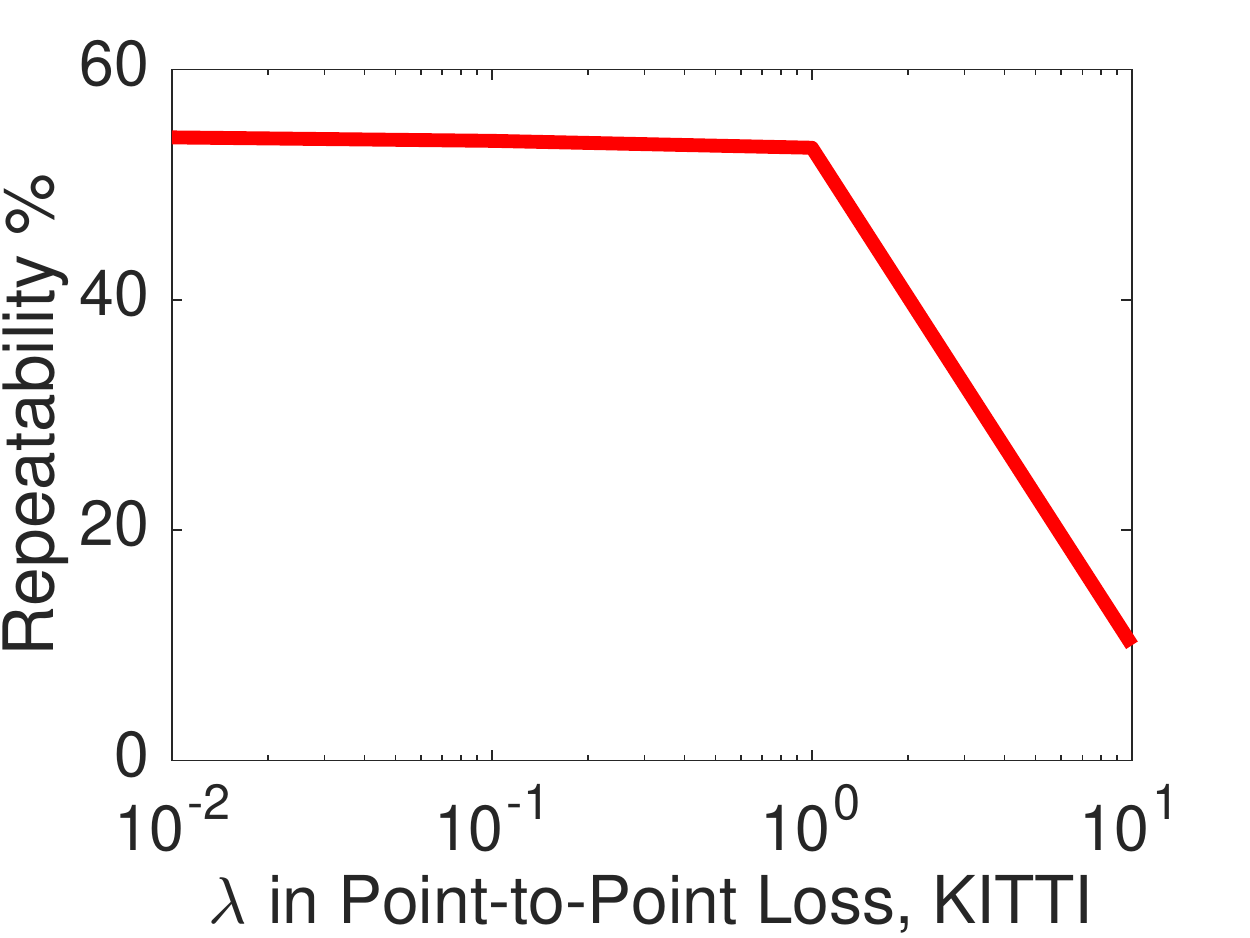}
    \includegraphics[width=0.24\textwidth]{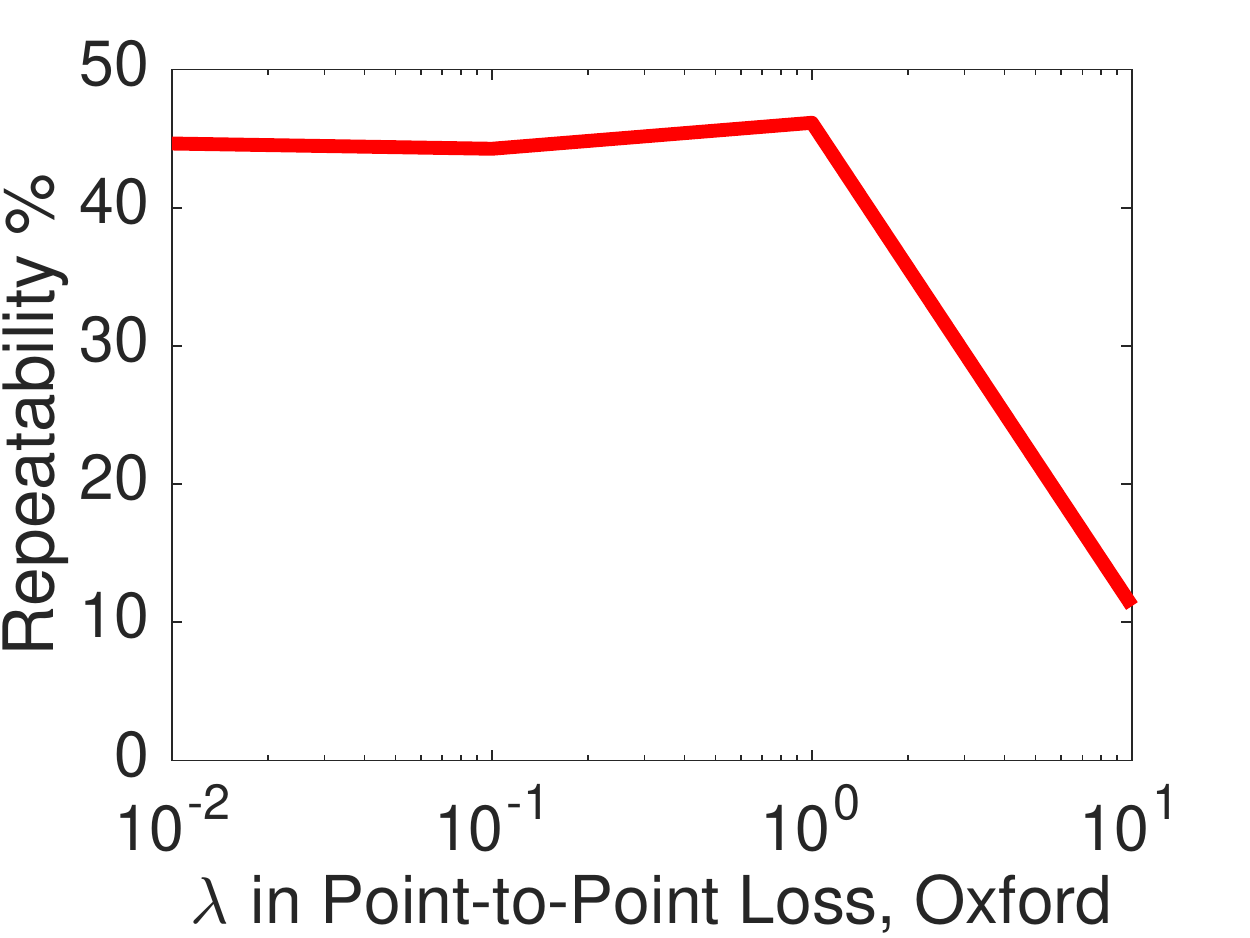}
    \includegraphics[width=0.24\textwidth]{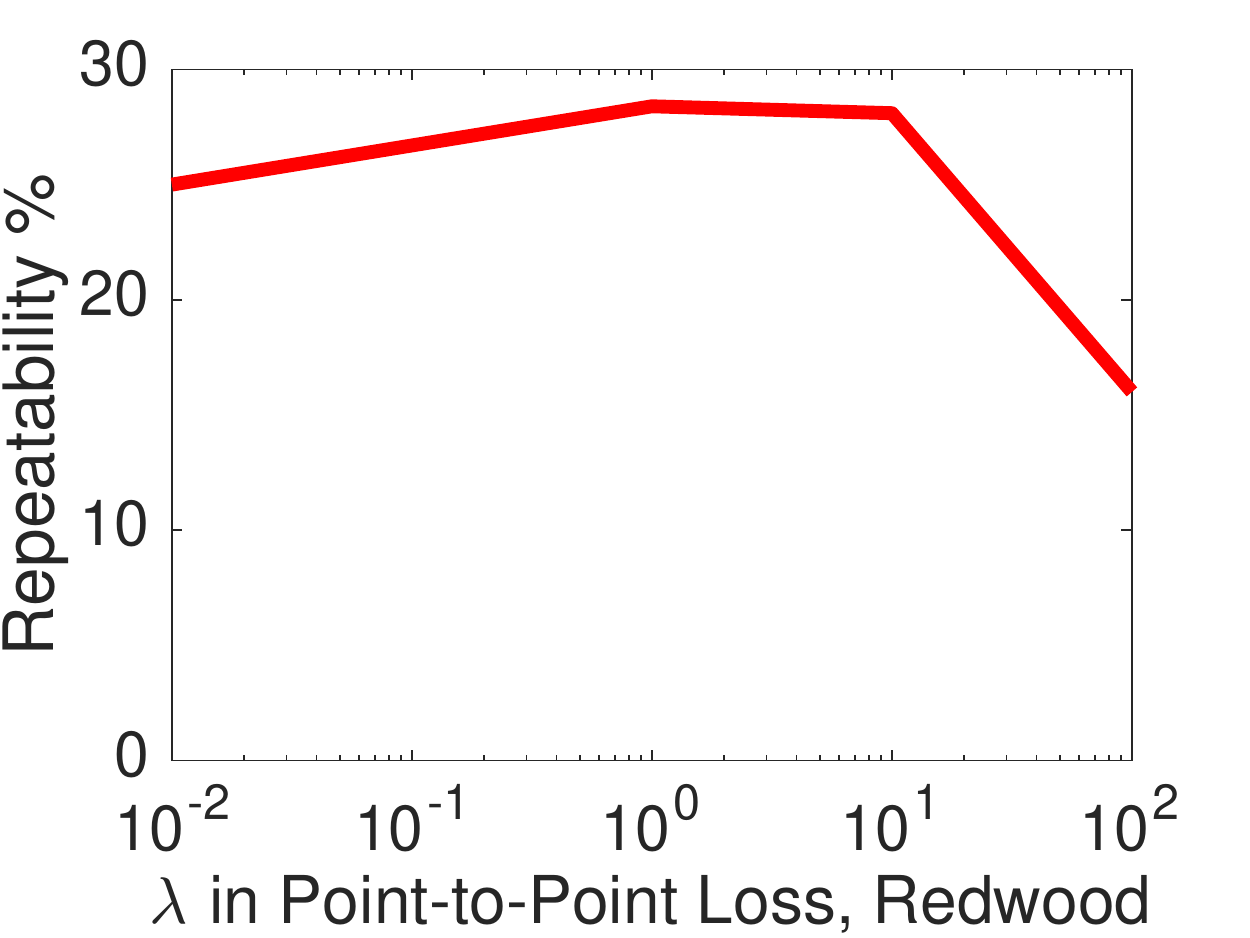}
    \includegraphics[width=0.24\textwidth]{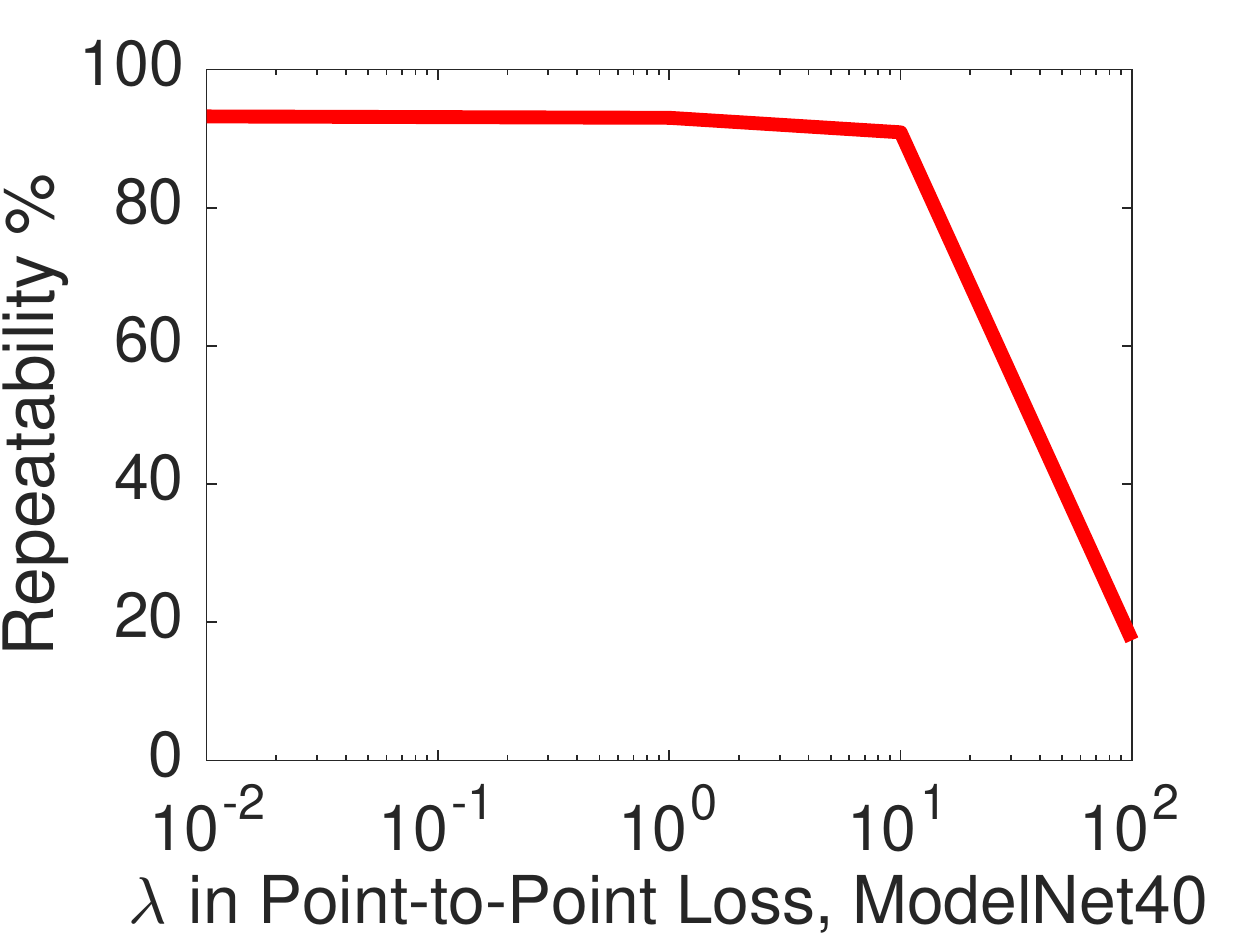}
    \caption{Relative repeatability with different weight $\lambda$ for the Point-to-Point Loss $\mathcal{L}_p$. Number of keypoints is fixed to 128. Left to right: KITTI, Oxford, Redwood, ModelNet40.} \label{fig_suppl_repeatability_lambda}
\end{figure*}

\section{Our Descriptor a.k.a ``Our Desc."} \label{sec_suppl_descriptor}
\begin{figure}[th] \centering
    \includegraphics[width=0.46\textwidth]{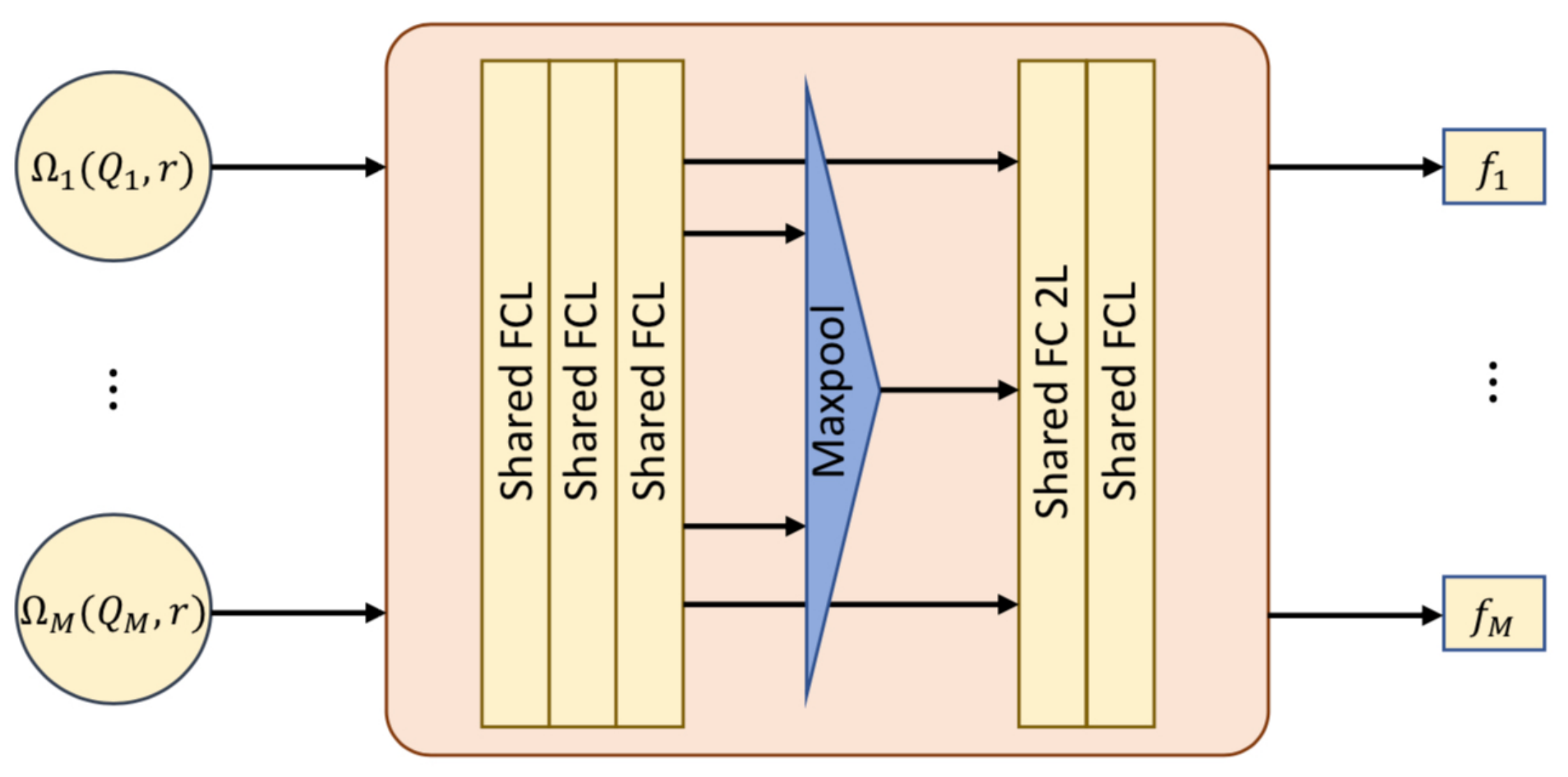}
    \caption{Network architecture of ``Our Desc.".} \label{fig_descriptor}
    \vspace{-4pt}
\end{figure} 

Fig~\ref{fig_descriptor} shows the network design of ``Our Desc." inspired by 3DFeat-Net \cite{jian20183dfeat} as mentioned in Sec.~6.2 of the main paper. Given the output $(\mathbf{Q},\mathbf{\Sigma})$ from FPN, a ball $\Omega_m(Q_m, r)$ of points from the point cloud $\mathbf{X}$ within a radius $r$ is built around each $Q_m \in \mathbf{Q}$. A keypoint descriptor $f_m\in \mathbb{R}^L$ is extracted for each $\Omega_m$. The descriptor can be trained with either weak \cite{jian20183dfeat} or strong supervision \cite{zeng20173dmatch, khoury2017learning}. We improve the keypoint descriptor training by utilizing the keypoint saliency uncertainty $\mathbf{\Sigma}$ in Sec.~\ref{sec_suppl_desc_weakly}, \ref{sec_suppl_desc_strong}, and \ref{sec_suppl_pc_registration}.

\subsection{Weak Supervision} \label{sec_suppl_desc_weakly}
Weak supervision of the descriptor is based on a triplet loss and the ground truth coarse registrations of the point clouds in the training dataset. Similar to \cite{jian20183dfeat}, point clouds from the dataset are selected as the anchor samples during training. All overlapping pairs of point clouds to the anchor are defined as positive samples, while non-overlapping pairs of point clouds are defined as the negative samples. We denote the sets of keypoint descriptors extracted from the anchor, positive and negative samples as $F_{\text{anc}}$, $F_{\text{pos}}$ and $F_{\text{neg}}$, respectively.
We generate these training samples from the Oxford RobotCar and KITTI datasets. More formally, the triplet loss is given by:
\begin{equation} \label{equ_desc_weak_loss}
    \mathcal{L}_{\text{dc}}^w = \sum_{m=1}^{M} w_m \bigg\lfloor \underset{f_i\in F_{\text{pos}}}{\text{min}}\|f_m-f_i\|_2 - \underset{f_j\in F_{\text{neg}}}{\text{min}}\|f_m-f_j\|_2 + \gamma \bigg\rfloor_+,
\end{equation}
where $f_m \in F_{\text{anc}}$ is a descriptor from the anchor sample. For each descriptor $f_m \in F_{\text{anc}}$, we minimize the Euclidean distance to its nearest neighbor $f_i \in F_{\text{pos}}$ and maximize the Euclidean distance to its nearest neighbor $f_j \in F_{\text{neg}}$.
In addition, a normalized weight $w_m$ is added to our triplet loss. $w_m$ is derived from our USIP keypoint saliency uncertainty $\sigma_m$ that indicates the reliability of $Q_m$ and $f_m$. More specifically:
\begin{equation} \label{equ_desc_weight}
    w_m = M \cdot \frac{\hat{w_m}}{\sum_{j=1}^{M}\hat{w_j}}, \qquad  \hat{w_m} = \lfloor\xi - \sigma_m\rfloor_+,  
\end{equation}
where $\xi$ is a threshold serves as the upper bound of $\sigma_m$.

\subsection{Strong Supervision} \label{sec_suppl_desc_strong}
We do strong supervision of the descriptor network on datasets with ground truth poses, \ie,  
SceneNN \cite{hua2016scenenn} and ``3D reconstruction dataset" \cite{zeng20173dmatch}. The loss function for strong supervision defined on a pair of overlapping point clouds $\mathbf{X}$ and $\mathbf{X}'$ with ground truth poses $G\in \text{SE}(3)$ and $G'\in \text{SE}(3)$ is given by:
\begin{equation} \label{equ_desc_strong_loss}
    \mathcal{L}_{\text{dc}}^s = \sum_{m=1}^{M} w_m \bigg\lfloor \|f_m-f'_i\|_2 - \|f_m-f'_j\|_2 + \gamma \bigg\rfloor_+.
\end{equation}
$f_m$ and $f'_i$ are keypoint descriptors from $\mathbf{X}$ and $\mathbf{X}'$, respectively. 
Additionally, $f'_i$ is a descriptor with keypoint location $Q'_i$ that is within a distance $\rho$ from the keypoint location $Q_m$ of the descriptor $f_m$, \ie, $\|Q_m - G G'^{-1}Q'_i\|_2 < \rho$. To achieve hard negative mining, we randomly select 50\% of $f'_j$ from $\mathbf{X}'$ with the distance between the keypoint locations $Q'_j$ and $Q_m$ larger than $\rho$. The other 50\% are chosen from keypoints with shortest but larger than $\rho$ keypoint distances to $Q_m$.

\section{More Point Cloud Registration Results} \label{sec_suppl_pc_registration}
\begin{table*}[t]
\centering
\resizebox{\textwidth}{!}
{%
\setlength\tabcolsep{4pt} 
\begin{tabular}{c|ccrrr|ccrrr}
\hline
\multirow{2}{*}{Method} & \multicolumn{5}{c|}{Oxford}                                                                                         & \multicolumn{5}{c}{KITTI}                                                                                         \\ \cline{2-11} 
                        & RTE (m)     & RRE (\degree)      & \multicolumn{1}{l}{Fail \%} & \multicolumn{1}{l}{Inlier \%} & \multicolumn{1}{l|}{\# Iter} & RTE (m)     & RRE (\degree)      & \multicolumn{1}{l}{Fail \%} & \multicolumn{1}{l}{Inlier \%} & \multicolumn{1}{l}{\# Iter} \\ \hline
ISS\cite{zhong2009intrinsic} + FPFH\cite{rusu2009fast}                & 0.40$\,\pm\,$0.29 & 1.60$\,\pm\,$1.02 & 7.68                          & 8.6                        & 7171                         & 0.33$\,\pm\,$0.27 & 1.04$\,\pm\,$0.77 & 39.00                         & 8.8                        & 8000                        \\
ISS\cite{zhong2009intrinsic} + SI\cite{johnson1999using}                  & 0.42$\,\pm\,$0.31 & 1.61$\,\pm\,$1.12 & 12.55                         & 4.7                        & 9888                         & 0.35$\,\pm\,$0.31 & 1.11$\,\pm\,$0.93 & 41.86                         & 4.6                        & 9401                        \\
ISS\cite{zhong2009intrinsic} + USC\cite{tombari2010usc}                 & 0.32$\,\pm\,$0.27 & 1.22$\,\pm\,$0.95 & 5.98                          & 8.6                        & 7084                         & 0.27$\,\pm\,$0.28 & 0.83$\,\pm\,$0.76 & 18.62                         & 7.7                        & 8149                        \\
ISS\cite{zhong2009intrinsic} + CGF\cite{khoury2017learning}                 & 0.43$\,\pm\,$0.32 & 1.62$\,\pm\,$1.10 & 12.64                         & 4.9                        & 9628                         & 0.23$\,\pm\,$0.25 & 0.69$\,\pm\,$0.60 & 8.90                          & 8.4                        & 7670                        \\
ISS\cite{zhong2009intrinsic} + 3DMatch\cite{zeng20173dmatch}             & 0.49$\,\pm\,$0.37 & 1.78$\,\pm\,$1.21 & 30.94                         & 5.4                        & 9131                         & 0.30$\,\pm\,$0.28 & 0.80$\,\pm\,$0.67 & 7.14                          & 8.4                        & 7165                        \\
3DFeat-Net\cite{jian20183dfeat}              & 0.30$\,\pm\,$0.26 & 1.07$\,\pm\,$0.85 & 1.90                          & 13.7                       & 2940                         & 0.26$\,\pm\,$0.26 & 0.56$\,\pm\,$0.46 & 0.57                          & 12.9                       & 3768                        \\ \hline

USIP + Our Desc.                     & \textbf{0.28$\,\pm\,$0.26} & \textbf{0.81$\,\pm\,$0.74} & \textbf{0.93}                          & \textbf{28.1}                       & \textbf{523}                          & \textbf{0.21$\,\pm\,$0.24} & \textbf{0.42$\,\pm\,$0.32} & \textbf{0.24}                          & \textbf{28.0}                       & \textbf{600}                         \\ \hline
\end{tabular}
}
\caption{Geometric registration performance on Oxford RobotCar and KITTI. The combination of our USIP keypoint detector and ``Our Desc." outperforms existing methods in all criteria with around $2\times$ inlier ratio.}
\label{tbl_suppl_pc_registration_outdoor}
\end{table*}
\begin{figure*}[ht!]
    \centering
    \includegraphics[width=0.49\textwidth]{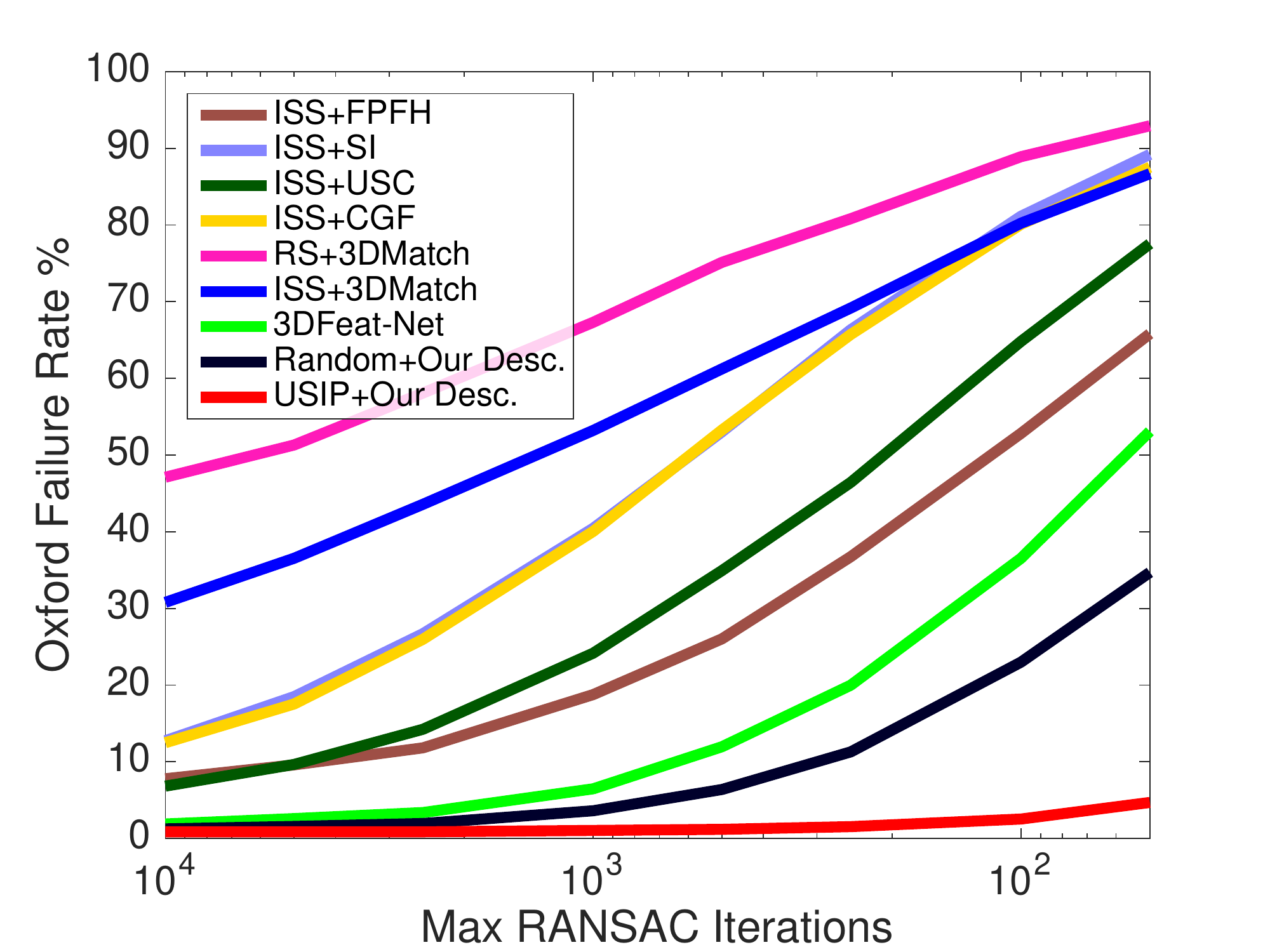}
    \includegraphics[width=0.49\textwidth]{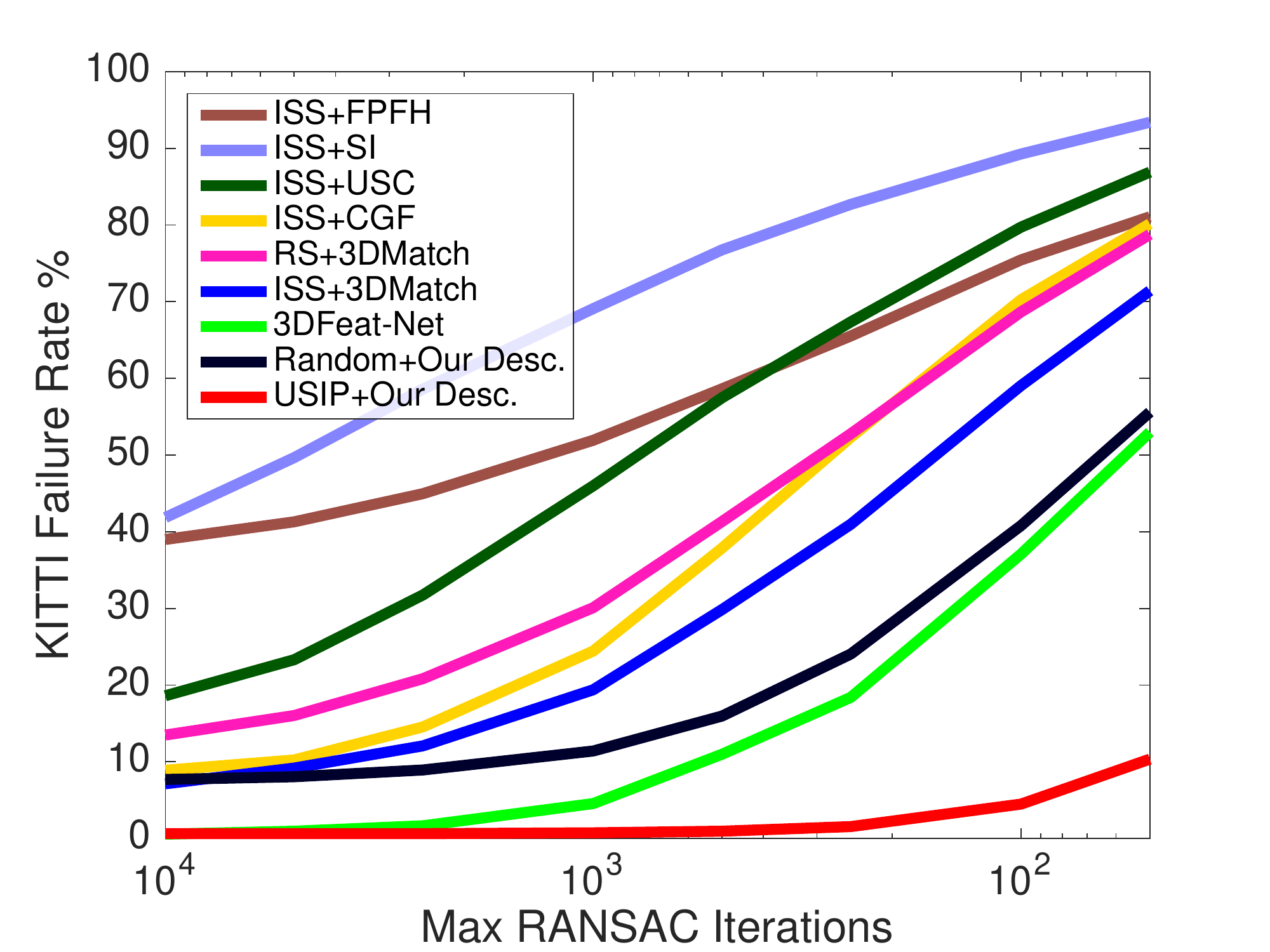}
    \caption{Registration failure rate versus maximum RANSAC iterations in Oxford RobotCar (left) and KITTI (right). Note that the x axis is in logarithmic scale. Our USIP detector + ``Our Desc." (red line) shows very little drop in performance with decreasing number of RANSAC iterations.} \label{fig_suppl_pc_registration_curve}
\end{figure*} 
\begin{figure*}[ht!] 
        \centering
        \subfigure[]{\includegraphics[width=0.3\textwidth]{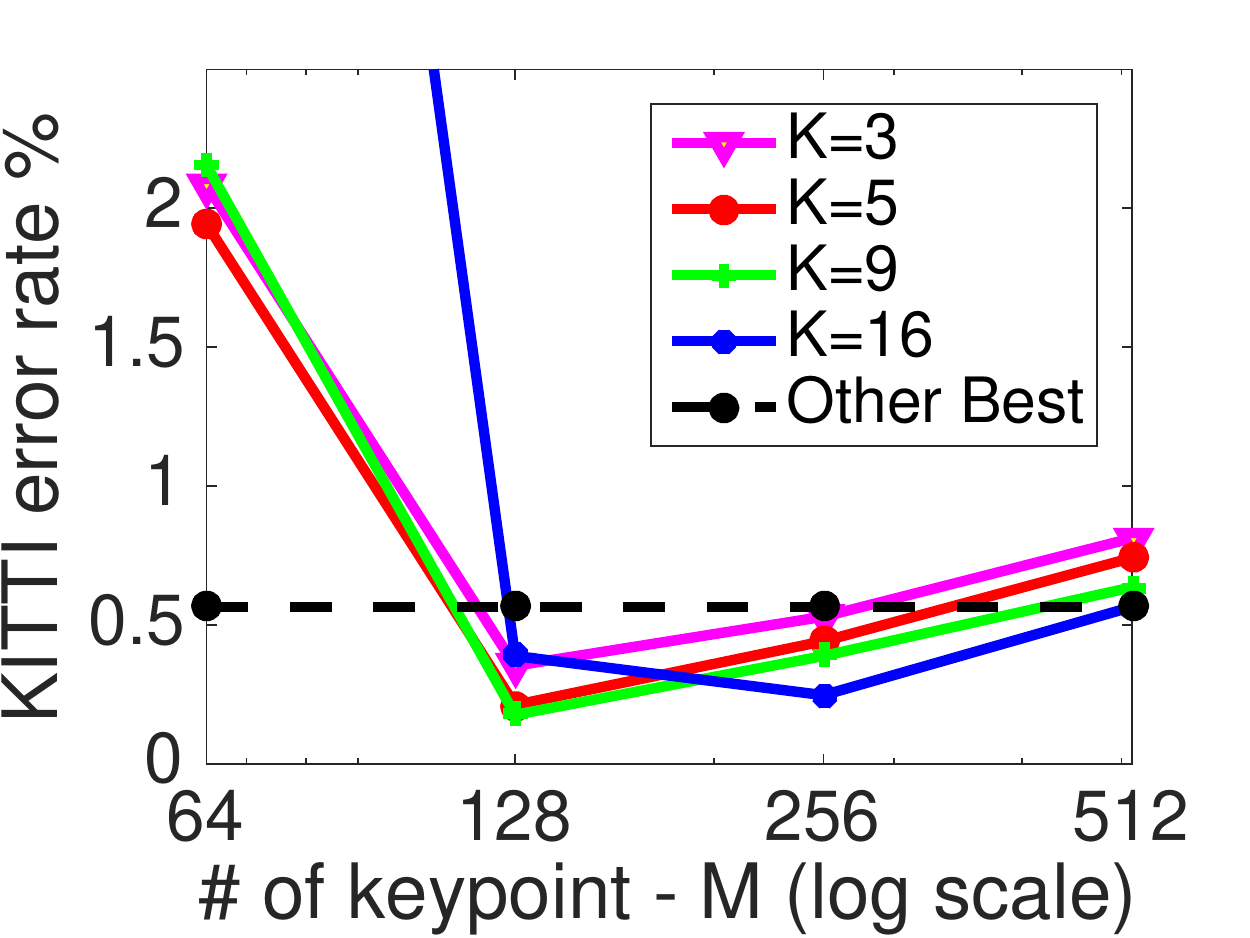}}
        \subfigure[]{\includegraphics[width=0.3\textwidth]{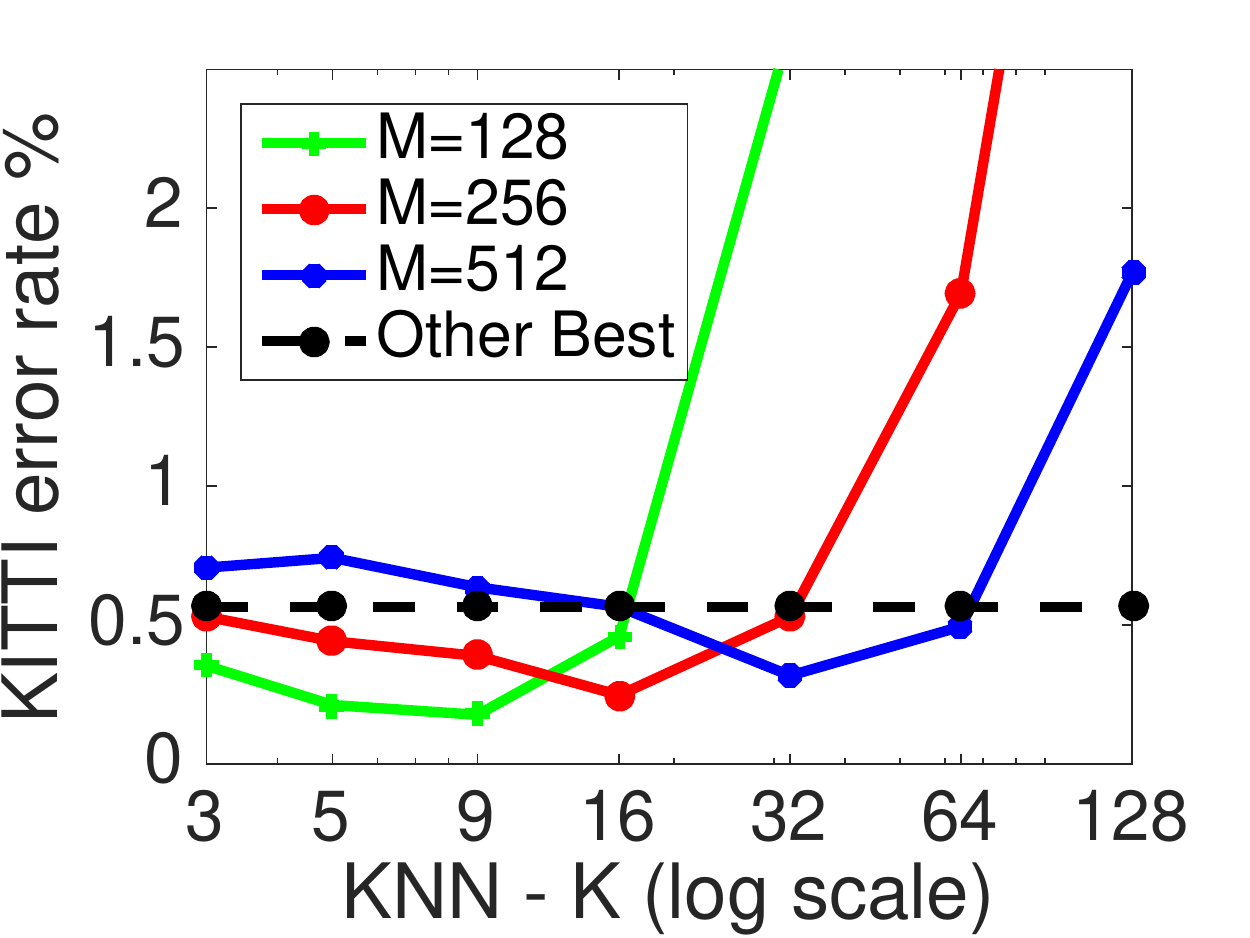}}
        \subfigure[]{\includegraphics[width=0.3\textwidth]{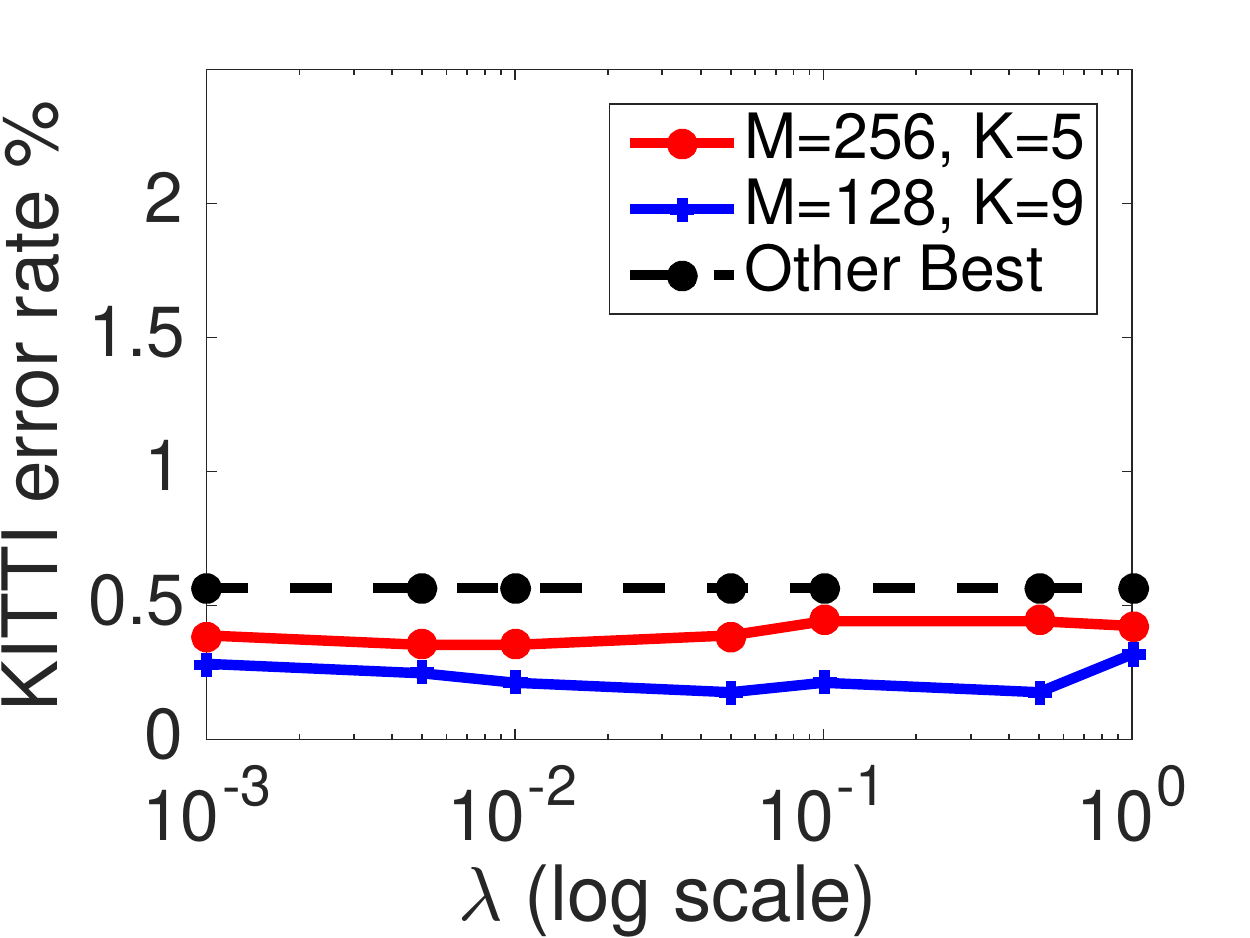}}
        \caption{Point cloud registration error rate (\%) on KITTI (trained on Oxford). Dash line is the best performance of existing methods. $\lambda=0.5$ in (a) (b).}
        \label{fig_suppl_m_k_lambda_error_rate}
\end{figure*}

We follow the experimental setup and pipeline of 3DFeat-Net \cite{jian20183dfeat} to provide more evaluation results on point cloud registration. More specifically, we compare the performance of our USIP detector and ``Our Desc." with other existing keypoint detector and descriptors. The evaluations are done on the Oxford RobotCar and KITTI datasets prepared by \cite{jian20183dfeat}. 
Refer to Sec.~6.2 of the main paper for the details of the registration steps.
A fixed number of 256 keypoints is extracted from each point cloud.
We extract the keypoints without Non-Maximum-Supression (NMS). Furthermore, keypoints with high saliency uncertainty, \ie, large $\sigma$, are filtered out.

\paragraph{Datasets}
The Oxford RobotCar consists of 40 traversals on the same route over a year. 3D point clouds are built by accumulating the 2D scans from SICK LMS-151 LiDAR with the GPS/INS readings. We use 35 traversals, \ie 21,875 point clouds for training. The remaining 5 traversals, \ie, 828 point clouds and 3,426 overlapping pairs are used for evaluation. Random rotations around the up-axis are applied to each evaluation point cloud. In KITTI, 3D point clouds are directly provided by a Velodyne HDL-64E. We use the 2,831 overlapping pairs of point clouds prepared by \cite{jian20183dfeat} for registration evaluation.

\paragraph{Performance}
Tab.~\ref{tbl_suppl_pc_registration_outdoor} shows the point cloud registration performances. Our USIP detector + ``Our Desc." outperforms previous methods with the lowest registration failure rate (Fail \%), Relative Translational Error (RTE), Relative Rotation Error (RRE), and highest inlier ratio (Inlier \%). In particular, our registration failure rate and inlier ratio are respectively 50\% and 2x of the second best keypoint detector + descriptor. 
We further analyze the performance over different number of RANSAC iterations. 
The registration failure rate versus the maximum number of RANSAC iterations is shown in Fig.~\ref{fig_suppl_pc_registration_curve}. 
Due to high repeatability, our USIP detector (red line) shows very little drop in performance with decreasing number of RANSAC iterations, while all other algorithms show rapid drops in performances. 
Additionally, we replace our USIP detector + ``Our Desc." with Random Sampling + ``Our Desc." to demonstrate the effectiveness of our USIP detector. It can be seen from Fig.~\ref{fig_suppl_pc_registration_curve} that the performance of Random Sampling + ``Our Desc." (black line) drops as quickly as other methods with decreasing number of RANSAC iterations.

\paragraph{Effect of USIP Keypoint Saliency Uncertainty $\Sigma$ on Descriptor Training}
We show that the keypoint salicency uncertainty $\Sigma$ from our USIP detector improves the performance of ``Our Desc.". To this end, we compare the performances of ``Our Desc." trained with USIP and randomly sampled keypoints, respectively.
In particular, the weight $w_m$ from Eq.~\ref{equ_desc_weak_loss} or Eq.~\ref{equ_desc_strong_loss} is set to 1 for the randomly sampled keypoints. We denote the descriptor trained with randomly sampled keypoints as ``Desc. w. RS''.
Tab.~\ref{tbl_suppl_exp_usip_on_desc} shows the registration failure rates of ``Desc. w. USIP'' and ``Desc. w. RS''. The results show that ``Desc. w. USIP'' performs better than ``Desc. w. RS'', which means that keypoints and saliency uncertainty $\Sigma$ from our USIP detector improve descriptor training. 
\begin{table}[H]
\centering
\resizebox{0.48\textwidth}{!}{%
\setlength\tabcolsep{4pt} 
\begin{tabular}{l|cc|cc}
\hline
\multirow{2}{*}{Failure \%} & \multicolumn{2}{c|}{Oxford} & \multicolumn{2}{c}{KITTI} \\ \cline{2-5} 
                         & Desc w. USIP   & Desc w. RS  & Desc w. USIP  & Desc w. RS \\ \hline
USIP                     & \textbf{0.93}          & 1.20        & \textbf{0.24}         & 1.02       \\ \hline
\end{tabular}%
}
\caption{Registration failure rate for ``Our Desc." trained keypoints from our USIP detector and randomly sampled keypoints.}
\label{tbl_suppl_exp_usip_on_desc}
\end{table}
\paragraph{Effect of Parameters $M, K, \lambda$}
We demonstrate the point cloud registration failure rate (\%) in Fig.~\ref{fig_suppl_m_k_lambda_error_rate}, when various USIP detector parameters, $M, K, \lambda$, are selected. In Fig.~\ref{fig_suppl_m_k_lambda_error_rate} we use the same descriptor mentioned in Sec.~\ref{sec_suppl_descriptor}. As shown in Fig.~\ref{fig_suppl_m_k_lambda_error_rate}, our method outperforms existing methods over a wide range of $M, K, \lambda$. We notice our network performance decreases significantly when $M$ is too small or $K$ is too large, \ie, the receptive is too large. This further verifies our design of limiting the receptive field. In addition, Fig.~\ref{fig_suppl_m_k_lambda_error_rate} shows that the registration failure rate remains satisfying when $\lambda$ is small. This is consistent with Fig.~\ref{fig_suppl_repeatability_lambda} that our USIP is able to detect repeatable keypoints even without the point-to-point loss. Nonetheless, it is still important to include the point-to-point loss to ensure that the keypoints are close to the input point cloud. 

%

\begin{figure*}[t] \centering
    \includegraphics[width=0.16\textwidth]{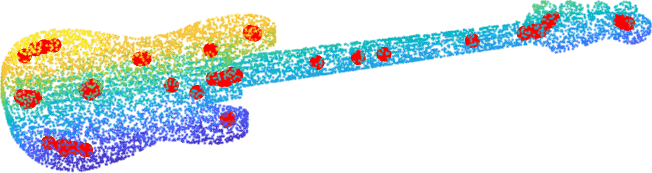}
    \includegraphics[width=0.16\textwidth]{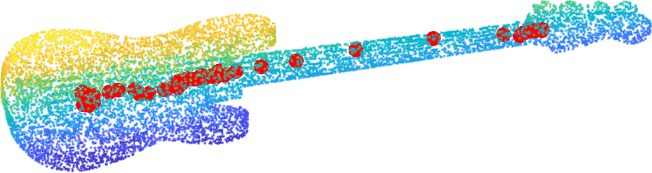}
    \includegraphics[width=0.16\textwidth]{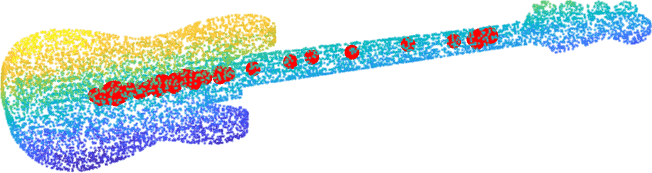}
    \includegraphics[width=0.16\textwidth]{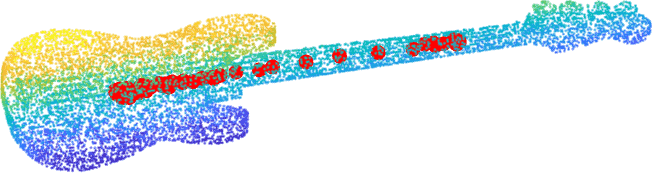}
    \includegraphics[width=0.16\textwidth]{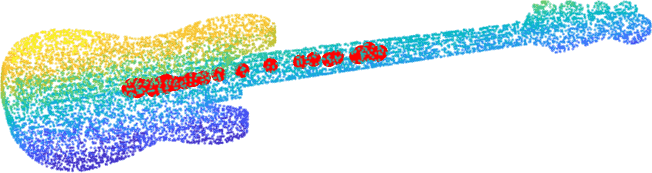}
    \includegraphics[width=0.16\textwidth]{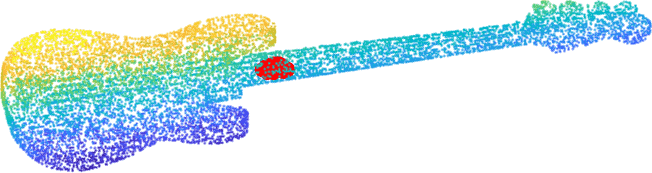}
    
    \includegraphics[width=0.16\textwidth]{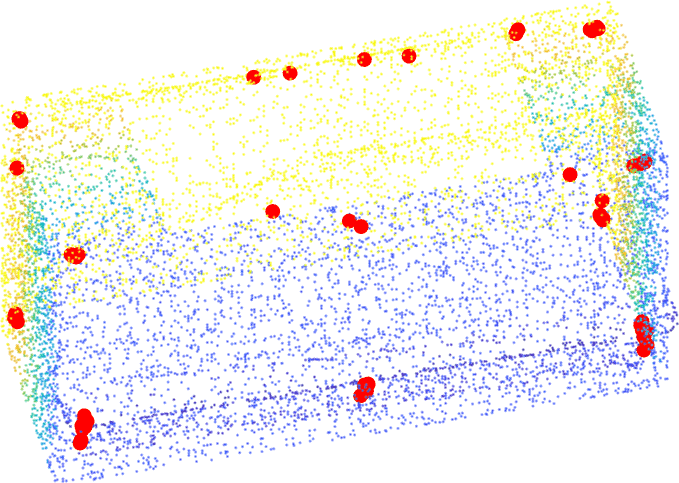}
    \includegraphics[width=0.16\textwidth]{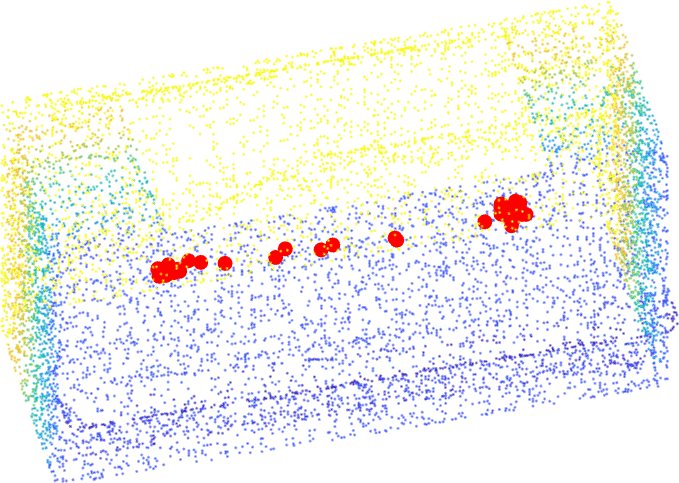}
    \includegraphics[width=0.16\textwidth]{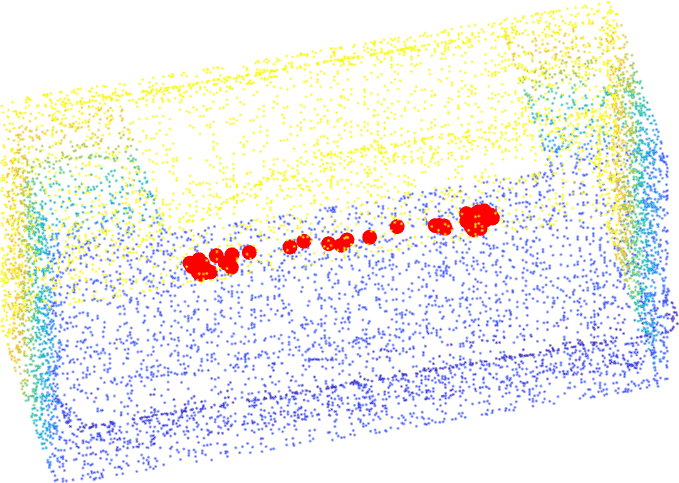}
    \includegraphics[width=0.16\textwidth]{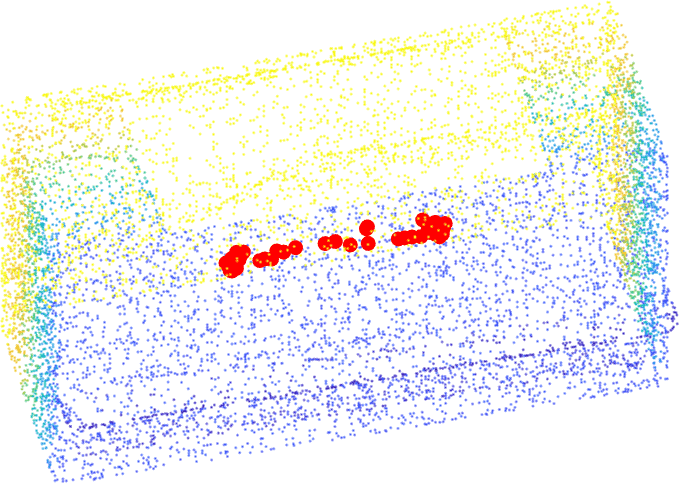}
    \includegraphics[width=0.16\textwidth]{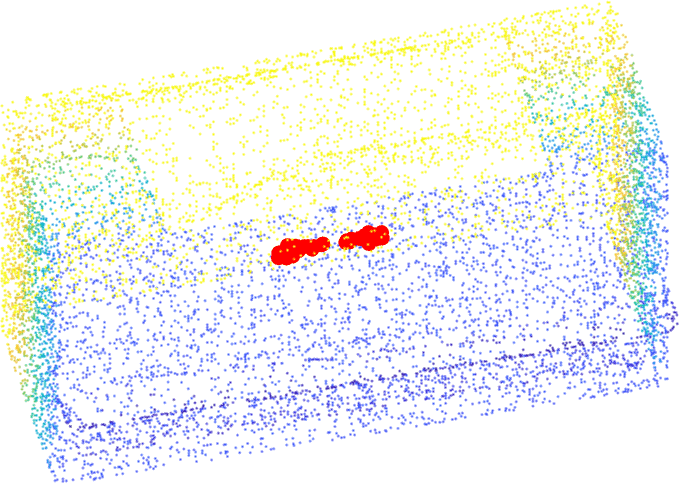}
    \includegraphics[width=0.16\textwidth]{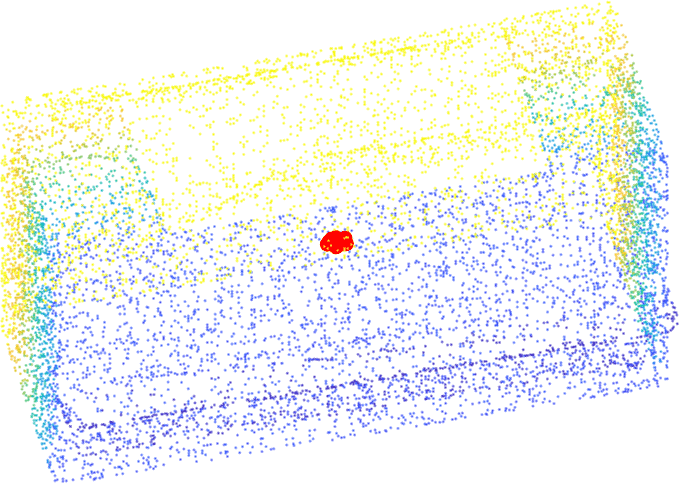}
    
    \includegraphics[width=0.16\textwidth, height=0.16\textwidth, keepaspectratio, angle=90]{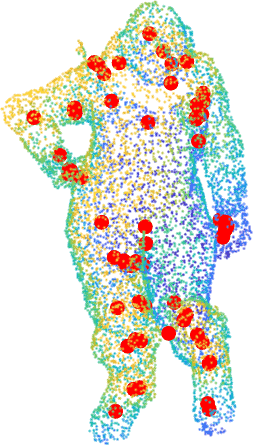}
    \includegraphics[width=0.16\textwidth, height=0.16\textwidth, keepaspectratio, angle=90]{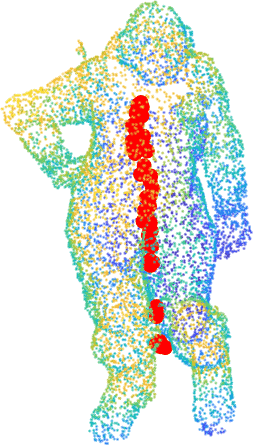}
    \includegraphics[width=0.16\textwidth, height=0.16\textwidth, keepaspectratio, angle=90]{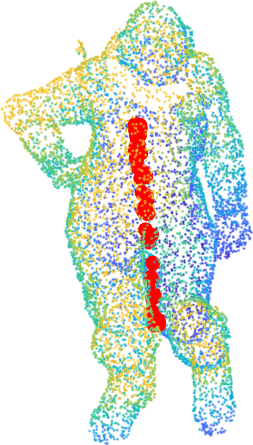}
    \includegraphics[width=0.16\textwidth, height=0.16\textwidth, keepaspectratio, angle=90]{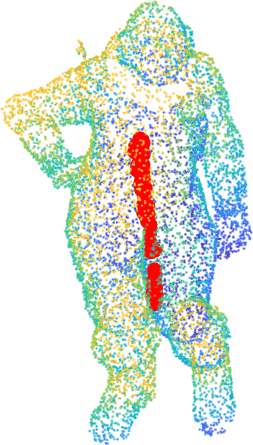}
    \includegraphics[width=0.16\textwidth, height=0.16\textwidth, keepaspectratio, angle=90]{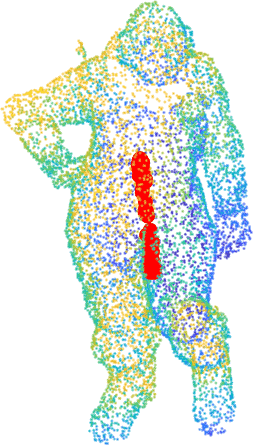}
    \includegraphics[width=0.16\textwidth, height=0.16\textwidth, keepaspectratio, angle=90]{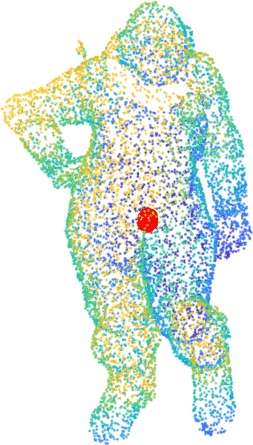}
    
    \includegraphics[width=0.16\textwidth, height=0.16\textwidth, keepaspectratio, angle=90]{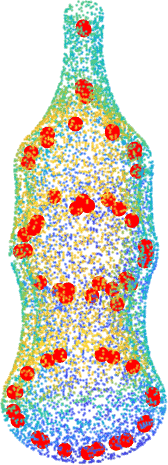}
    \includegraphics[width=0.16\textwidth, height=0.16\textwidth, keepaspectratio, angle=90]{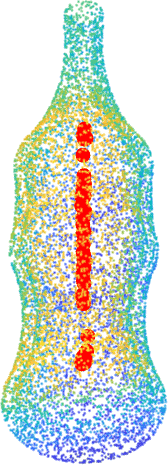}
    \includegraphics[width=0.16\textwidth, height=0.16\textwidth, keepaspectratio, angle=90]{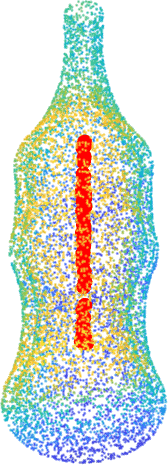}
    \includegraphics[width=0.16\textwidth, height=0.16\textwidth, keepaspectratio, angle=90]{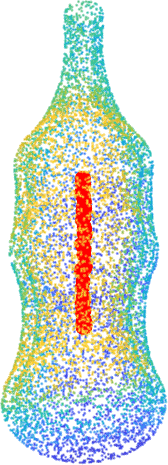}
    \includegraphics[width=0.16\textwidth, height=0.16\textwidth, keepaspectratio, angle=90]{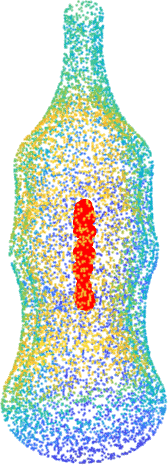}
    \includegraphics[width=0.16\textwidth, height=0.16\textwidth, keepaspectratio, angle=90]{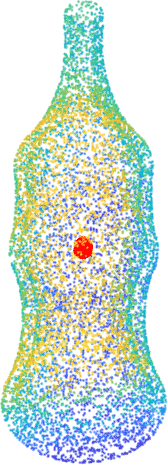}
    
    \includegraphics[width=0.16\textwidth, height=0.16\textwidth, keepaspectratio, angle=90]{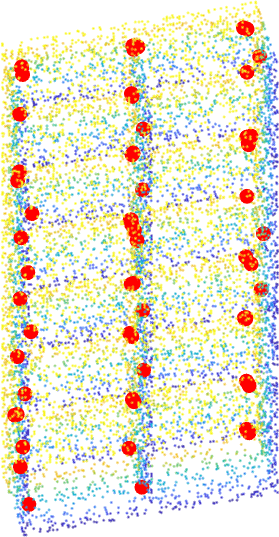}
    \includegraphics[width=0.16\textwidth, height=0.16\textwidth, keepaspectratio, angle=90]{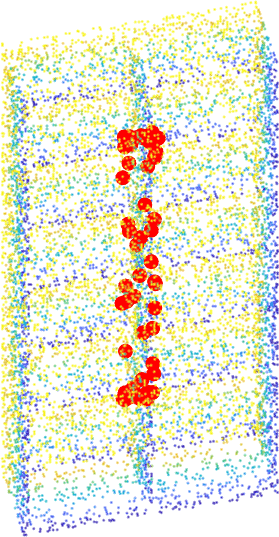}
    \includegraphics[width=0.16\textwidth, height=0.16\textwidth, keepaspectratio, angle=90]{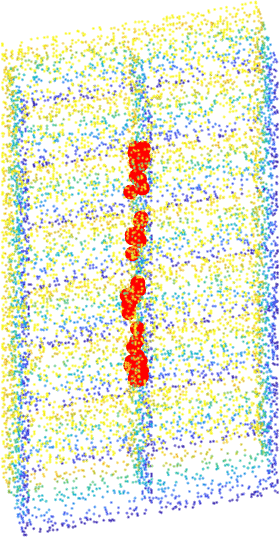}
    \includegraphics[width=0.16\textwidth, height=0.16\textwidth, keepaspectratio, angle=90]{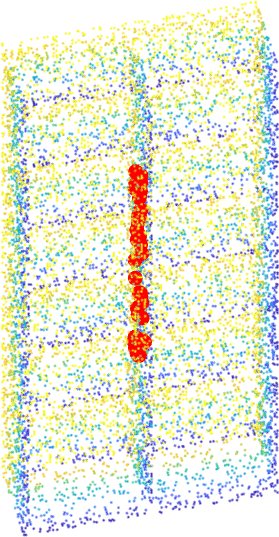}
    \includegraphics[width=0.16\textwidth, height=0.16\textwidth, keepaspectratio, angle=90]{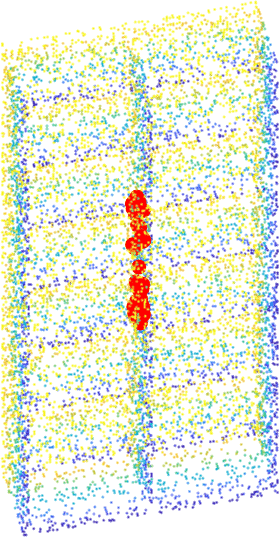}
    \includegraphics[width=0.16\textwidth, height=0.16\textwidth, keepaspectratio, angle=90]{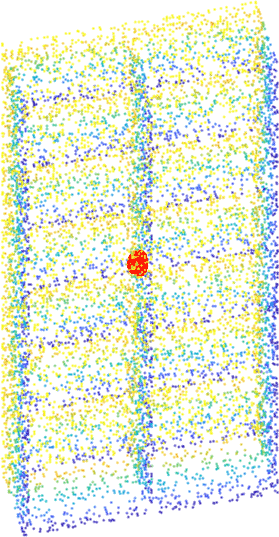}
    
    \includegraphics[width=0.16\textwidth]{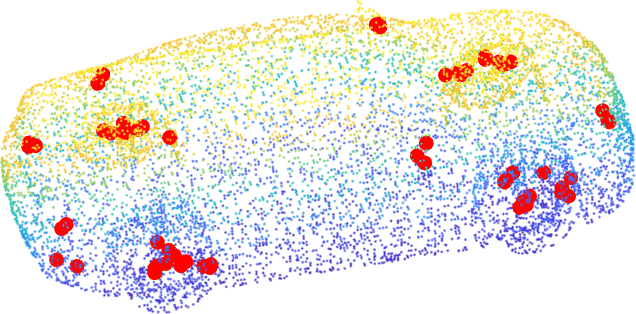}
    \includegraphics[width=0.16\textwidth]{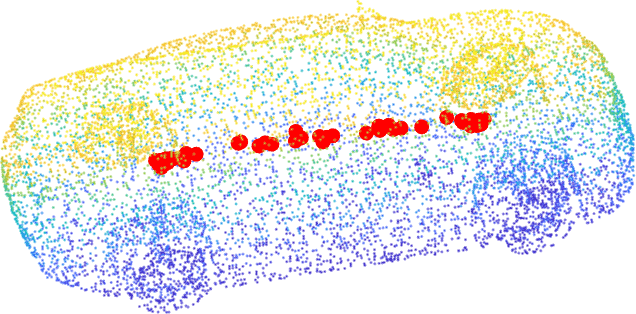}
    \includegraphics[width=0.16\textwidth]{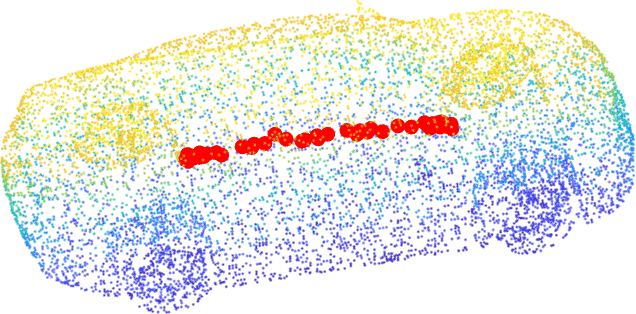}
    \includegraphics[width=0.16\textwidth]{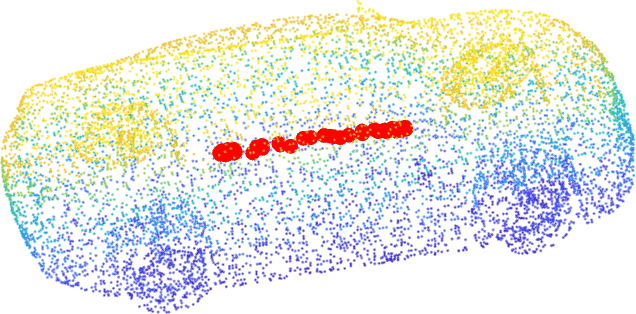}
    \includegraphics[width=0.16\textwidth]{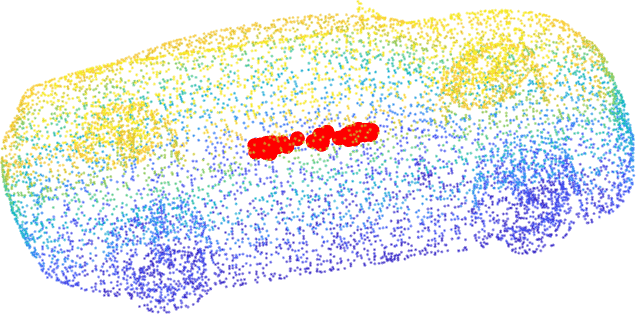}
    \includegraphics[width=0.16\textwidth]{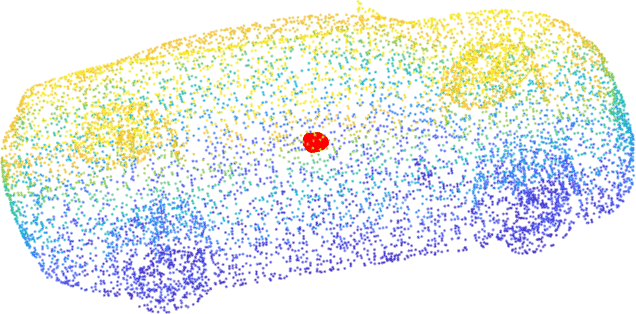}

    \caption{Visualization of FPN degeneracy. $M=64$ and from left to right: $K=9, 24, 32, 40, 48, 64$, \ie, receptive field of FPN increases from left to right.} \label{fig_suppl_vis_degeneracy_m64}
\end{figure*} 
\begin{figure*}[t] \centering
    \includegraphics[width=0.16\textwidth]{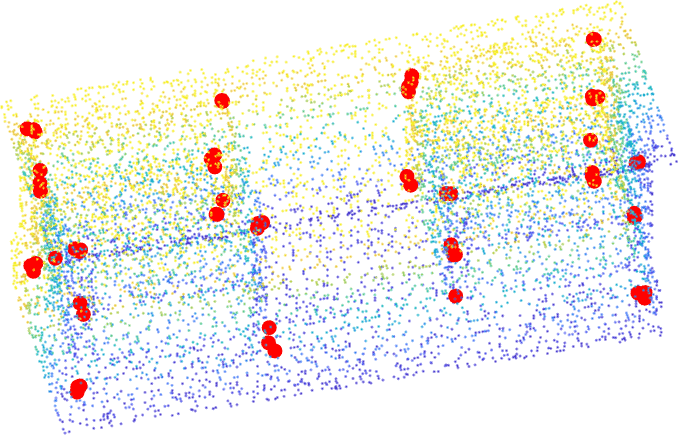}
    \includegraphics[width=0.16\textwidth]{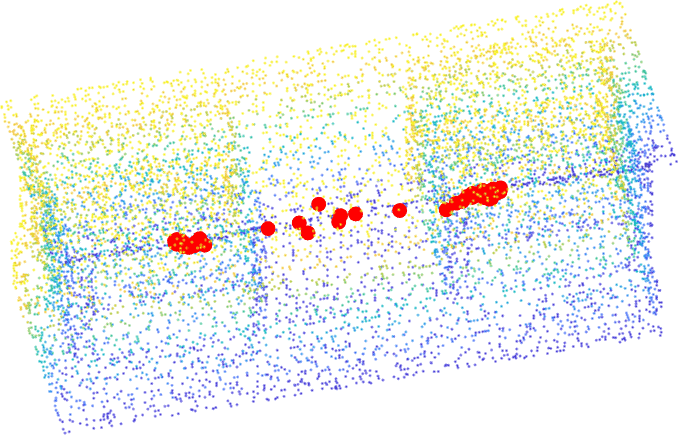}
    \includegraphics[width=0.16\textwidth]{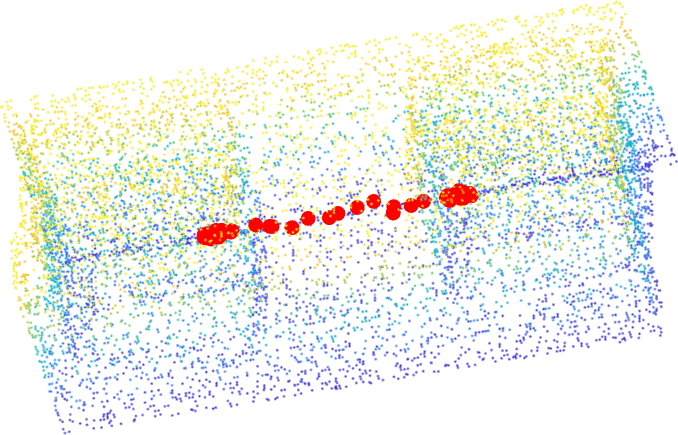}
    \includegraphics[width=0.16\textwidth]{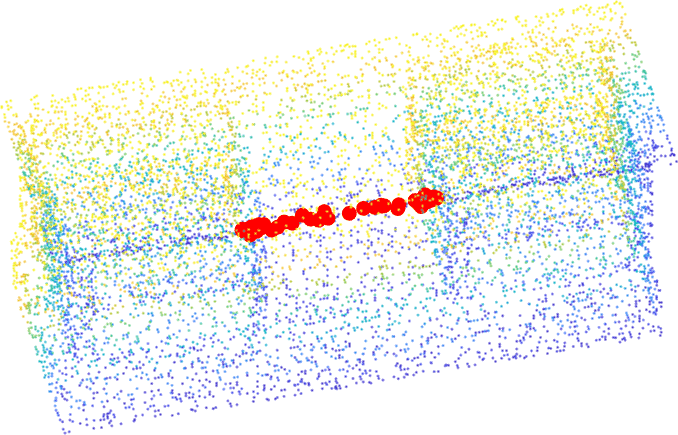}
    \includegraphics[width=0.16\textwidth]{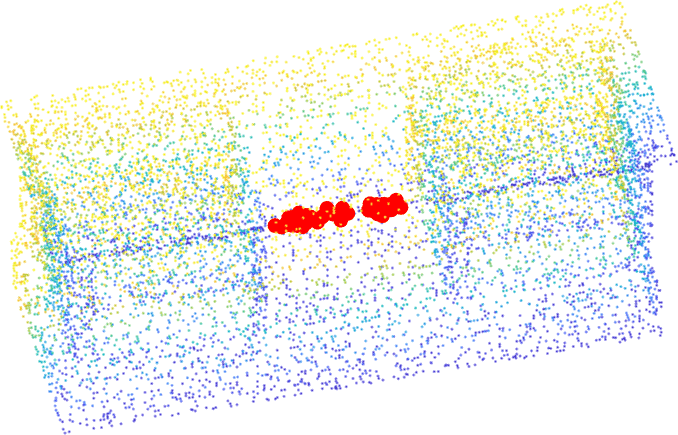}
    \includegraphics[width=0.16\textwidth]{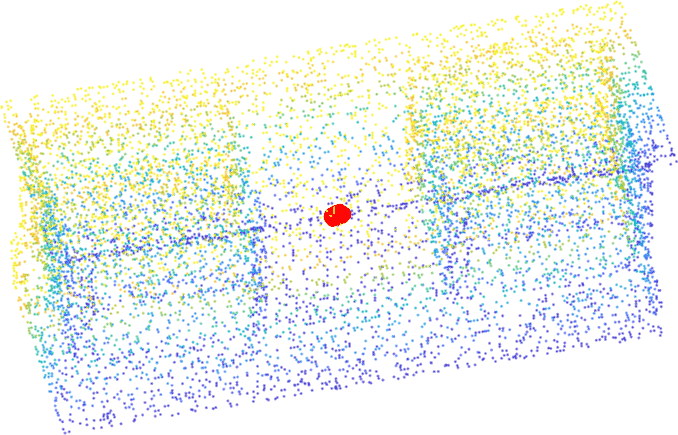}
    
    \includegraphics[width=0.16\textwidth]{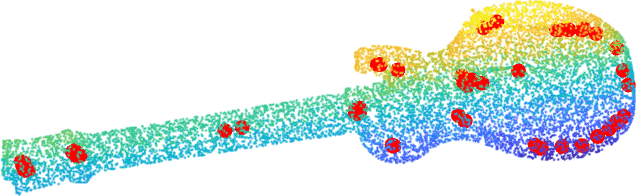}
    \includegraphics[width=0.16\textwidth]{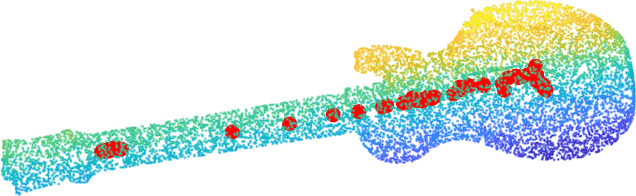}
    \includegraphics[width=0.16\textwidth]{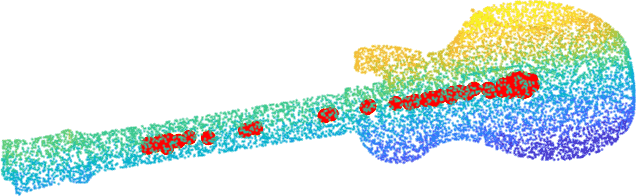}
    \includegraphics[width=0.16\textwidth]{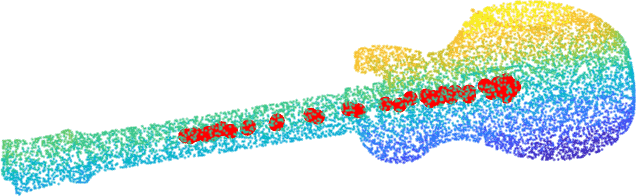}
    \includegraphics[width=0.16\textwidth]{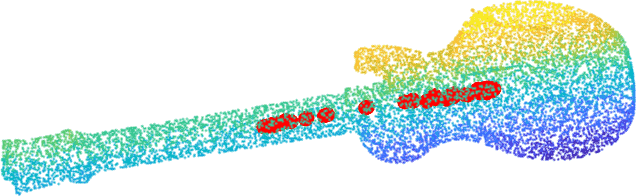}
    \includegraphics[width=0.16\textwidth]{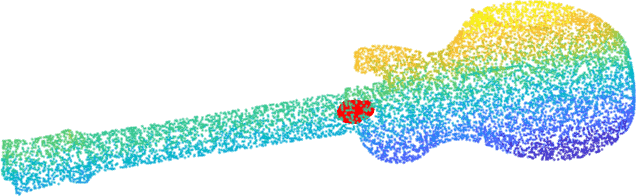}
    
    \includegraphics[width=0.16\textwidth, height=0.16\textwidth, keepaspectratio, angle=90]{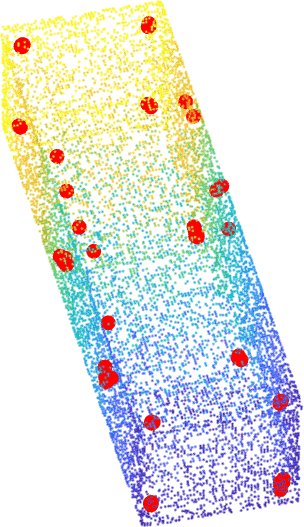}
    \includegraphics[width=0.16\textwidth, height=0.16\textwidth, keepaspectratio, angle=90]{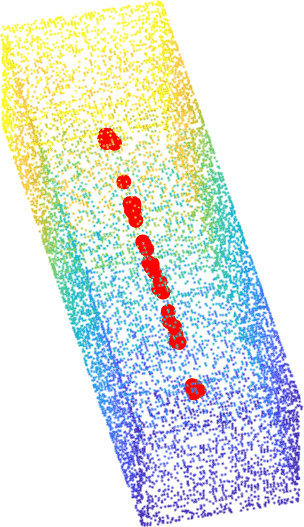}
    \includegraphics[width=0.16\textwidth, height=0.16\textwidth, keepaspectratio, angle=90]{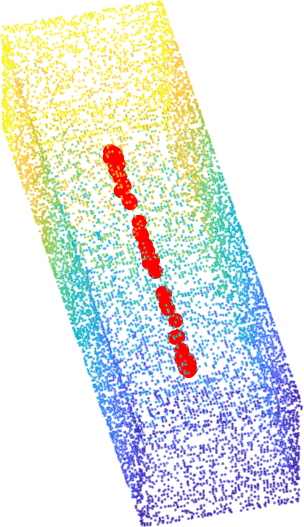}
    \includegraphics[width=0.16\textwidth, height=0.16\textwidth, keepaspectratio, angle=90]{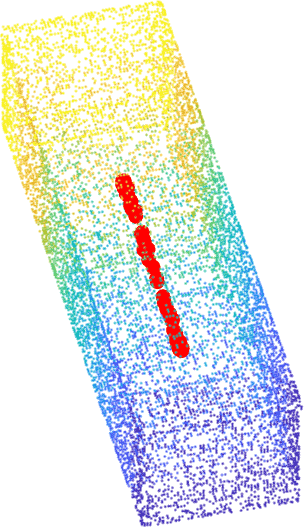}
    \includegraphics[width=0.16\textwidth, height=0.16\textwidth, keepaspectratio, angle=90]{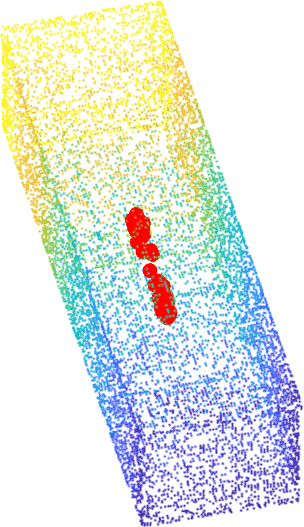}
    \includegraphics[width=0.16\textwidth, height=0.16\textwidth, keepaspectratio, angle=90]{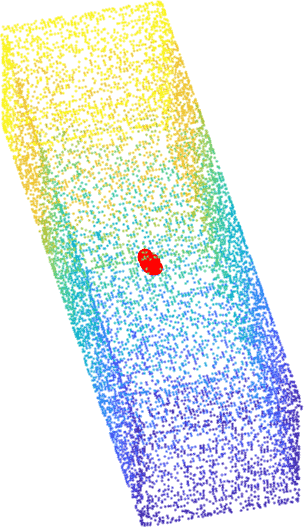}
    
    \includegraphics[width=0.16\textwidth, height=0.16\textwidth, keepaspectratio, angle=90]{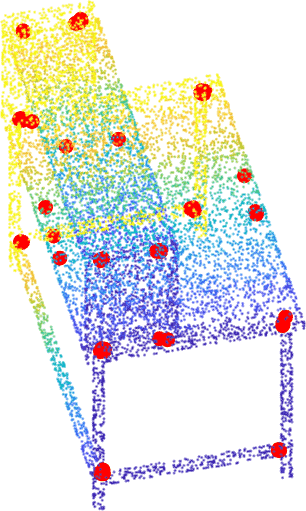}
    \includegraphics[width=0.16\textwidth, height=0.16\textwidth, keepaspectratio, angle=90]{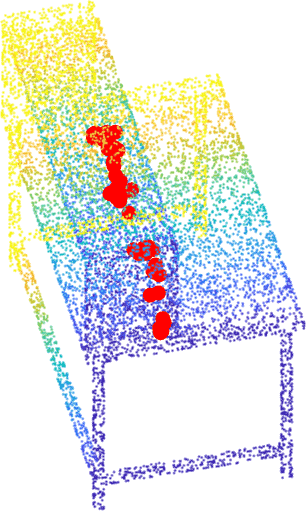}
    \includegraphics[width=0.16\textwidth, height=0.16\textwidth, keepaspectratio, angle=90]{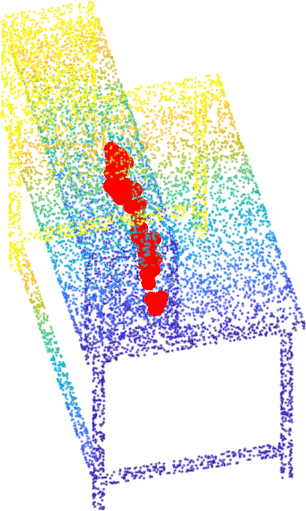}
    \includegraphics[width=0.16\textwidth, height=0.16\textwidth, keepaspectratio, angle=90]{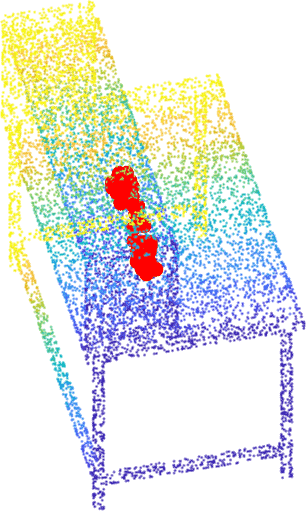}
    \includegraphics[width=0.16\textwidth, height=0.16\textwidth, keepaspectratio, angle=90]{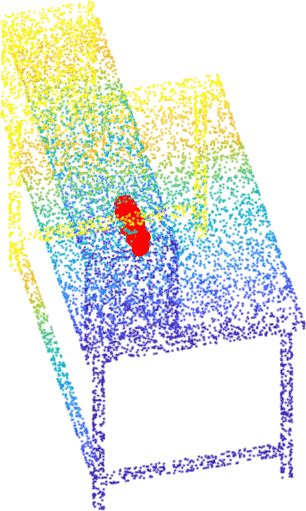}
    \includegraphics[width=0.16\textwidth, height=0.16\textwidth, keepaspectratio, angle=90]{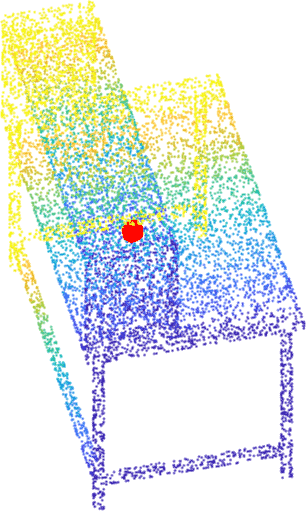}
    
    \includegraphics[width=0.16\textwidth, height=0.16\textwidth, keepaspectratio, angle=90]{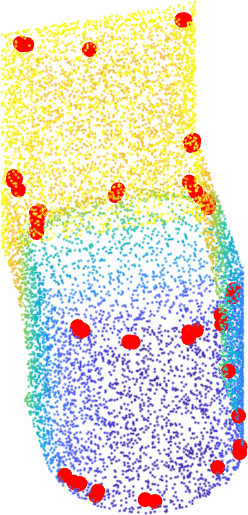}
    \includegraphics[width=0.16\textwidth, height=0.16\textwidth, keepaspectratio, angle=90]{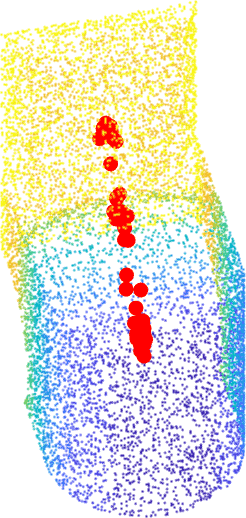}
    \includegraphics[width=0.16\textwidth, height=0.16\textwidth, keepaspectratio, angle=90]{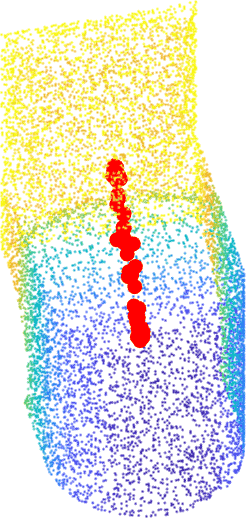}
    \includegraphics[width=0.16\textwidth, height=0.16\textwidth, keepaspectratio, angle=90]{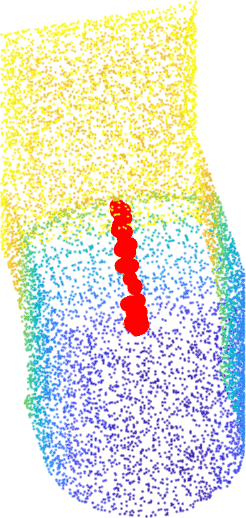}
    \includegraphics[width=0.16\textwidth, height=0.16\textwidth, keepaspectratio, angle=90]{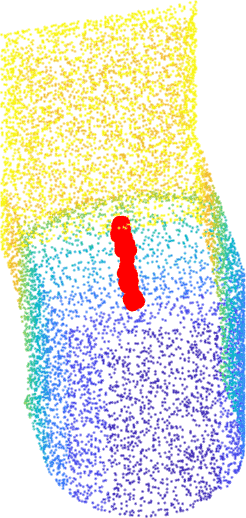}
    \includegraphics[width=0.16\textwidth, height=0.16\textwidth, keepaspectratio, angle=90]{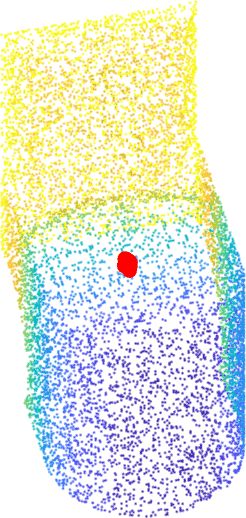}
    
    \includegraphics[width=0.16\textwidth, height=0.16\textwidth, keepaspectratio, angle=90]{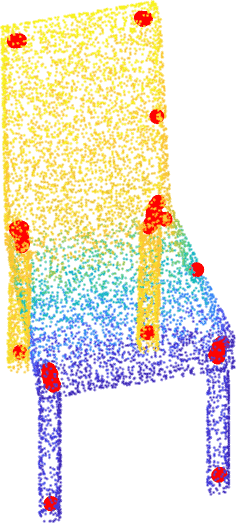}
    \includegraphics[width=0.16\textwidth, height=0.16\textwidth, keepaspectratio, angle=90]{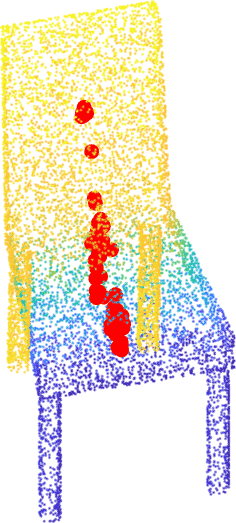}
    \includegraphics[width=0.16\textwidth, height=0.16\textwidth, keepaspectratio, angle=90]{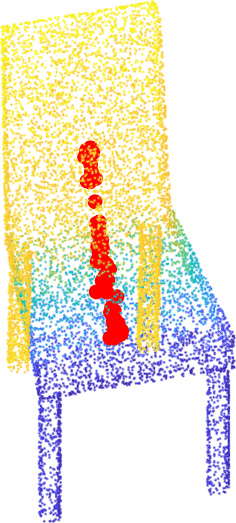}
    \includegraphics[width=0.16\textwidth, height=0.16\textwidth, keepaspectratio, angle=90]{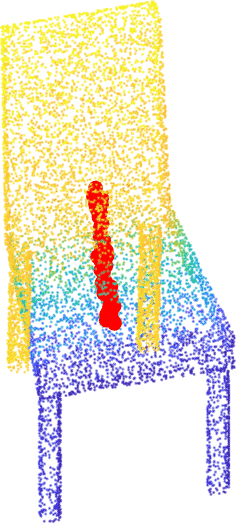}
    \includegraphics[width=0.16\textwidth, height=0.16\textwidth, keepaspectratio, angle=90]{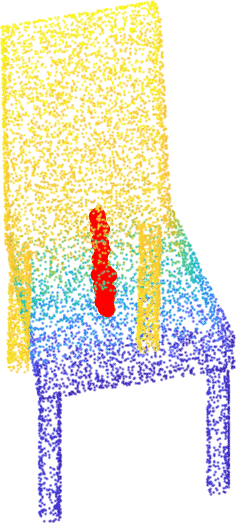}
    \includegraphics[width=0.16\textwidth, height=0.16\textwidth, keepaspectratio, angle=90]{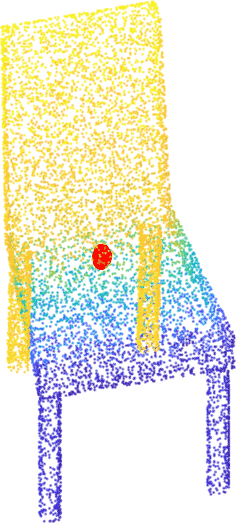}
    
    \caption{Visualization of FPN degeneracy. $K=9$ and from left to right: $M=64, 24, 20, 16, 12, 9$, \ie, receptive field of FPN increases from left to right.} \label{fig_suppl_vis_degeneracy_k9}
\end{figure*} 

\section{Qualitative Visualization of USIP Keypoints} \label{sec_suppl_visualization}
We show more visualizations of the keypoints detected from our USIP detector on ModelNet40, KITTI, Oxford RobotCar, Redwood in Fig.~\ref{fig_suppl_vis_modelnet}, \ref{fig_suppl_vis_kitti}, \ref{fig_suppl_vis_oxford}, \ref{fig_suppl_vis_redwood}, respectively. NMS and $\Sigma$ thresholding are applied here. 
A limitation of our USIP detector is shown in Fig.~\ref{fig_suppl_vis_modelnet}, where there are no or very few keypoints on objects that are highly symmetrical or with smooth surfaces. 
The saliency uncertainties $\Sigma$ of the keypoints detected on these objects are large, thus discarded by the $\Sigma$ thresholding.

\begin{figure*}[t] \centering
    \includegraphics[width=0.18\textwidth, height=0.18\textwidth, keepaspectratio]{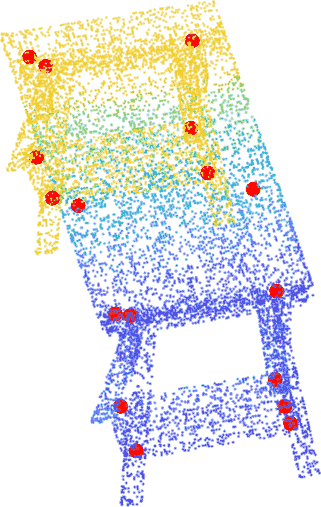}
    \includegraphics[width=0.18\textwidth, height=0.18\textwidth, keepaspectratio]{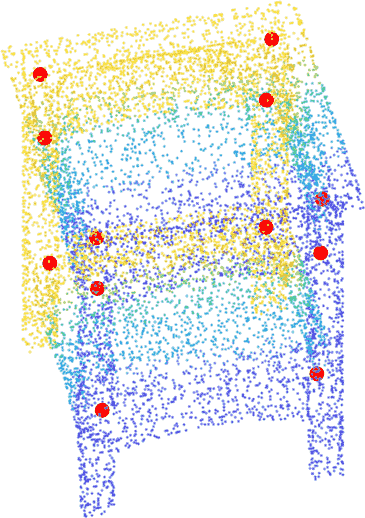}
    \includegraphics[width=0.18\textwidth, height=0.18\textwidth, keepaspectratio]{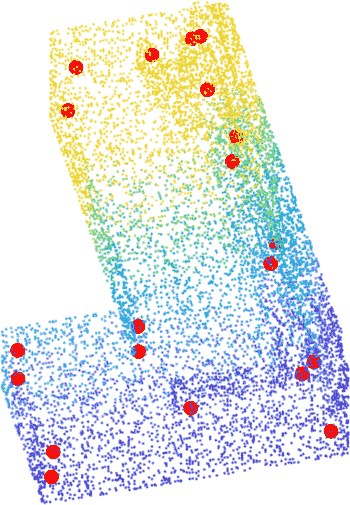}
    \includegraphics[width=0.18\textwidth, height=0.18\textwidth, keepaspectratio]{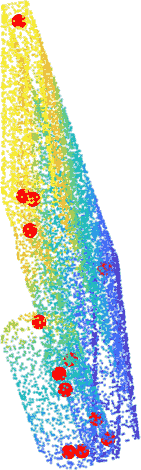}
    \includegraphics[width=0.18\textwidth, height=0.18\textwidth, keepaspectratio]{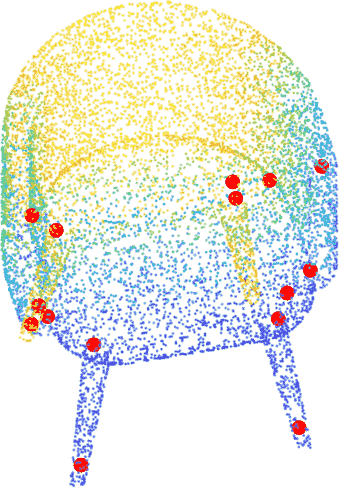}
    \includegraphics[width=0.18\textwidth, height=0.18\textwidth, keepaspectratio]{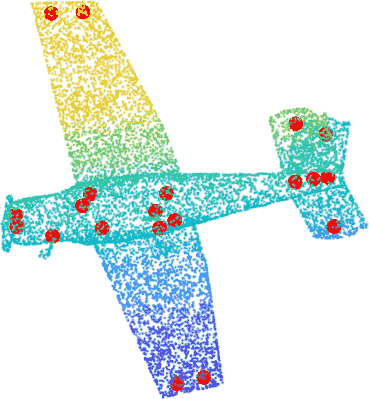}
    \includegraphics[width=0.18\textwidth, height=0.18\textwidth, keepaspectratio]{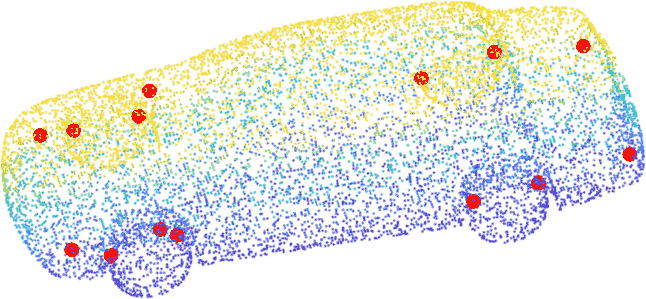}
    \includegraphics[width=0.18\textwidth, height=0.18\textwidth, keepaspectratio]{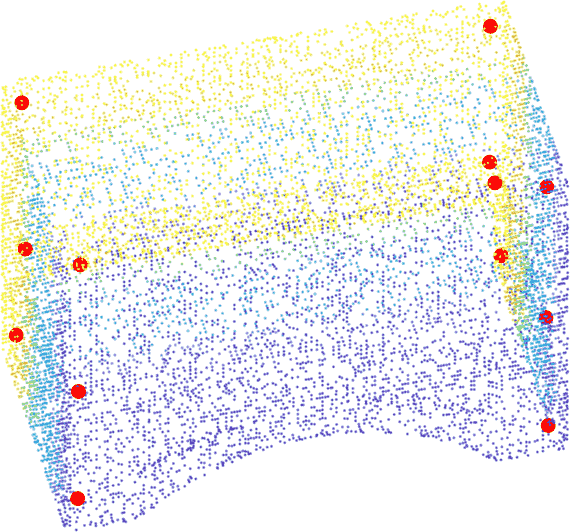}
    \includegraphics[width=0.18\textwidth, height=0.18\textwidth, keepaspectratio]{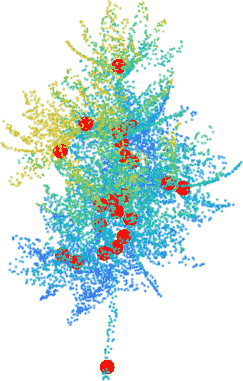}
    \includegraphics[width=0.18\textwidth, height=0.18\textwidth, keepaspectratio]{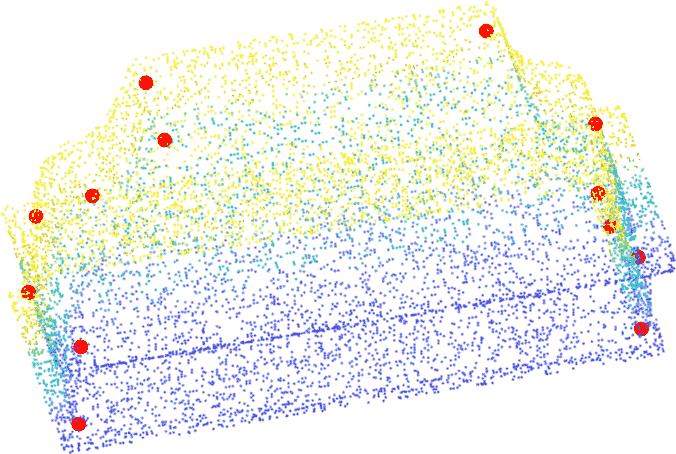}
    \includegraphics[width=0.18\textwidth, height=0.18\textwidth, keepaspectratio]{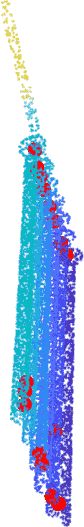}
    \includegraphics[width=0.18\textwidth, height=0.18\textwidth, keepaspectratio]{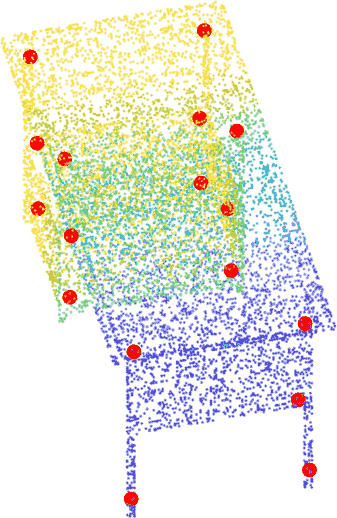}
    \includegraphics[width=0.18\textwidth, height=0.18\textwidth, keepaspectratio]{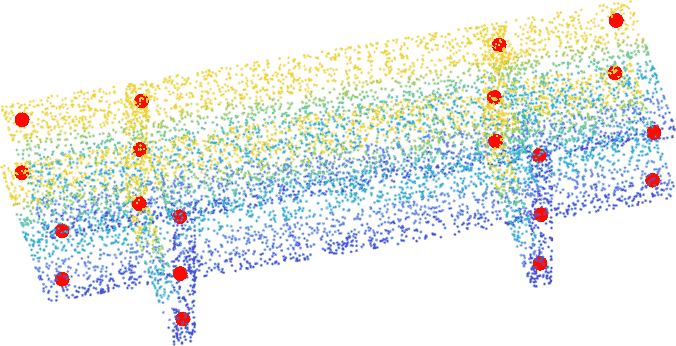}
    \includegraphics[width=0.18\textwidth, height=0.18\textwidth, keepaspectratio]{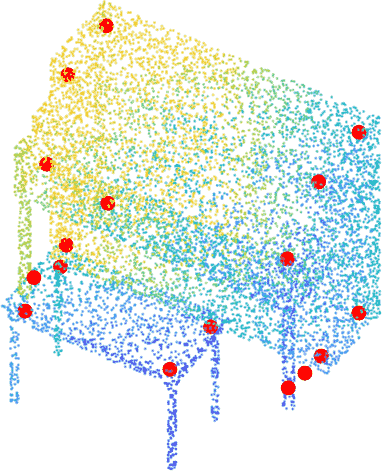}
    \includegraphics[width=0.18\textwidth, height=0.18\textwidth, keepaspectratio]{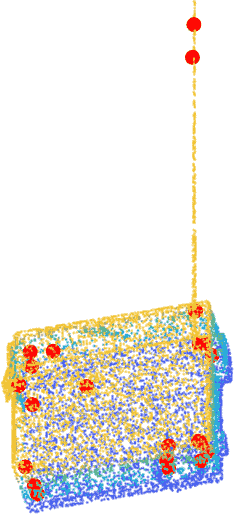}
    \includegraphics[width=0.18\textwidth, height=0.18\textwidth, keepaspectratio]{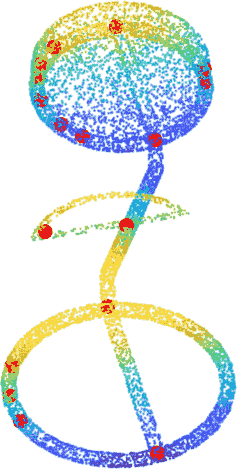}
    \includegraphics[width=0.18\textwidth, height=0.18\textwidth, keepaspectratio]{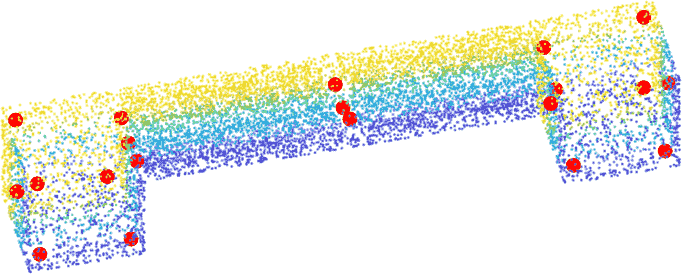}
    \includegraphics[width=0.18\textwidth, height=0.18\textwidth, keepaspectratio]{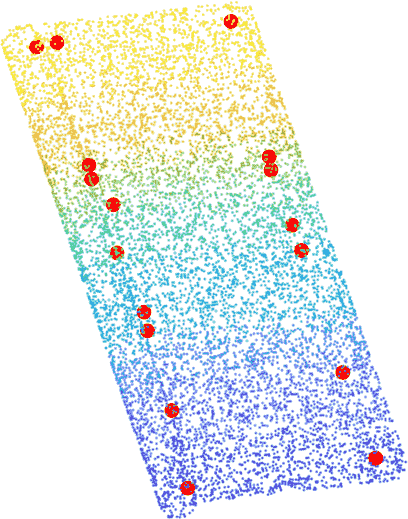}
    \includegraphics[width=0.18\textwidth, height=0.18\textwidth, keepaspectratio]{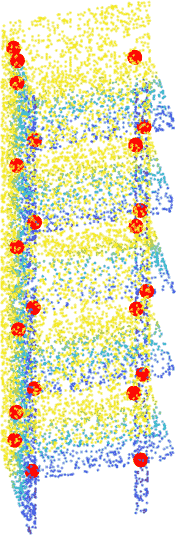}
    \includegraphics[width=0.18\textwidth, height=0.18\textwidth, keepaspectratio]{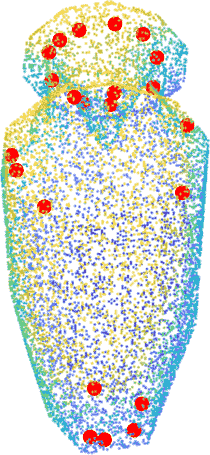}
    \includegraphics[width=0.18\textwidth, height=0.18\textwidth, keepaspectratio]{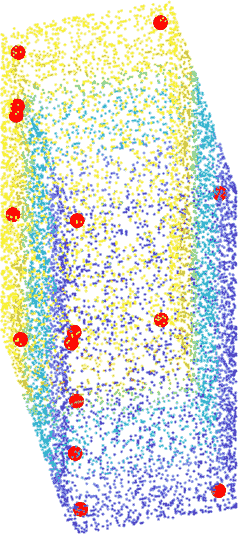}
    \includegraphics[width=0.18\textwidth, height=0.18\textwidth, keepaspectratio]{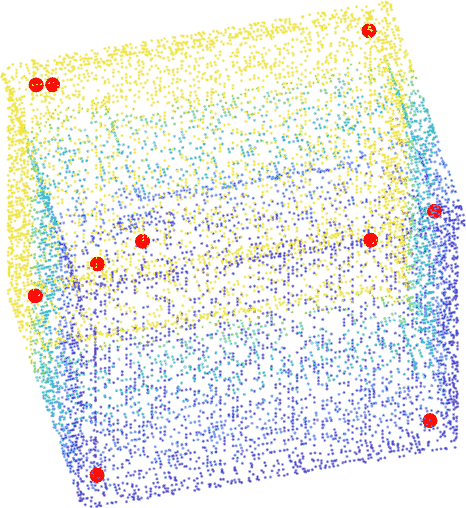}
    \includegraphics[width=0.18\textwidth, height=0.18\textwidth, keepaspectratio]{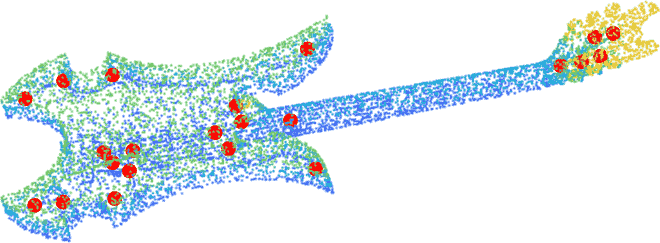}
    \includegraphics[width=0.18\textwidth, height=0.18\textwidth, keepaspectratio]{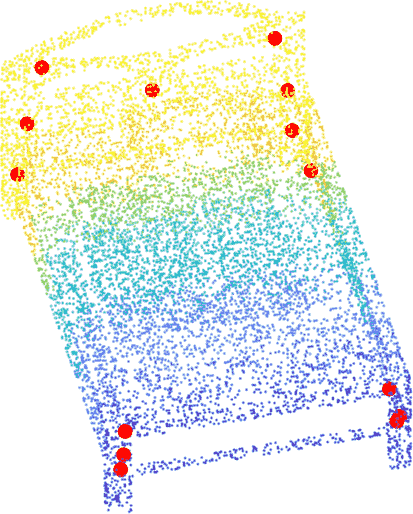}
    \includegraphics[width=0.18\textwidth, height=0.18\textwidth, keepaspectratio]{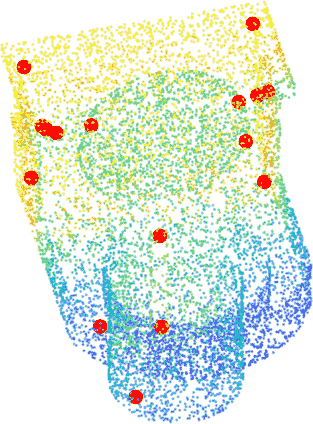}
    \includegraphics[width=0.18\textwidth, height=0.18\textwidth, keepaspectratio]{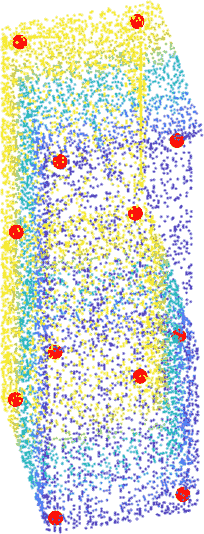}
    \includegraphics[width=0.18\textwidth, height=0.18\textwidth, keepaspectratio]{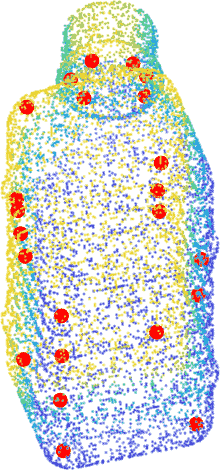}
    \includegraphics[width=0.18\textwidth, height=0.18\textwidth, keepaspectratio]{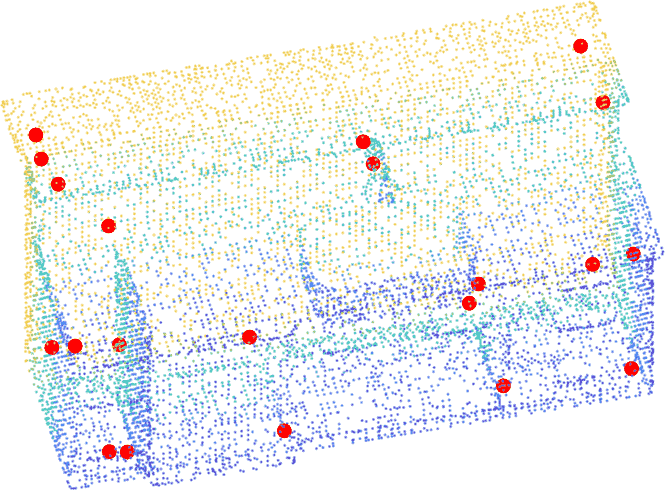}
    \includegraphics[width=0.18\textwidth, height=0.18\textwidth, keepaspectratio]{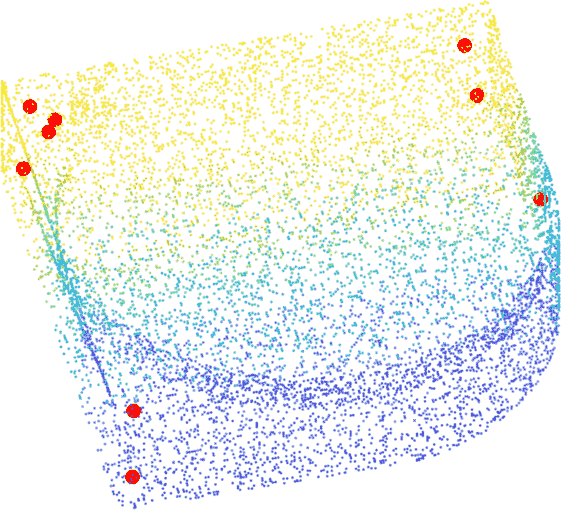}
    \includegraphics[width=0.18\textwidth, height=0.18\textwidth, keepaspectratio]{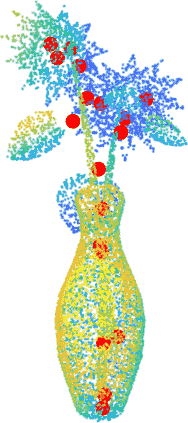}
    \includegraphics[width=0.18\textwidth, height=0.18\textwidth, keepaspectratio]{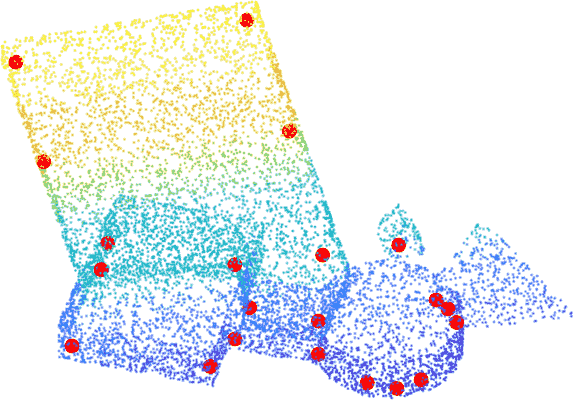}
    \includegraphics[width=0.18\textwidth, height=0.18\textwidth, keepaspectratio]{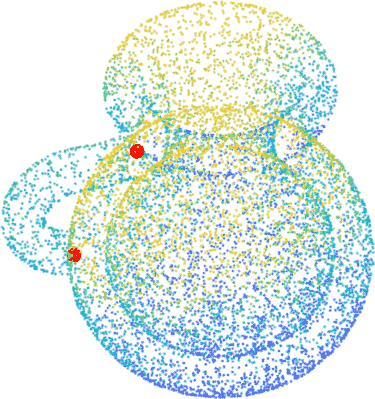}
    \includegraphics[width=0.18\textwidth, height=0.18\textwidth, keepaspectratio]{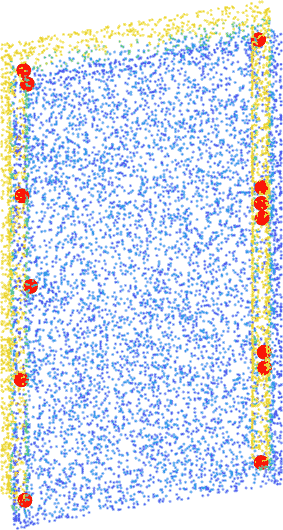}
    \includegraphics[width=0.18\textwidth, height=0.18\textwidth, keepaspectratio]{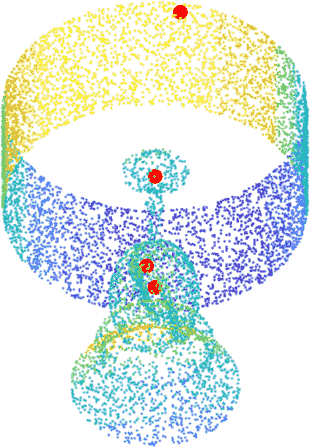}
    \includegraphics[width=0.18\textwidth, height=0.18\textwidth, keepaspectratio]{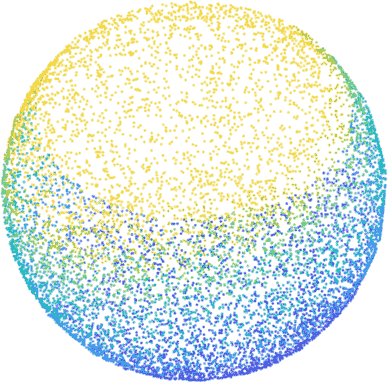}
    \includegraphics[width=0.18\textwidth, height=0.18\textwidth, keepaspectratio]{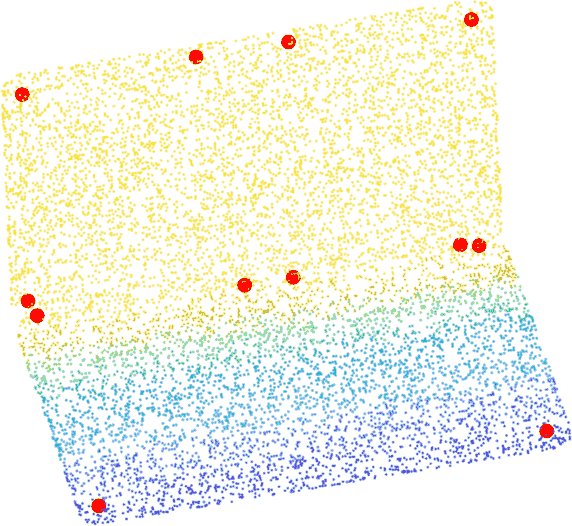}
    \includegraphics[width=0.18\textwidth, height=0.18\textwidth, keepaspectratio]{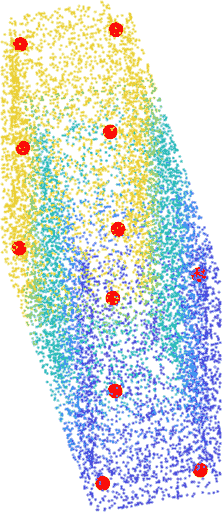}
    \includegraphics[width=0.18\textwidth, height=0.18\textwidth, keepaspectratio]{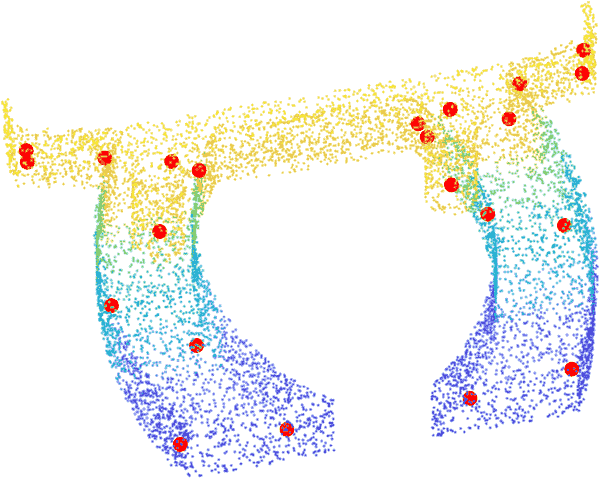}
    \includegraphics[width=0.18\textwidth, height=0.18\textwidth, keepaspectratio]{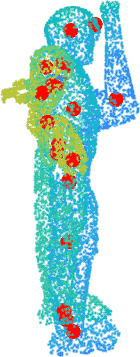}
    \includegraphics[width=0.18\textwidth, height=0.18\textwidth, keepaspectratio]{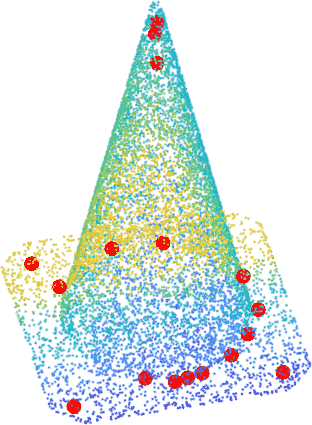}
    
    \caption{Visualization of USIP keypoints on ModelNet40. Best view with color and zoom-in.} \label{fig_suppl_vis_modelnet}
\end{figure*} 

\begin{figure*}[t] \centering
    \includegraphics[angle=90, width=0.33\textwidth, height=0.40\textwidth, keepaspectratio]{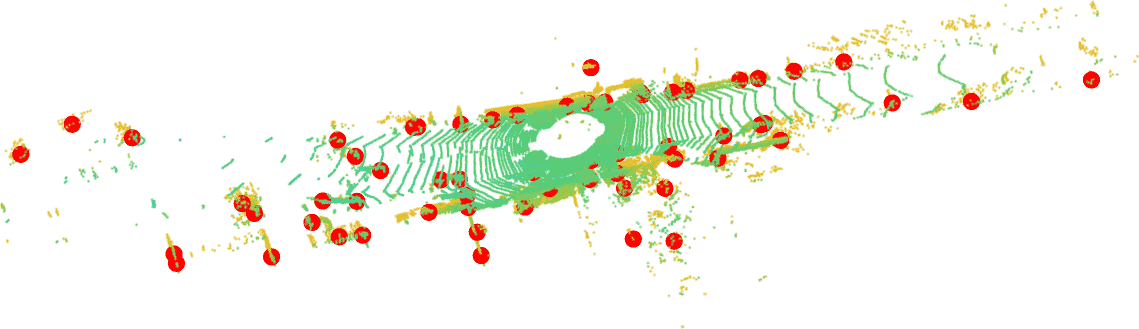}
    \includegraphics[angle=90, width=0.33\textwidth, height=0.40\textwidth, keepaspectratio]{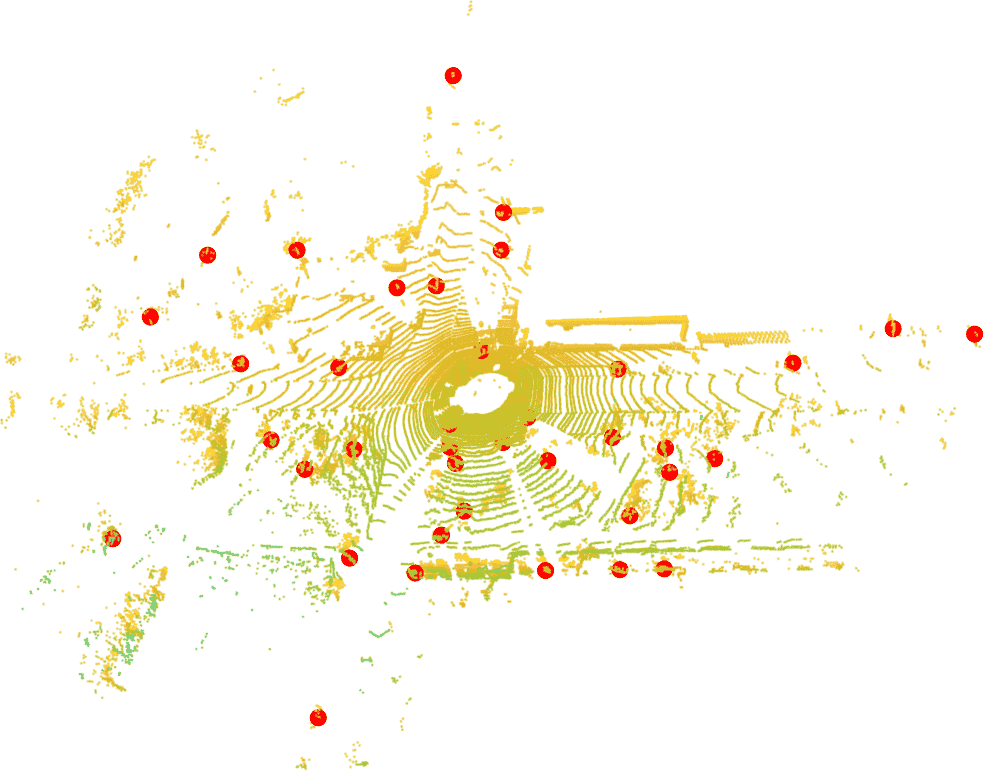}
    \includegraphics[angle=90, width=0.33\textwidth, height=0.40\textwidth, keepaspectratio]{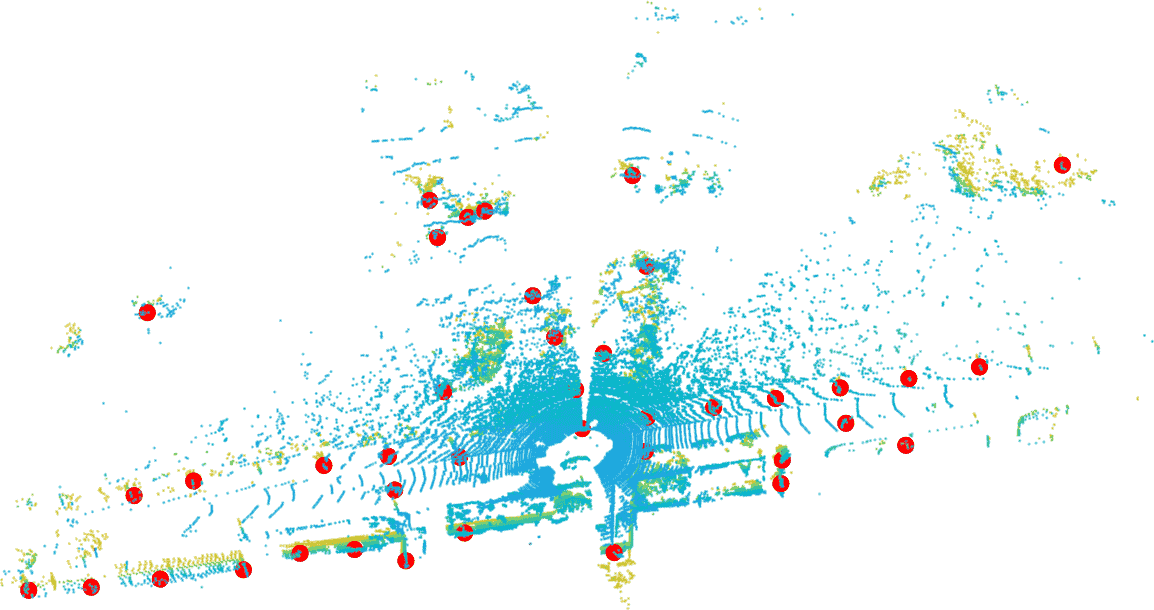}
    \includegraphics[angle=90, width=0.33\textwidth, height=0.40\textwidth, keepaspectratio]{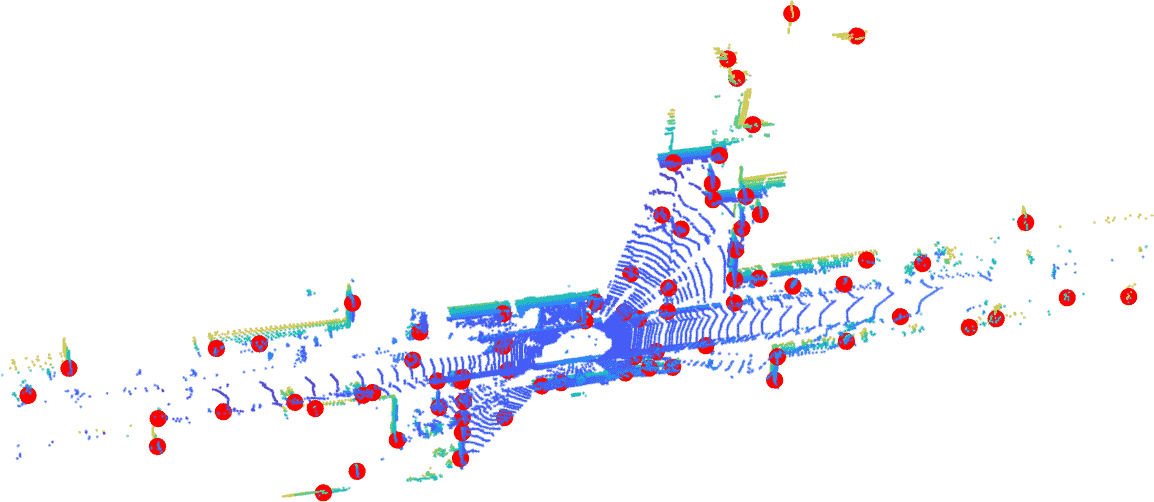}
    \includegraphics[angle=90, width=0.33\textwidth, height=0.40\textwidth, keepaspectratio]{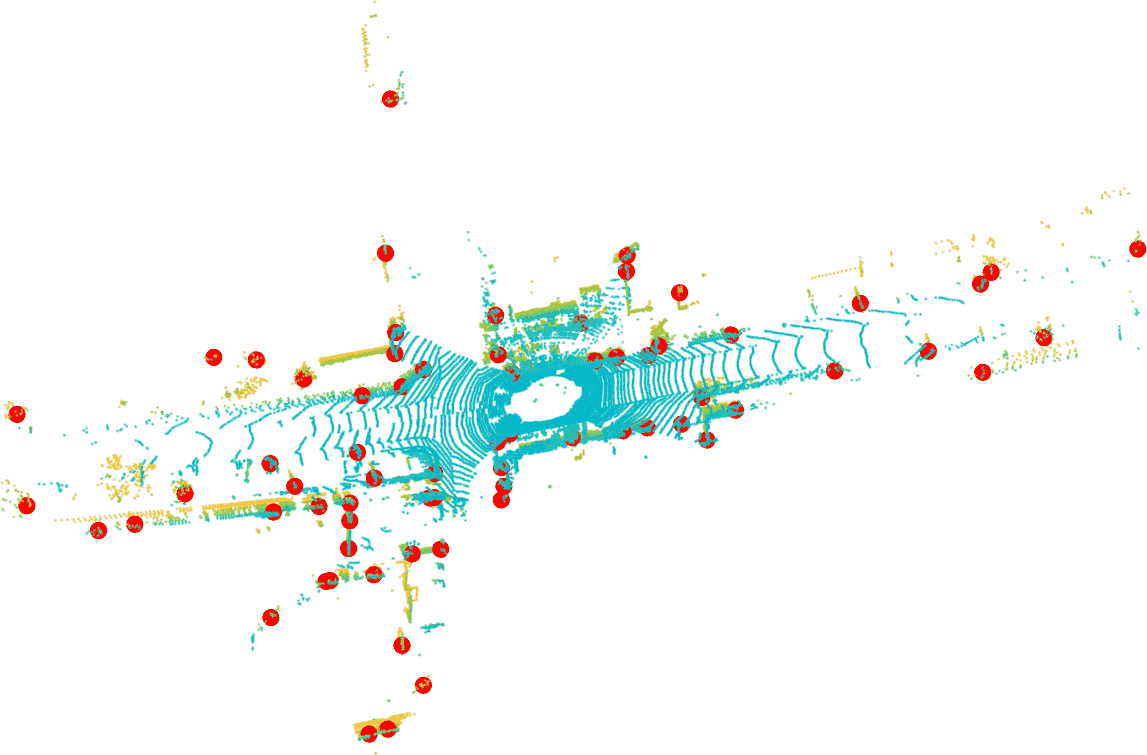}
    \includegraphics[angle=90, width=0.33\textwidth, height=0.40\textwidth, keepaspectratio]{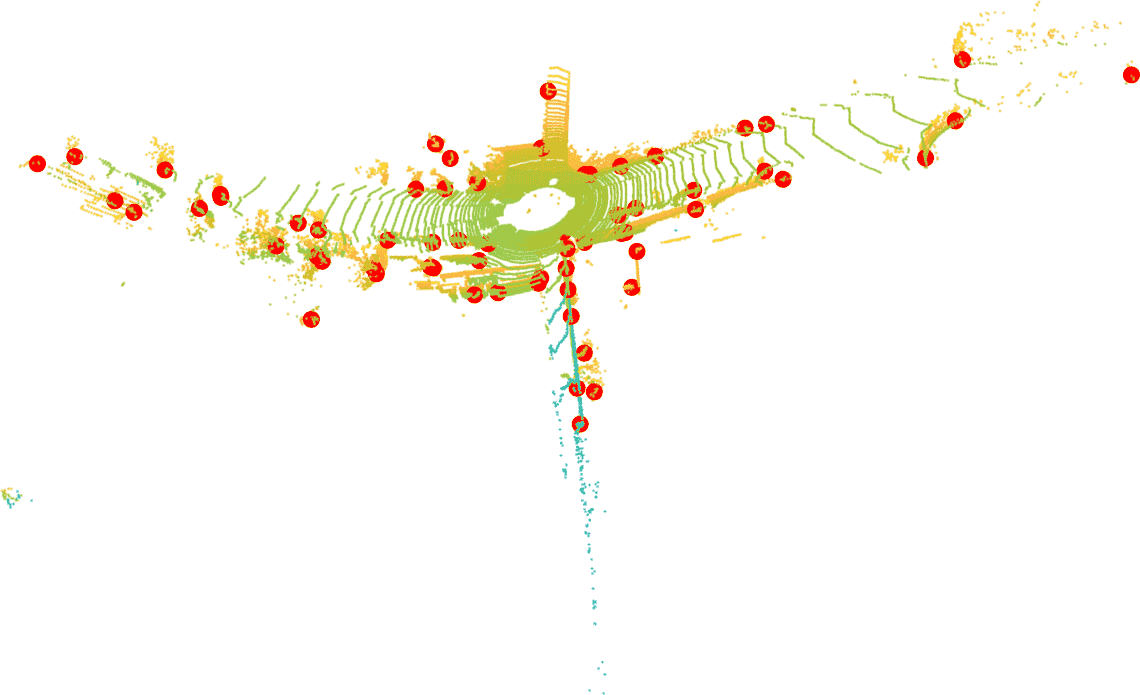}
    \includegraphics[angle=90, width=0.33\textwidth, height=0.40\textwidth, keepaspectratio]{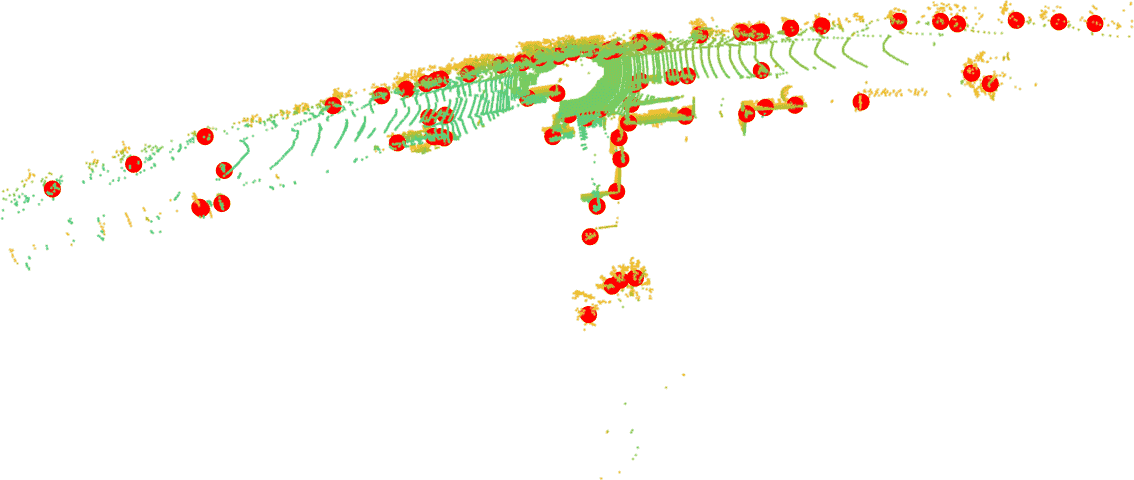}
    \includegraphics[angle=90, width=0.33\textwidth, height=0.40\textwidth, keepaspectratio]{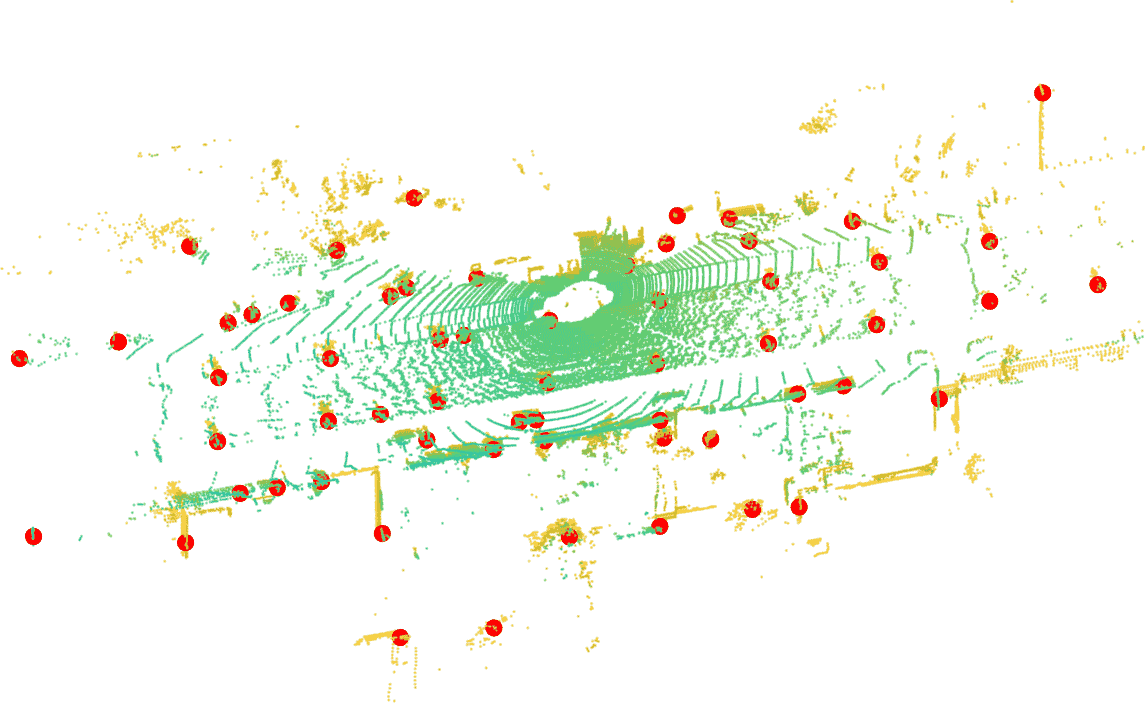}
    \includegraphics[angle=90, width=0.33\textwidth, height=0.40\textwidth, keepaspectratio]{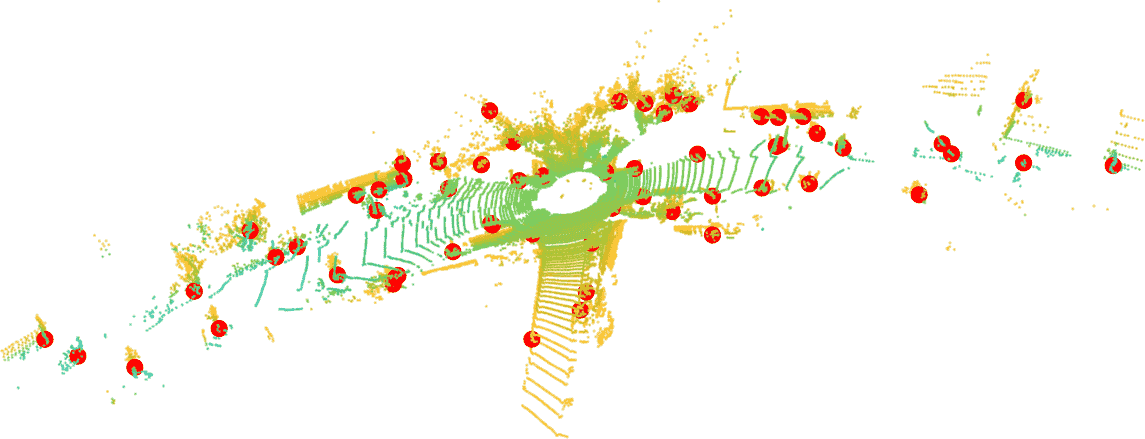}
    \includegraphics[angle=90, width=0.33\textwidth, height=0.40\textwidth, keepaspectratio]{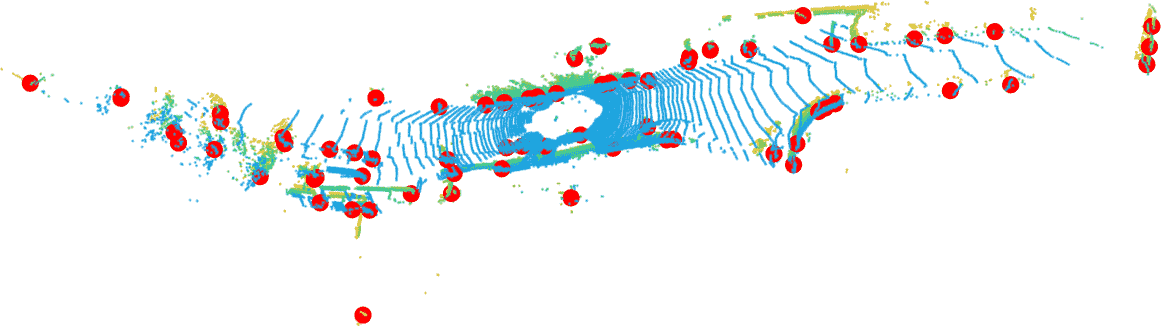}
    \includegraphics[angle=90, width=0.33\textwidth, height=0.40\textwidth, keepaspectratio]{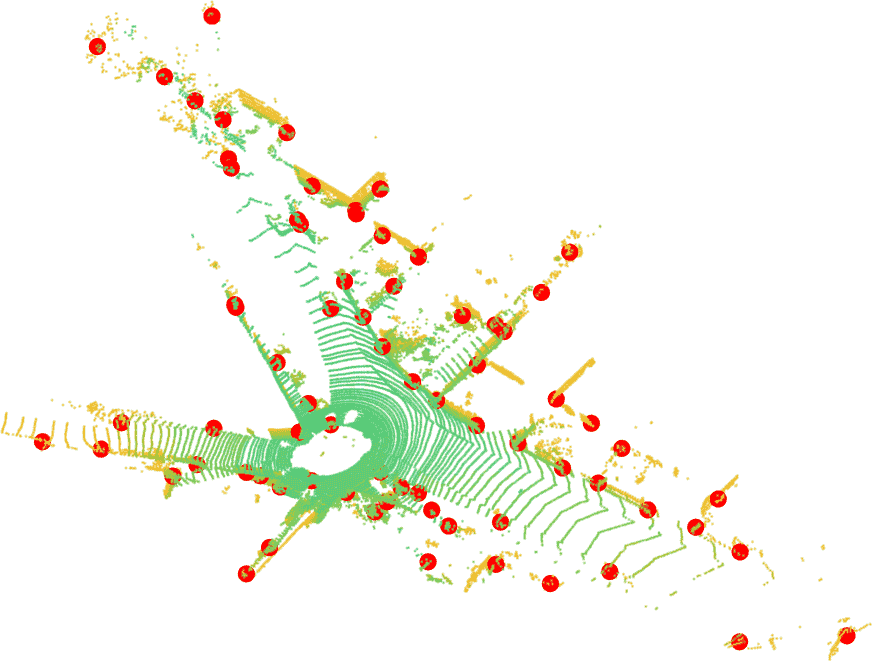}
    \includegraphics[angle=90, width=0.33\textwidth, height=0.40\textwidth, keepaspectratio]{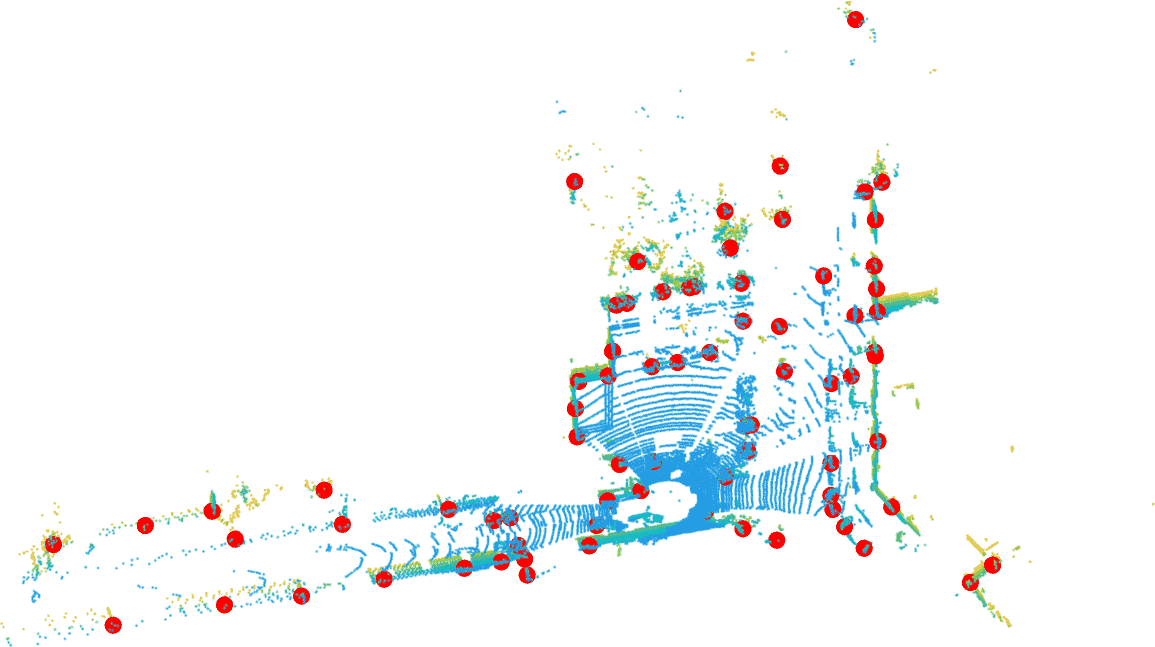}
    \includegraphics[angle=90, width=0.33\textwidth, height=0.40\textwidth, keepaspectratio]{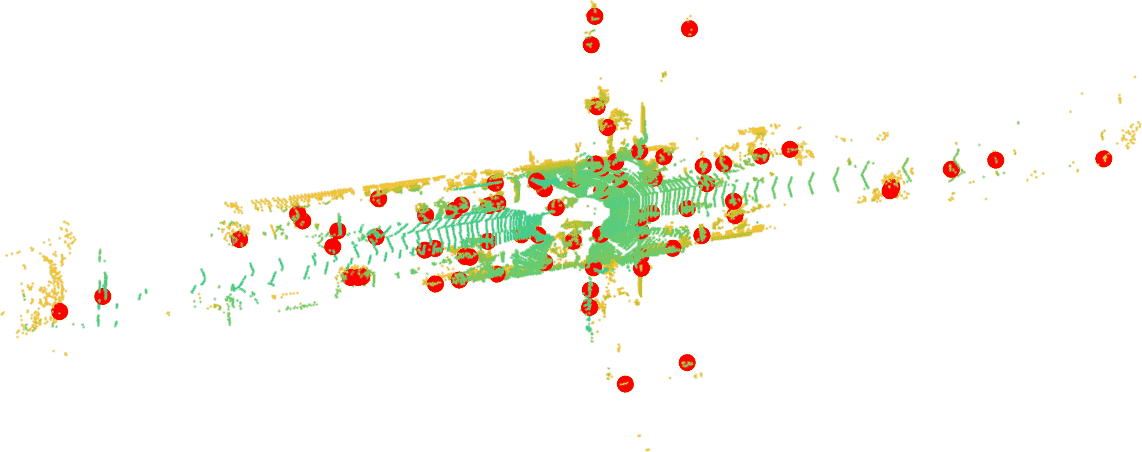}
    
    \caption{Visualization of USIP keypoints on KITTI with our USIP detector trained on Oxford RobotCar dataset. Best view with color and zoom-in.} \label{fig_suppl_vis_kitti}
\end{figure*}

\begin{figure*}[t] \centering
    \includegraphics[angle=0, width=0.33\textwidth, height=0.301\textwidth, keepaspectratio]{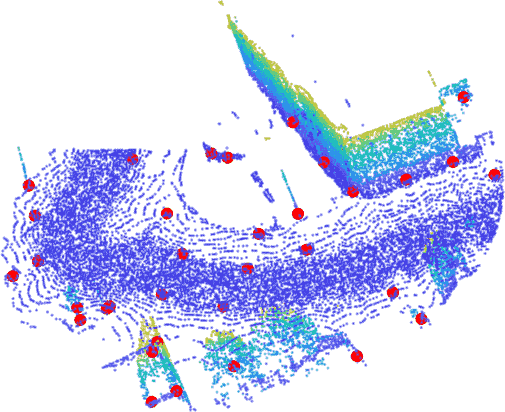}
    \includegraphics[angle=0, width=0.33\textwidth, height=0.301\textwidth, keepaspectratio]{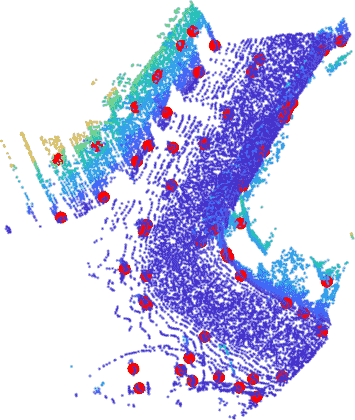}
    \includegraphics[angle=0, width=0.33\textwidth, height=0.301\textwidth, keepaspectratio]{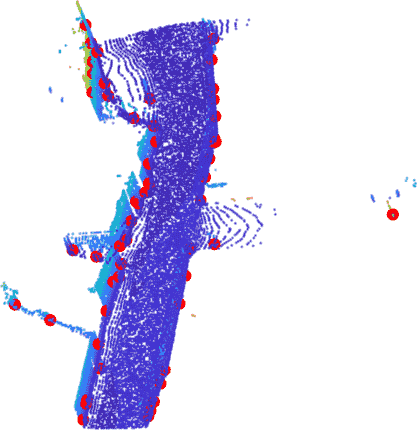}
    \includegraphics[angle=0, width=0.33\textwidth, height=0.301\textwidth, keepaspectratio]{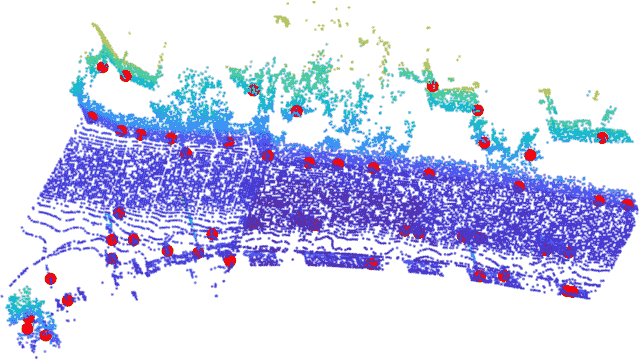}
    \includegraphics[angle=0, width=0.33\textwidth, height=0.301\textwidth, keepaspectratio]{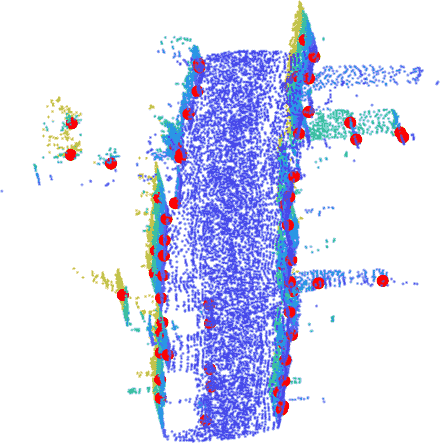}
    \includegraphics[angle=0, width=0.33\textwidth, height=0.301\textwidth, keepaspectratio]{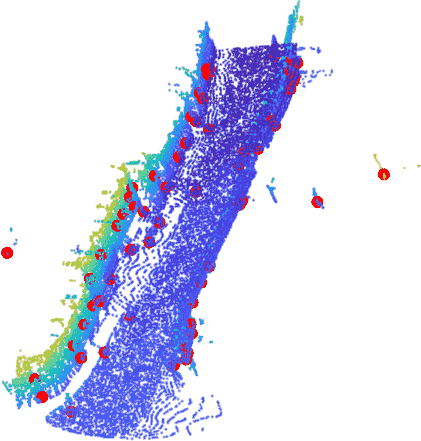}
    \includegraphics[angle=0, width=0.33\textwidth, height=0.301\textwidth, keepaspectratio]{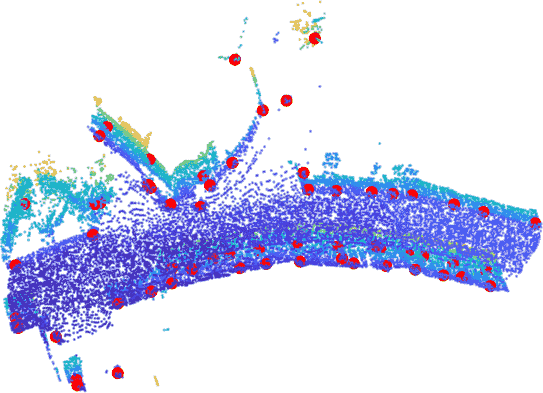}
    \includegraphics[angle=0, width=0.33\textwidth, height=0.301\textwidth, keepaspectratio]{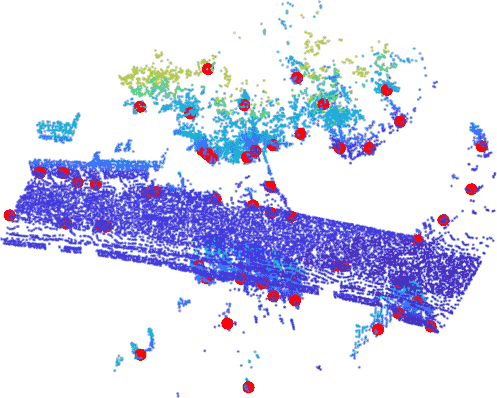}
    \includegraphics[angle=0, width=0.33\textwidth, height=0.301\textwidth, keepaspectratio]{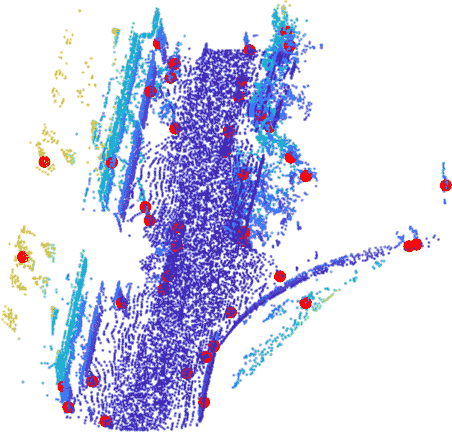}
    \includegraphics[angle=0, width=0.33\textwidth, height=0.301\textwidth, keepaspectratio]{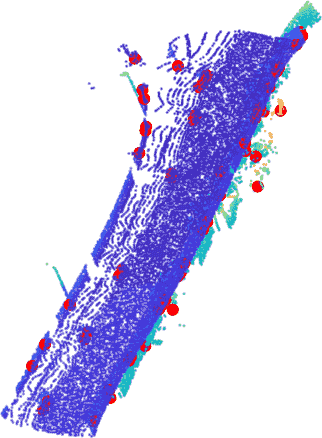}
    \includegraphics[angle=0, width=0.33\textwidth, height=0.301\textwidth, keepaspectratio]{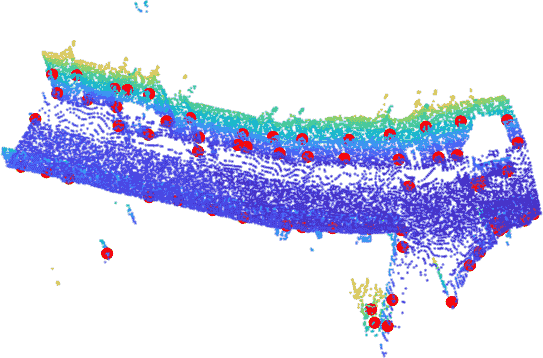}
    \includegraphics[angle=0, width=0.33\textwidth, height=0.301\textwidth, keepaspectratio]{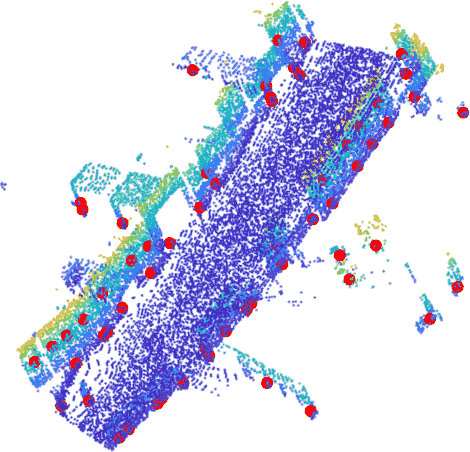}
    
    \caption{Visualization of USIP keypoints on Oxford RobotCar. Best view with color and zoom-in.} \label{fig_suppl_vis_oxford}
\end{figure*}

\begin{figure*}[t] \centering
    \includegraphics[angle=0, width=0.33\textwidth, height=0.301\textwidth, keepaspectratio]{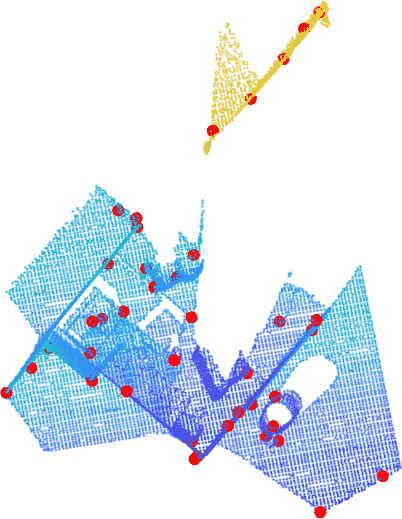}
    \includegraphics[angle=0, width=0.33\textwidth, height=0.301\textwidth, keepaspectratio]{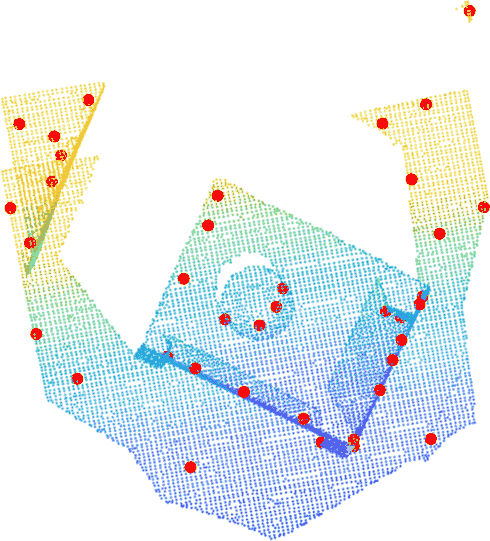}
    \includegraphics[angle=0, width=0.33\textwidth, height=0.301\textwidth, keepaspectratio]{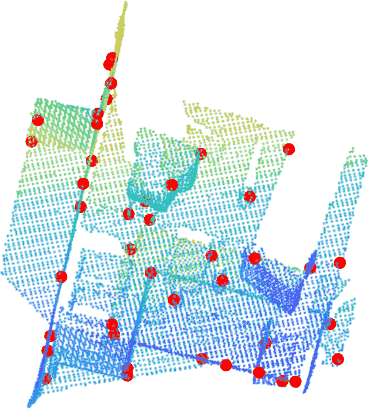}
    \includegraphics[angle=0, width=0.33\textwidth, height=0.301\textwidth, keepaspectratio]{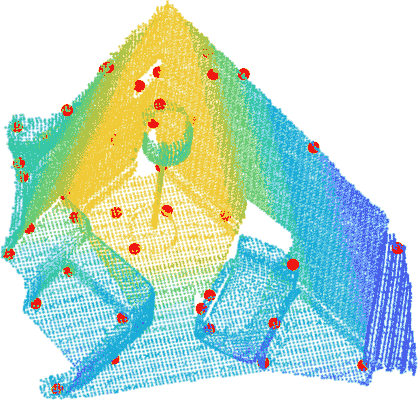}
    \includegraphics[angle=0, width=0.33\textwidth, height=0.301\textwidth, keepaspectratio]{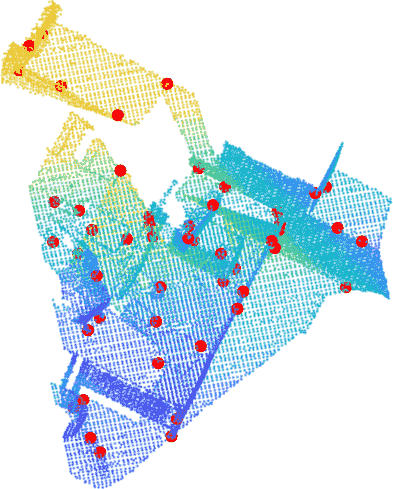}
    \includegraphics[angle=0, width=0.33\textwidth, height=0.301\textwidth, keepaspectratio]{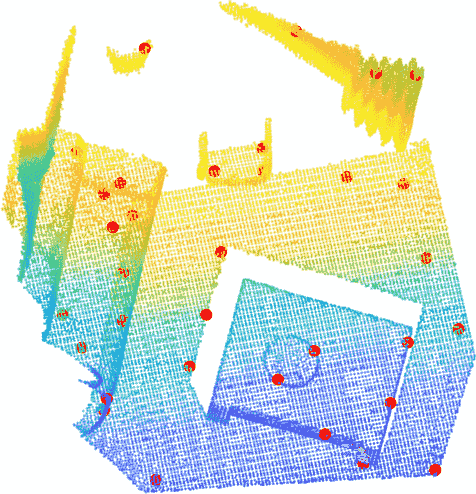}
    \includegraphics[angle=0, width=0.33\textwidth, height=0.301\textwidth, keepaspectratio]{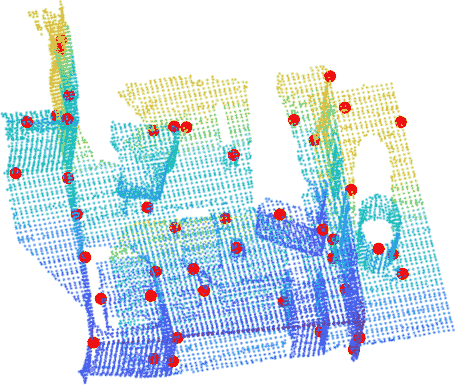}
    \includegraphics[angle=0, width=0.33\textwidth, height=0.301\textwidth, keepaspectratio]{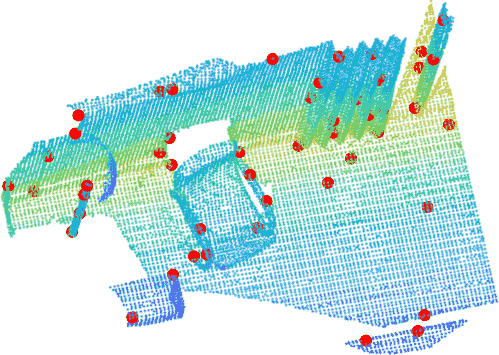}
    \includegraphics[angle=0, width=0.33\textwidth, height=0.301\textwidth, keepaspectratio]{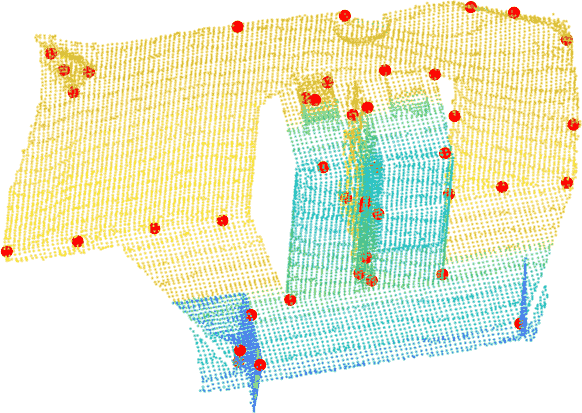}
    \includegraphics[angle=0, width=0.33\textwidth, height=0.301\textwidth, keepaspectratio]{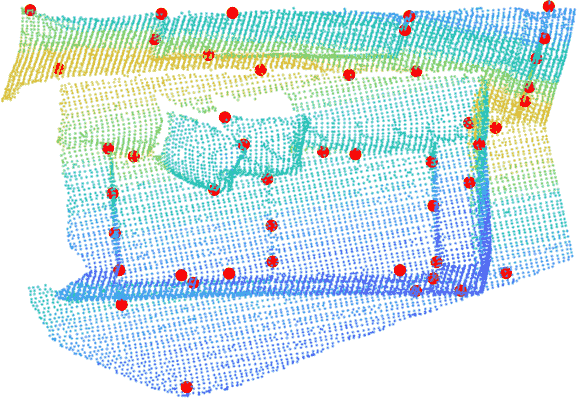}
    \includegraphics[angle=0, width=0.33\textwidth, height=0.301\textwidth, keepaspectratio]{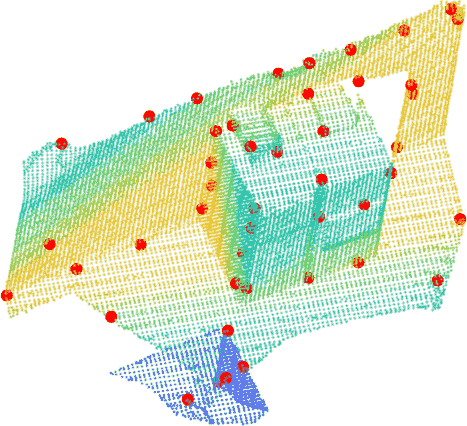}
    \includegraphics[angle=0, width=0.33\textwidth, height=0.301\textwidth, keepaspectratio]{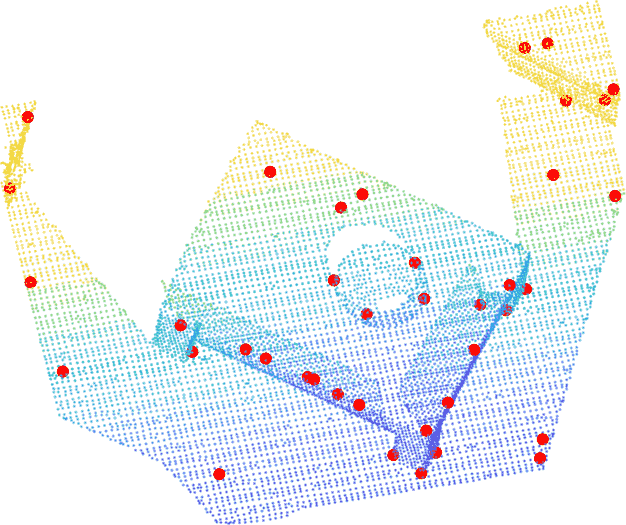}
    
    \caption{Visualization of USIP keypoints on Redwood with our USIP detector trained on ``3D Reconstruction Dataset" \cite{zeng20173dmatch}. Best view with color and zoom-in.} \label{fig_suppl_vis_redwood}
\end{figure*}

\end{document}